\documentclass{article}

\usepackage[utf8]{inputenc} 
\usepackage[T1]{fontenc}    
\usepackage{hyperref}       
\hypersetup{
    colorlinks=true,
    linkcolor=blue,
    filecolor=blue,      
    urlcolor=blue,
}

\PassOptionsToPackage{natbib=true,backend=biber,uniquelist=false}{biblatex}
\PassOptionsToPackage{table}{xcolor}
\usepackage[citations,authoryear,oxfordcolors,fancytheorems]{MJH}
\usepackage[final,nonatbib]{neurips_2022}

\usepackage{booktabs}       
\usepackage{amsfonts}       
\usepackage{nicefrac}       
\usepackage{microtype}      
\usepackage{printlen}
\usepackage{verbatim}
\usepackage[utf8]{inputenc} 
\usepackage[T1]{fontenc}    
\usepackage{hyperref}       
\usepackage{url}            
\usepackage{booktabs}       
\usepackage{amsfonts}       
\usepackage{nicefrac}       
\usepackage{microtype}      
\usepackage{tikz}

\usepackage{floatrow}
\usetikzlibrary{arrows.meta}
\usetikzlibrary{calc}
\usepackage[utf8]{inputenc}   
\usepackage[T1]{fontenc}      

\usepackage{url}
\usepackage{comment}

\usepackage{amsfonts}
\usepackage{amssymb,mathrsfs}
\usepackage{nccmath}
\usepackage{upgreek}
\usepackage{graphicx}
\usepackage{color}
\usepackage{algorithm}
\usepackage[noend]{algpseudocode}

\usepackage{stmaryrd}
\usepackage[inline]{enumitem}
\usepackage{url}

\usepackage{tikz}
\usetikzlibrary{calc}

\usepackage{bbm}
\usepackage{xargs}
\usepackage[textwidth=1.8cm]{todonotes}

\usepackage{aliascnt}
\usepackage[capitalise]{cleveref}
\Crefname{equation}{Eq.}{Eqs.}
\Crefname{figure}{Fig.}{Figs.}
\Crefname{tabular}{Tab.}{Tabs.}
\Crefname{section}{Sec.}{Secs.}
\Crefname{appendix}{App.}{Apps.}

\let\etoolboxforlistloop\forlistloop 
\usepackage{autonum}
\let\forlistloop\etoolboxforlistloop 
\cslet{blx@noerroretextools}\empty

\usepackage{bm}
\usepackage{wrapfig}
\usepackage{hyperref}

\def\rmB{\mathrm{B}}
\def\ellim{\ell^{\mathrm{im}}}
\def\piinv{p_{\textup{ref}}}

\def\piinv{p_{\textup{ref}}}
\def\pizero{p_0}

\newcommand{\grad}{\mathrm{grad}}
\newcommand{\dive}{\mathrm{div}}

\newcommand{\prodM}[2]{\langle #1, #2 \rangle_\M}
\newcommand{\XM}{\mathcal{X}(\mathcal{M})}
\newcommand{\XMdeux}{\mathcal{X}^2(\mathcal{M})}
\newcommand{\Xgamma}{\mathcal{X}(\gamma)}
\newcommand{\TM}{\mathrm{T}\mathcal{M}}
\newcommand{\FM}{\mathrm{F}\mathcal{M}}
\newcommand{\OM}{\mathrm{O}\mathcal{M}}
\newcommand{\TMstar}{\mathrm{T}^\star\mathcal{M}}

\def\carrechamp{\Upsilon}
\def\carrechampb{\bar{\Upsilon}}

\def\contspace{\mathcal{C}}
\def\pdata{p_{\textup{data}}}

\def\Pens{\mathcal{P}}

\newcommand{\tta}{\mathtt{A}}

\newcommand{\Capprox}{\tta}

\newcommandx\ctun[1][1=T]{\Capprox_{#1,1}}

\newcommand{\rref}[1]{\tup{\Cref{#1}}}

\newcommandx{\expec}[2]{{\mathbb E}\left[#1 \middle \vert #2  \right]} 

\newcommand{\rme}{\mathrm{e}}

\newcommand{\Lip}{\mathtt{L}}

\newcommand{\Mtt}{\mathtt{M}}

\newcommand{\Ktt}{\mathtt{K}}

\newcommand{\SDE}{\mathrm{SDE}}

\def\u{{ \boldsymbol u}}

\newcommandx{\norm}[2][1=]{\ifthenelse{\equal{#1}{}}{\left\Vert #2 \right\Vert}{\left\Vert #2 \right\Vert^{#1}}}
\newcommandx{\normLigne}[2][1=]{\ifthenelse{\equal{#1}{}}{\Vert #2 \Vert}{\Vert #2\Vert^{#1}}}

\def\bfc{\mathbf{c}}
\def\bfY{\mathbf{Y}}
\def\bfhY{\hat{\mathbf{Y}}}

\def\bfX{\mathbf{X}}
\def\bfhX{\hat{\mathbf{X}}}

\def\bfU{\mathbf{U}}
\def\bfE{\mathbf{E}}

\def\bfZ{\mathbf{Z}}

\def\bfZ{\mathbf{Z}}

\def\barbfX{\bar{\mathbf{X}}}
\def\bfM{\mathbf{M}}
\def\bfB{\mathbf{B}}

\def\msa{\mathsf{A}}

\def\msf{\mathsf{F}}

\def\msu{\mathsf{U}}
\def\msv{\mathsf{V}}

\def\msx{\mathsf{X}}

\newcommand{\mcb}[1]{\mathcal{B}(#1)}

\def\mcx{\mathcal{X}}

\def\mcf{\mathcal{F}}

\def\Qbb{\mathbb{Q}}

\def\Pbb{\mathbb{P}}

\def\rset{\mathbb{R}}

\def\zset{\mathbb{Z}}
\def\tset{\mathbb{T}}
\def\nset{\mathbb{N}}

\def\rmA{\mathrm{A}}
\def\rmP{\mathrm{P}}

\def\rmd{\mathrm{d}}

\def\rme{\mathrm{e}}

\def\rmc{\mathrm{C}}

\newcommand{\R}{\mathbb R}

\newcommand{\M}{\mathcal M}

\newcommandx{\functionspace}[2][1=+]{\mathbb{F}_{#1}(#2)}

\newcommand{\argmin}{\operatorname*{arg\,min}}

\newcommandx{\VarDeux}[3][3=]{\operatorname{Var}^{#3}_{#1}\left\{#2 \right\}}

\newcommand{\1}{\mathbbm{1}}

\newcommand{\LeftEqNo}{\let\veqno\@@leqno}

\newcommand{\floor}[1]{\left\lfloor #1 \right\rfloor}
\newcommand{\ceil}[1]{\left\lceil #1 \right\rceil}

\newcommand{\N}{\ensuremath{\mathbb{N}}}

\newcommand{\C}{\ensuremath{\mathbb{C}}}

\newcommand{\PE}{\mathbb{E}}

\newcommand{\abs}[1]{\left\vert #1 \right\vert}
\newcommand{\absLigne}[1]{\vert #1 \vert}
\newcommand{\tvnorm}[1]{\| #1 \|_{\mathrm{TV}}}

\newcommandx{\Vnorm}[2][1=V]{\| #2 \|_{#1}}
\newcommandx{\VnormEq}[2][1=V]{\left\| #2 \right\|_{#1}}

\newcommand{\eqdef}{=}

\newcommandx\probaMarkovTilde[2][2=]
{\ifthenelse{\equal{#2}{}}{{\widetilde{\mathbb{P}}_{#1}}}{\widetilde{\mathbb{P}}_{#1}\left[ #2\right]}}

\newcommand{\expe}[1]{\PE \left[ #1 \right]}

\newcommand{\expeLigne}[1]{\PE [ #1 ]}

\def\ie{i.e.}

\def\eqsp{\;}
\newcommand{\coint}[1]{\left[#1\right)}
\newcommand{\ocint}[1]{\left(#1\right]}
\newcommand{\ooint}[1]{\left(#1\right)}
\newcommand{\ccint}[1]{\left[#1\right]}

\newcommandx{\weight}[2][2=n]{\omega_{#1,#2}^N}

\newcommand{\cball}[2]{\bar{\operatorname{B}}(#1,#2)}

\newcommandx\sequence[3][2=,3=]
{\ifthenelse{\equal{#3}{}}{\ensuremath{\{ #1_{#2}\}}}{\ensuremath{\{ #1_{#2}, \eqsp #2 \in #3 \}}}}

\newcommandx\sequenceD[3][2=,3=]
{\ifthenelse{\equal{#3}{}}{\ensuremath{\{ #1_{#2}\}}}{\ensuremath{( #1)_{ #2 \in #3} }}}

\newcommandx{\sequencen}[2][2=n\in\N]{\ensuremath{\{ #1_n, \eqsp #2 \}}}
\newcommandx\sequenceDouble[4][3=,4=]
{\ifthenelse{\equal{#3}{}}{\ensuremath{\{ (#1_{#3},#2_{#3}) \}}}{\ensuremath{\{  (#1_{#3},#2_{#3}), \eqsp #3 \in #4 \}}}}
\newcommandx{\sequencenDouble}[3][3=n\in\N]{\ensuremath{\{ (#1_{n},#2_{n}), \eqsp #3 \}}}

\newcommand{\opnorm}[1]{{\left\vert\kern-0.25ex\left\vert\kern-0.25ex\left\vert #1
    \right\vert\kern-0.25ex\right\vert\kern-0.25ex\right\vert}}

\def\Lip{\operatorname{Lip}}
\def\Ltt{\mathtt{L}}
\def\generator{\mathcal{A}}
\def\generatorb{\bar{\mathcal{A}}}
\def\generatort{\tilde{\mathcal{A}}}

\def\Id{\operatorname{Id}}

\newcommandx{\CPE}[3][1=]{{\mathbb E}_{#1}\left[#2 \middle \vert #3  \right]} 
\newcommandx{\CPELigne}[3][1=]{{\mathbb E}_{#1}[#2  \vert #3  ]} 
\newcommandx{\CPEsq}[3][1=]{{\mathbb{E}^{1/2}}_{#1}\left[#2 \middle \vert #3  \right]} 
\newcommandx{\CPVar}[3][1=]{\mathrm{Var}^{#3}_{#1}\left\{ #2 \right\}}
\newcommand{\CPP}[3][]
{\ifthenelse{\equal{#1}{}}{{\mathbb P}\left(\left. #2 \, \right| #3 \right)}{{\mathbb P}_{#1}\left(\left. #2 \, \right | #3 \right)}}

\newcommandx{\osc}[2][1=]{\mathrm{osc}_{#1}(#2)}

\def\Id{\operatorname{Id}}

\def\c{c}

\def\dom{\mathrm{dom}}

\def\bb{\overline{b}}
\def\bc{\bar{c}}

\newcommand{\ensembleLigne}[2]{\{#1\,:\eqsp #2\}}

\newcommand\coupling[2]{\Gamma(\mu,\nu)}

\newcommand{\complementary}{\mathrm{c}}

\def\interior{\mathrm{int}}

\def\vareps{\varepsilon}

\def\Phibf{\mathbf{\Phi}}

\newcommandx{\KL}[2]{\operatorname{KL}\left( #1 | #2 \right)}
\newcommandx{\KLsqrt}[2]{\operatorname{KL}^{1/2}\left( #1 | #2 \right)}
\newcommandx{\Jef}[2]{\operatorname{J}\left( #1 , #2 \right)}
\newcommandx{\JefLigne}[2]{\operatorname{J}( #1 , #2 )}
\newcommandx{\KLLigne}[2]{\operatorname{KL}( #1 | #2 )}
\newcommandx{\KLLignesqrt}[2]{\operatorname{KL}^{1/2}( #1 | #2 )}

\def\gaStep
\def\QKer{Q}

\def\distance{\mathbf{d}}
\newcommandx{\wasserstein}[3][1=\distance,3=]{\mathbf{W}_{#1}^{#3}\left(#2\right)}
\newcommandx{\wassersteinLigne}[3][1=\distance,3=]{\mathbf{W}_{#1}^{#3}(#2)}
\newcommandx{\wassersteinD}[1][1=\distance]{\mathbf{W}_{#1}}
\newcommandx{\wassersteinDLigne}[1][1=\distance]{\mathbf{W}_{#1}}

\def\Rcoupling{\mathrm{R}}

\def\sigmaD{\sigma^2}

\newcommandx{\phibfs}[1][1=]{\pmb{\varphi}_{\sigmaD_{#1}}}

\def\E{\mathbb{E}}

\newcommandx\sequenceg[3][2=,3=]
{\ifthenelse{\equal{#3}{}}{\ensuremath{( #1_{#2})}}{\ensuremath{( #1_{#2})_{ #2 \geq #3}}}}

\def\Rker{\Rcoupling}

\def\Pker{\mathrm{P}}

\def\Qker{\mathrm{Q}}

\def\rmL{\mathrm{L}}

\newcommandx{\distV}[1][1=\bfc]{\mathbf{W}_{#1}}
\newcommandx{\distVdeux}[1][1=W_2]{\mathbf{d}_{#1}}

\newcommand{\tup}[1]{\textup{#1}}

\usepackage{comment}
\usepackage{authblk}
\usepackage{cancel}

\makeatletter
\renewcommand\AB@affilsepx{, \protect\Affilfont}
\makeatother
\newcommand{\appendixhead}{
  \centerline{\textbf{\LARGE Supplementary to: }\vspace{0.15in}}
  \centerline{\textbf{\LARGE Riemannian Score-Based Generative Modelling}\vspace{0.25in}}
}
\let\origappendix\appendix 

\renewcommand\appendix{\pagenumbering{arabic}\origappendix}
\AtBeginEnvironment{appendices}{\crefalias{section}{appendix}}

\colorlet{linkcolor}{blue!70!black}
\definecolor{pearDark}{HTML}{2980B9}
\hypersetup{
  colorlinks,
  linkcolor={red!50!black},
  citecolor={blue!50!black},
  urlcolor={blue!50!black}
}
\usepackage[font={small}]{caption, subcaption}
\graphicspath{{images/}}
\newcommand*\samethanks[1][\value{footnote}]{\footnotemark[#1]}
\definecolor{commentcolor}{rgb}{0., 0.5, 0.}
\newcommand{\red}[1]{{#1}}

\usepackage{multirow}
\usepackage{amssymb}
\usepackage{pifont}
\newcommand{\cmark}{\ding{51}}%
\newcommand{\xmark}{\ding{55}}

\title{Riemannian Score-Based Generative Modelling}

\author{
    \textbf{Valentin De Bortoli}\thanks{equal contribution.}\;\;\thanks{Dept. of Computer Science ENS, CNRS, PSL University Paris, France.}\ ,
    \textbf{\'Emile Mathieu}\samethanks[1]\;\;\thanks{Dept. of Statistics, University of Oxford, Oxford, UK.}\ , 
    \textbf{Michael Hutchinson}\samethanks[1]\;\;\samethanks[3], \\
    \textbf{James Thornton}\samethanks[3], 
    \textbf{Yee Whye Teh}\samethanks[3], 
    \textbf{Arnaud Doucet}\samethanks[3]
}

\begin{document}

\maketitle

\begin{abstract}
Score-based generative models (SGMs) are a powerful class of generative models that exhibit remarkable empirical performance. Score-based generative modelling (SGM) consists of a ``noising'' stage, whereby a diffusion is used to gradually add Gaussian noise to data, and a generative model, which entails a ``denoising'' process defined by approximating the time-reversal of the diffusion. Existing SGMs assume that data is supported on a Euclidean space, i.e.\ a manifold with flat geometry.  In many domains such as robotics, geoscience or protein modelling,  data is often naturally described by distributions living on Riemannian manifolds and current SGM techniques are not appropriate. We introduce here \emph{Riemannian Score-based Generative Models} (RSGMs), a class of generative models extending SGMs to Riemannian manifolds.  We demonstrate our approach on a variety of manifolds, and in particular with earth and climate science spherical data.

\end{abstract}

\section{Introduction}
\label{sec:introduction}
Score-based Generative Models (SGMs) also called diffusion models 
\citep{song2019generative,song2020score,ho2020denoising,nichol2021beatgans}
formulate generative modelling as a denoising process. Noise is incrementally
added to data using a diffusion process until it becomes approximately Gaussian. The
generative model is then obtained by simulating an approximation of the
corresponding time-reversal process, which progressively denoises a Gaussian
sample to obtain a data sample. This process is also a diffusion whose drift
depends on the logarithmic gradients of the noised data densities, i.e. the
Stein scores, estimated using a neural network via score matching
\citep{hyvarinen2005estimation,vincent2011connection}.

SGMs have been primarily applied to data living on Euclidean spaces, i.e. manifolds with flat geometry.
However, in a large number of scientific domains the distributions of interest
are supported on Riemannian manifolds. These include, to name a few, protein
modelling
\citep{shapovalov2011smoothed}, 
cell development \citep{klimovskaia2020poincare}, image recognition
\citep{lui2012advances}, geological sciences
\citep{karpatne2018machine,peel2001fitting}, graph-structured and hierarchical
data \citep{roy2007learning,steyvers2005large}, robotics
\citep{feiten2013rigid,senanayake2018directional} and high-energy physics
\citep{brehmer2020flows}.

We introduce in this work \emph{Riemannian Score-based Generative Models} (RSGMs), an
extension of SGMs to Riemannian manifolds which incorporate the geometry
of the data by defining the forward diffusion process directly on the Riemannian
manifold, inducing a manifold-valued reverse process. This requires constructing a noising process on the manifold that
converges to an easy-to-sample reference distribution. We establish that,
as in the Euclidean case, the corresponding time-reversal process is also a
diffusion whose drift includes the Stein score which is intractable but can similarly be estimated
via score matching.
Methodological extensions are required as in most cases the
transition kernel of the noising process cannot be sampled exactly. For example on compact manifolds it is
typically only available as an infinite sum through the Sturm--Liouville
decomposition \citep{chavel1984eigenvalues}. To
this end, we develop non-standard techniques for score estimation and rely on
the use of Geodesic Random Walks for sampling \citep{jorgensen1975central}.
We provide theoretical convergence bounds for RSGMs on compact manifolds and demonstrate our approach on 
a range of manifolds and tasks, including 
modelling a number of natural disaster
occurrence datasets collected by \textcite{mathieu2020riemannian}. We show that
RGSMs achieve better performance than recent baselines
\citep{mathieu2020riemannian,rozen2021moser} and scale better to
high-dimensional manifolds.

\section{Euclidean Score-based Generative Modelling}
\label{sec:eucl-sgm-riem}
We recall here briefly the key concepts behind SGMs on the Euclidean space $\rset^d$ and refer the readers to \textcite{song2020score} for a more detailed introduction.
We consider a forward \emph{noising} process $(\bfX_t)_{t \geq 0}$ defined by the following Stochastic
Differential Equation (SDE)
\begin{equation}\label{eq:forward_SDE}
  \rmd \bfX_t = -\bfX_t \rmd t + \sqrt{2} \rmd \bfB_t,\quad \bfX_0 \sim p_0 ,
\end{equation}
where $(\bfB_t)_{t \geq 0}$ is a $d$-dimensional Brownian motion and $p_0$ is the data distribution. The available data gives us an empirical approximation of $p_0$. The 
process $(\bfX_t)_{t \geq 0}$ is simply an Ornstein--Ulhenbeck (OU) process which converges with geometric rate to $\mathrm{N}(0,\Id)$. Under mild conditions on $p_0$, the time-reversed process
$(\bfY_t)_{t \geq 0} = (\bfX_{T-t})_{t \in \ccint{0,T}}$ also satisfies an SDE
\citep{cattiaux2021time,haussmann1986time} given by
\begin{equation} \label{eq:backward_SDE}
  \rmd \bfY_t = \{ \bfY_t + 2 \nabla \log p_{T-t}(\bfY_t)\} \rmd t + \sqrt{2} \rmd \bfB_t,\quad \bfY_0 \sim p_T ,
\end{equation}
where $p_t$ denotes the density of $\bfX_t$. By construction, the law of
$\bfY_{T-t}$ is equal to the law of $\bfX_t$ for $t \in \ccint{0,T}$ and in
particular $\bfY_{T}\sim p_0$. Hence, if one could sample from
$(\bfY_t)_{t \in \ccint{0,T}}$ then its final distribution would be the data distribution $p_0$.  Unfortunately we cannot sample exactly from (\ref{eq:backward_SDE}) as $p_T$ and the scores $(\nabla \log p_t(x))_{t\in[0,T]}$ are intractable. Hence SGMs rely on a few approximations.
First, $p_T$ is replaced by the reference distribution $\mathrm{N}(0,\Id)$ as we know that $p_T$ converges geometrically towards it.
Second, the following denoising score matching identity is exploited to estimate the scores 
\begin{equation}\label{eq:scoreidentity}
  \textstyle{\nabla_{x_t} \log p_t(x_t) = \int_{\rset^d} \nabla_{x_t} \log p_{t|0}(x_t|x_0)~p_{0|t}(x_0|x_t) \rmd x_0,}
\end{equation}
where $p_{t|0}(x_t|x_0)$ is the transition density of the OU process
(\ref{eq:forward_SDE}) which is available in closed-form. It follows directly
that $\nabla \log p_t$ is the minimizer of
$\ell_t(\mathbf{s}) = \expeLigne{\normLigne{\mathbf{s}(\bfX_t) - \nabla_{x_t}
    \log p_{t|0}(\bfX_t|\bfX_0)}^2}$ over functions $\mathbf{s}$ where the
expectation is over the joint distribution of $\bfX_0,\bfX_t$. This result can
be leveraged by considering a neural network
$\mathbf{s}_\theta: \ccint{0,T} \times \rset^d \to \rset^d$ trained by
minimizing the loss function
$\ell(\theta)=\int_0^T \lambda_t \ell_{t}(\mathbf{s}_\theta(t,\cdot))\rmd t$ for
some weighting function $\lambda_t>0$.
Finally, an Euler--Maruyama
discretization of (\ref{eq:backward_SDE}) is performed using a discretization step $\gamma$
such that $T=\gamma N$ for $N \in
\nset$
\begin{equation}\label{eq:backward_discrete_final}
  Y_{n+1} = Y_n + \gamma \{Y_n + 2  \mathbf{s}_\theta(T -n \gamma, Y_n) \} + \sqrt{2 \gamma} Z_{n+1},\quad Y_0\sim \mathrm{N}(0,\Id), \quad Z_n \overset{\textup{i.i.d.}}{\sim} \mathrm{N}(0,\Id).
\end{equation}
The above showcases the basics of SGMs but we highlight that many improvements have been proposed; see e.g.~\textcite{song2020improved,jolicoeur2020adversarial,nichol2021beatgans}. In
particular, selecting an adaptive stepsize
$(\gamma_n)_{n \in \nset}$ \citep{bao2022analyticdpm,watson2021learning}  and using a predictor-corrector scheme
\citep{song2020score} instead of a simple Euler--Maruyama discretization drastically improves performance. 


\section{Riemannian Score-based Generative Modelling}
\label{sec:score-appr-manif}
We now move to the Riemannian manifold setting, and more specifically assume
that $\M$ is a \red{complete, orientable connected and boundaryless Riemannian
  manifold}, endowed with a Riemannian metric $g$~\footnote{
Metrics $g$ are sections of $\mathrm{T}^* \M \otimes \mathrm{T}^* \M$, the rank 2 tensor bundle of the dual tangent space, i.e.\ smooth varying bilinear maps on $\mathrm{T} \M$, verifying symmetry and positive semi-definiteness.}.
Four components are required to extend SGMs to this setting: 
\begin{enumerate*}[label=\roman*)]
\item a forward \emph{noising} process on $\M$  which converges to an easy-to-sample reference distribution, 
\item a time-reversal formula on $\M$  which defines a backward generative process,
\item a method for approximating samples of SDEs on manifolds,
\item a method to efficiently approximate the drift of the time-reversal process.
\end{enumerate*}
Notation are gathered in \Cref{sec:notation-1}.

\subsection{Noising processes on manifolds}

\label{sec:non-compact}

The first necessary component is a suitable generic noising process on manifolds that will converge to a convenient stationary distribution.
A simple choice is to use Langevin dynamics 
described by
\begin{equation}
  \label{eq:langevin}
 \rmd \bfX_t = - \tfrac{1}{2}~\nabla_{\bfX_t} U(\bfX_t) \rmd t + \rmd \bfB_t^\M,
\end{equation}
which admits the invariant density (w.r.t.\ the volume form) given by $\rmd \piinv/ \rmd \textrm{Vol}_\M(x) \propto \rme^{-U(x)}$ \parencite[Section 2.4]{durmus2016high}, where $\nabla$ is the Riemannian gradient\footnote{The (Riemannian) gradient $\nabla$ is defined s.t.\ for any $f:\M \rightarrow \R$, $x \in \M, v \in \mathrm{T}_x\M$, $\langle \nabla f, v \rangle_g = \rmd f (v)$.}.

Two simple choices for $U(x)$ present themselves. Firstly, setting $U(x) = d_\M(x, \mu)^2/(2 \gamma^2)$, where $d_\M$ is the geodesic distance 
and $\mu \in \M$ is an arbitrary mean location,
induces the drift $\nabla_{\bfX_t} U(\bfX_t) = -\exp^{-1}_{\bfX_t}(\mu) /\gamma^2$~\footnote{
    $\exp_x: \mathrm{T}_x \M \to \M$ denotes the exponential mapping on the
    manifold, see e.g. \textcite[Chapter 20]{lee2013smooth}.}.
This is the potential of the `Riemannian normal' \parencite{pennec2006Intrinsic} distribution.
An alternative is to target the 'exponential wrapped' Gaussian. This is the pushforward of a Gaussian distribution in the tangent space at the mean location along the exponential map. The potential is given by $U(x) = d_\M(x, \mu)^2/(2 \gamma^2) + \log |D \exp^{-1}_{\mu}(x)|$\footnote{$|\cdot|$ denotes the absolute value of the determinant, and $D f$ the Jacobian of $f$.}.
In contrast to the Riemannian normal, sampling and evaluating the density of this distribution is easy \parencite[e.g.][]{mathieu2019continuous}.

One recovers the standard Ornstein--Uhlenbeck noising process~\citep{song2020score} for both of these target distributions when $\M=\rset^d$ and $\mu = 0$ since then the drift $b(t, \bfX_t) = \tfrac{1}{2} ~\exp^{-1}_{\bfX_t}(0) = -\tfrac{1}{2}~\bfX_t$.
On compact manifolds, the invariant measure $\textrm{Vol}_\M$ has finite volume,
thus a natural choice is to target the uniform distribution which is given by
$\textrm{Vol}_\M / |\M|$. In this case, $\nabla_{\bfX_t} U(\bfX_t) = 0$ and the
noising process is simply a Brownian motion on $\M$.

\subsection{\red{Time-reversal on Riemannian manifolds}}
\label{sec:brown-moti-comp}

\red{In order to use these noising processes we prove the time-reversal formula for manifolds, a generalisation of the results in the Euclidean case, e.g.\ see \citet[Theorem 4.9]{cattiaux2021time}. Consider an SDE of the form  $\rmd \bfX_t = b(\bfX_t) \rmd t + \rmd \bfB_t^\M$ where $\bfB_t^\M$ is a Brownian motion on $\M$. We refer to \Cref{sec:brown-moti-manif} for an introduction to Brownian motions on manifolds.}
This result shows that if $(\bfX_t)_{t \in \ccint{0,T}}$ is a diffusion process then $(\bfX_{T-t})_{t \in \ccint{0,T}}$ is also a diffusion process w.r.t.\ the backward filtration whose coefficients can be computed, and are shown in \cref{eq:time_reversal_manifold}. The proof relies on an extension of \citet[Theorem 4.9]{cattiaux2021time} to the Riemannian manifold case and is postponed to \cref{sec:time-reversal}.


\begin{theorem}{Time-reversed diffusion}{time_reversal_manifold}
  Let $T \geq 0$ and $(\bfB_t^\M)_{t \geq 0}$ be a Brownian motion on $\M$ such
  that $\bfB_0^\M$ has distribution the volume form $\piinv$\footnote{Note that
    in the case of a non-compact manifold $\piinv$ is only a measure and not a
    probability measure.}.  Let $(\bfX_t)_{t \in \ccint{0,T}}$ be associated
  with the SDE $\rmd \bfX_t = b(\bfX_t) \rmd t + \rmd \bfB_t^\M$.  Let
  $(\bfY_t)_{t \in \ccint{0,T}} = (\bfX_{T-t})_{t \in \ccint{0,T}}$ and assume
  that $\KLLigne{\Pbb}{\Qbb} < +\infty$, where $\Qbb$ is the distribution of
  $(\bfB_t^\M)_{t \in \ccint{0,T}}$ and $\Pbb $ the distribution of
  $(\bfX_t)_{t \in \ccint{0,T}}$. In addition, assume that
  $\Pbb_t=\mathcal{L}(\bfX_t)$, the distribution of $\bfX_t$, admits a smooth
  positive density $p_t$ w.r.t.\ $\piinv$ for any $t \in \ccint{0,T}$. Then,
  $(\bfY_t)_{t \in \ccint{0,T}}$ is associated with the SDE
      \begin{equation}
        \label{eq:time_reversal_manifold}
      \rmd \bfY_t = \{-b(\bfY_t) + \nabla \log p_{T-t}(\bfY_t)\} \rmd t + \rmd \bfB_t^\M. 
      \end{equation}
\end{theorem}

This result can easily be extended to the case where $(\bfB_t^\M)_{t \geq 0}$ is replaced by $(g(t)\bfB_t^\M)_{t \geq 0}$.

\subsection{Approximate sampling of diffusions} \label{sec:geodesic_random_walk}

\begin{figure}
    \centering
    \begin{subfigure}[t]{0.24\textwidth}
        \includegraphics[width=\textwidth, trim={4.5em 4.5em 0em 0em},clip]{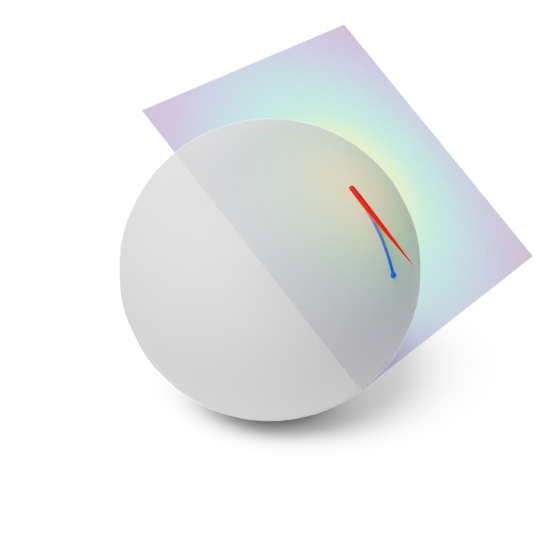}
        \caption{A single step of a Geodesic Random Walk.}
    \end{subfigure}
    \hfill
    \begin{subfigure}[t]{0.24\textwidth}
        \includegraphics[width=\textwidth, trim={2.5em 4.5em 2.0em 0em},clip]{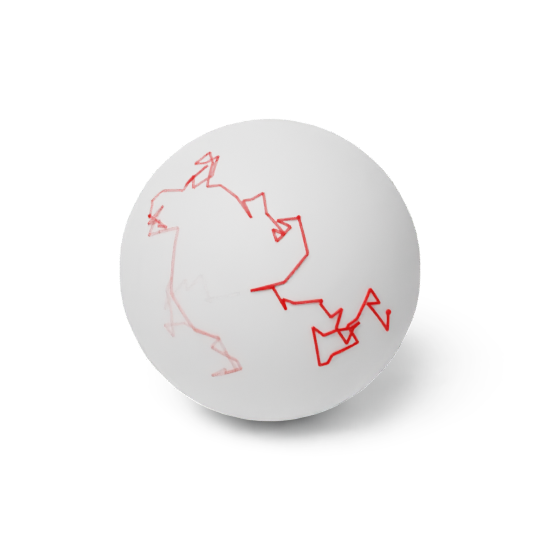}
        \caption{Many steps yield an approximate trajectory.}
    \end{subfigure}
    \hfill
        \begin{subfigure}[t]{0.48\textwidth}
        \includegraphics[width=\textwidth, trim={2.5em 2.5em 2.5em 2.5em},clip]{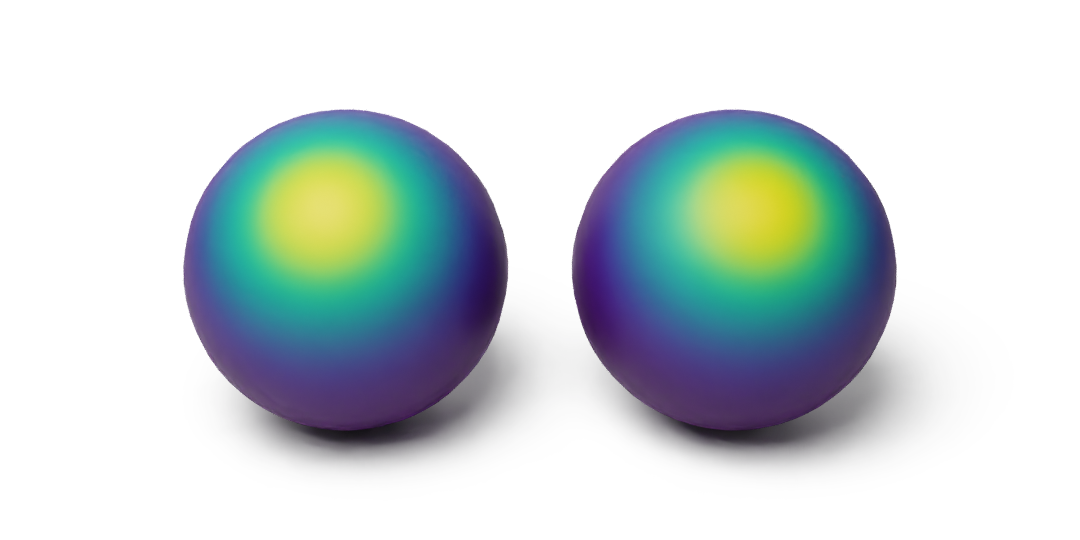}
        \caption{Gaussian Random Walk [Left] and the Brownian motion density [Right] agree well for small time steps.}
    \end{subfigure}
    \caption{Geodesic Random Walks can be used to approximate Brownian motion and more generally SDEs on manifolds. (a) At each step, tangential noise is sampled (red), which is added the drift term (not pictured). This tangent vector is then pushed through the exponential map to produce a geodesics step on the manifold (blue).
    (b) Iterating this procedure yield approximate sample paths from the process.
    }
    \label{fig:grw}
\end{figure}

Obtaining samples from SDEs on a manifold is non-trivial in
general. 
If $\M$ is isometrically embedded into $\rset^p$ (with $p \geq d$) one
can define $(\bfB_t^\M)_{t \geq 0}$ as a $\rset^p$-valued process, see
\Cref{sec:brown-moti-manif}. However, this approach is \emph{extrinsic}, as it
requires the knowledge of the projection
operator to place points back on the manifold at each step which can accumulate errors. 

Here we consider an \emph{intrisic} approach based on Geodesic Random Walks
(GRWs), see \textcite{jorgensen1975central} for a review of their properties. GRWs can
approximate \emph{any} well-behaved diffusion on $\M$.  Hence, we introduce GRWs
in a general framework and consider a discrete-time process
$(X_n^\gamma)_{n \in \nset}$ which approximates the diffusion $(\bfX_t)_{t \geq 0}$ defined by
\begin{equation}
  \label{eq:generic}
 \rmd \bfX_t = b(t, \bfX_t) \rmd t + \sigma(t, \bfX_t) \rmd \bfB_t^\M.
\end{equation}
This generalisation is key to sampling the backward diffusion process
defined in \Cref{thm:time_reversal_manifold}.

\begin{definition}{Geodesic Random Walk}{}
  Let $X^\gamma_0$ be a $\M$-valued random variable.  For any $\gamma > 0$, we
  define $(X_n^\gamma)_{n \in \nset}$ such that for any $n \in \nset$,
  $X_{n+1}^\gamma = \exp_{X_{n}^\gamma}[\gamma \{ b(X_{n}^\gamma) +
  \sqrt{\gamma} V_{n+1} \}]$, where
  $(V_n)_{n \in \nset}$ is a sequence of $\mathrm{T}\M$-valued random variables such that
  for any $n \in \nset$, $\CPELigne{V_{n+1}}{\mcf_n} = 0$ and
  $\CPELigne{V_{n+1} V_{n+1}^\top}{\mcf_n} = \sigma \sigma^\top(X_n^\gamma)$,
  where $\mcf_n$ is the filtration generated by $\{X_k^\gamma\}_{k=0}^n$.  We
  say that the $\M$-valued process $(X_n^\gamma)_{n \in \nset}$ is a Geodesic
  Random Walk.
\end{definition}

\cref{alg:grw} approximately simulates the diffusion
$(\bfX_t)_{t \in \ccint{0,T}}$ defined in \cref{eq:generic} using GRWs; see \textcite{kuwada2012convergence,cheng2022theorymanifold} for quantitative error bounds in the time-homogeneous
case and \Cref{sec:discr-bounds-grw} for a novel extentsion for the time-inhomogeneous case. \cref{fig:grw} provides a graphical illustration of this procedure.

\begin{algorithm}[!t]
\caption{\small GRW (Geodesic Random Walk)}
\label{alg:grw}
\begin{algorithmic}[1]
 \small
  \Require $T, N, X_0^\gamma, b, \sigma, \mathrm{P}$
  \State $\gamma = T / N$ \Comment Step-size
  \For{$k \in \{0, \dots, N-1\}$}
  \State $Z_{k+1} \sim \mathrm{N}(0, \Id)$ \Comment Sample a Gaussian in the tangent space of $X_{k}^\gamma$
  \State $W_{k+1} = \gamma b(k \gamma, X_k^\gamma) + \sqrt{\gamma} \sigma(k \gamma, X_k^\gamma) Z_{k+1}$ \Comment Compute the Euler--Maruyama step on tangent space 
  \State $X_{k+1}^\gamma = \exp_{X_k^\gamma}[W_{k+1}]$ \Comment Move along the geodesic defined by $W_{k+1}$ and $X_{k}^\gamma$ on $\M$
  \EndFor
  \State {\bfseries return} $\{ X_k^\gamma\}_{k=0}^{N}$
\end{algorithmic}
\end{algorithm}

\subsection{Score approximation on Riemannian manifolds}
\label{sec:riem-score-appr}
\paragraph{Score matching and loss functions.}
The reverse process from \cref{eq:time_reversal_manifold} involves the Stein score
$\nabla \log p_t$ 
which is unfortunately intractable.
To derive an approximation, we first remark that for any $s,t \in \ocint{0,T}$ with $t > s$ and $x_t \in \M$,
$p_t(x_t) = \int_{\M} p_{t|s}(x_t|x_s) \rmd \Pbb_s(x_s)$, where
$\Pbb_s = \mathcal{L}(\bfX_s)$, the distribution of $\bfX_s$. Thus, we have that for any
$s, t \in \ccint{0,T}$ with $t > s$ and $x_t \in \M$
  \begin{equation}
  \textstyle{
    \nabla_{x_t} \log p_t(x_t) = \int_{\M} \nabla_{x_t} \log p_{t|s}(x_t|x_s) \Pbb_{s|t}(x_t, \rmd x_s)  .
    }    
  \end{equation}
  Hence, for any $s, t \in \ccint{0,T}$ with $t > s$ we have that
  $\nabla \log p_t = \argmin
  \ensembleLigne{\ell_{t|s}(\mathbf{s}_t)}{\mathbf{s}_t \in \rmL^2(\Pbb_t)}$,

    where
    $\ell_{t|s}(\mathbf{s}_t) = \int_{\M^2} \normLigne{\nabla_x \log p_{t|s}(x_t|x_s) -
      \mathbf{s}_t(x_t)}^2 \rmd \Pbb_{s,t}(x_s,x_t)$,
    which is referred as the Denoising Score Matching (DSM)
    loss. It can also be written in an \emph{implicit} fashion.
    
    \begin{proposition}{}{implicit_der}
      Let $t, s \in \ocint{0,T}$ with $t>s$. Then, under sufficient regularity of $p_{t|s}(x_t|x_s)s(x_t)$, for any
      $\mathbf{s}_t \in \rmc^\infty(\M)$,
      $\ell_{t|s}(\mathbf{s}_t) = 2 \ellim_t(\mathbf{s}_t) + \int_{\M^2}
      \normLigne{\nabla_{x_t} \log p_{t|s}(x_t|x_s)}^2 \rmd
      \Pbb_{s,t}(x_s,x_t)$, where
      $\ellim_t(\mathbf{s}_t) = \int_\M \{
      \tfrac{1}{2}\normLigne{\mathbf{s}_t(x_t)}^2 + \dive(\mathbf{s}_t)(x_t) \}
      \rmd \Pbb_t(x_t)$.
    \end{proposition}

      The proof is postponed to \Cref{sec:implicit-losses}.
    For any $t \in \ocint{0,T}$ the minimizers of the loss $\ellim_t$ on $\XM$ (where $\XM$ is the set of vector fields on $\M$) 
    are the same as the ones for $\ell_{t|s}$. The loss $\ellim_t$ is referred to as the
\emph{implicit} score matching (ISM) loss~\citep{hyvarinen2005estimation}. These losses are direct analogous to the versions typically used in Euclidean space.
    
In the case where we have access to
$\ensembleLigne{\nabla \log p_{t|s}}{T \le t > s \ge 0}$, the forward noising process transition kernels, or an
approximation of this family, then we can use the DSM loss to learn $\ensembleLigne{\mathbf{s}_t \in \XM}{t \in \ccint{0,t}}$. If this is not the case then we turn to $\ellim_t$. Note that $\ellim_t$ requires the computation of a divergence term which requires $d$ Jacobian-vector calls. 
In high dimension, a stochastic estimator is necessary
\citep{hutchinson1989stochastic}. Following \textcite{song2020improved,
  nichol2021improved} the loss can be weighted with a term $\lambda_t > 0$.

\paragraph{Parametric family of vector fields.}
We approximate $(\nabla \log p_t)_{t \in \ccint{0,T}}$ by a family of functions
$\{\mathbf{s}_\theta\}_{\theta \in \Theta}$ where $\Theta$ is a set of
parameters and 
$\mathbf{s}_\theta: \ \ccint{0,T} \to \XM$.  In a Euclidean space, vector fields
are simply functions $\mathbf{s}_\theta: \R^d \rightarrow \R^d$.  In manifolds,
although for any $x \in \M$, $\mathrm{T}_x\M \cong \R^d$, there does not
necessarily exist a set of $d$ 
smooth vector fields $\{E_i\}_{i=1}^d$ such that
$\text{span}\left( \{E_i(x)\}_{i=1}^d \right) =
\mathrm{T}_x\M$~\citep[Chapter 8, page 179,][]{lee2006riemannian}~\footnote{Manifolds for which there exists such a \emph{global
    frame} $\{E_i(x)\}_{i=1}^d$ are referred as
  \emph{parallelizable}. $\mathbb{S}^2$ is a well-known example of
  \emph{non-parallelizable} manifold as per the \emph{Hairy ball theorem}.}.
Fortunately, one can rely on a larger set of smooth vector fields
$\{E_i(x)\}_{i=1}^n$ with $n > d$ that \emph{does} span the tangent bundle.
Then it suffices to construct a neural network
$\mathbf{s}_\theta: [0, T] \times \M \rightarrow \R^n$ to parametrise the score
network as
$\mathbf{s}_\theta(t,x) = \sum_{i=1}^n \mathbf{s}^i_\theta(t,x) E_i(x)$.
See \cref{sec:vector_field} for a discussion on the different choices of generating sets
$\{E_i(x)\}_{i=1}^n$.

Combining this parameterization with the score matching losses, the
time-reversal formula of \Cref{thm:time_reversal_manifold} and the sampling of
forward and backward processes described in \cref{sec:geodesic_random_walk}, we define our RGSM
algorithm in \cref{alg:rsgm}. This algorithm can also benefit from a predictor-corrector scheme as in \textcite{song2020score}, see
\Cref{sec:predictor-corrector}.

\begin{algorithm}[!t]
\caption{\small RSGM (Riemannian Score-Based Generative Model)}
\label{alg:rsgm}
\begin{algorithmic}[1]
  \small
  \Require $\vareps, T, N, \{X_0^m\}_{m=1}^M, \mathrm{loss}, \mathbf{s}, \theta_0, N_{\textrm{iter}}, \piinv, b_{\text{fwd}}, \mathrm{P}$
\State{/// TRAINING ///}

\For{$n \in \{0, \dots, N_{\textrm{iter}}-1\}$}  
\State $X_0 \sim (1/M) \sum_{m=1}^M \updelta_{X_0^m}$ \Comment Random mini-batch from dataset
\State $t \sim U(\ccint{\vareps, T})$ \Comment Uniform sampling between $\vareps$ and $T$
\State $\bfX_t = \textrm{GRW}(t, N, X_0, b, \Id, \mathrm{P})$ \Comment Approximate forward diffusion with \cref{alg:grw}
\State $\ell(\theta_n) = \ell_t(T, N, X_0, \bfX_t, \mathrm{loss}, \mathbf{s}_{\theta_n})$ \Comment Compute score matching loss from \cref{tab:sm_losses}
\State $\theta_{n+1} = \verb|optimizer_update|(\theta_n, \ell(\theta_n))$ \Comment ADAM optimizer step 
\EndFor
\State $\theta^\star = \theta_{N_{\textrm{epoch}}}$
\State{/// SAMPLING ///}
\State $Y_0 \sim \piinv$ \Comment Sample from uniform distribution
\State $b_\theta^\star(t, x) = -b(T-t, x) + \mathbf{s}_{\theta^\star}(T-t, x)$ for any $t \in \ccint{0,T}$, $x \in \M$ \Comment Reverse process drift
\State $\{Y_k\}_{k=0}^{N} = \textrm{GRW}(T, N, Y_0, b_{\theta^\star}, \Id, \mathrm{P})$ \Comment Approximate reverse diffusion with \Cref{alg:grw}
\State {\bfseries return} $\theta^\star, \{Y_k\}_{k=0}^{N}$
\end{algorithmic}
\end{algorithm}

\section{RSGMs on compact manifolds}
\label{sec:compact}

Assuming compactness of the manifold $\M$, we can leverage a number of
  special properties to implement a specific case of our algorithm.  In
  particular we benefit from the fact that on compact manifolds we have a proper
  \textit{uniform} distribution over the manifold, and have access to a variety
  of approximations of the heat kernel. As highlighted in
\cref{sec:non-compact}, in the compact setting we use Brownian motion as the
noising SDE, which targets the uniform distribution as the stationary
distribution. \cref{tab:difference} highlights the main differences between
RSGMs on compact manifolds, generic manifolds and Euclidean score-based models.

\begin{table}[tbh]
\small
\centering
\renewcommand*{\arraystretch}{1.2}
\begin{tabular}{lccc}
  \toprule 
  Ingredient \textbackslash ~Space  &       Euclidean                & `Generic' Manifold & Compact Manifold \\ \hline
  Forward process $\rmd \bfX_t=$ & $- \tfrac{1}{2} \bfX_t \rmd t + \rmd \bfB_t^\M$ & $ - \tfrac{1}{2} \nabla_{\bfX_t} U(\bfX_t) \rmd t + \rmd \bfB_t^\M$ & $\rmd \bfB_t^\M$ \\
  Easy-to-sample distribution & Gaussian & Wrapped Gaussian & Uniform \\
Time reversal  &  \textcite{cattiaux2021time} & \multicolumn{2}{c}{\Cref{thm:time_reversal_manifold}}   \\   
  Sampling forward process & Direct & \multicolumn{2}{c}{Geodesic Random Walk (\Cref{alg:grw})}\\
  Sampling backward process & Euler--Maruyama & \multicolumn{2}{c}{Geodesic Random Walk (\Cref{alg:grw})} \\
  \bottomrule
\end{tabular}
\vspace{.2cm}
\caption{\small Differences between SGM on Euclidean spaces and RSGM on Riemannian manifolds.}
\label{tab:difference}
\end{table}

\paragraph{Heat kernel on compact Riemannian manifolds.}
For any $x_0 \in \M$ and $t \geq s \geq 0$, the heat kernel
$p_{t|s}(\cdot|x_s)$ is defined as the density of $\bfB_t^\M$ w.r.t. the uniform measure on the manifold.

Contrary to the Gaussian transition density of the
OU process (or the Brownian motion) in the Euclidean setting,
it is typically only available as an infinite series. In order to circumvent
this issue we consider two techniques: \begin{enumerate*}[label=\roman*)]
\item a truncation approach, 
\item a Taylor expansion around $t=0$ called a Varadhan asymptotics.
\end{enumerate*}    
First, we recall that in the case of compact manifolds the heat kernel is given by the Sturm--Liouville decomposition \citep{chavel1984eigenvalues} given for any
$t > 0$ and $x_0, x_t \in \M$ by 
\begin{equation}
  \label{eq:infinite_sum1}
  \textstyle{p_{t|0}(x_t|x_0) = \sum_{j \in \nset} \rme^{-\lambda_j t} \phi_j(x_0)\phi_j(x_t),}
\end{equation}
where the convergence occurs in $\mathrm{L}^2(\piinv \otimes \piinv)$,
$(\lambda_j)_{j \in \nset}$ and $(\phi_j)_{j \in \nset}$ are the
eigenvalues, respectively the eigenvectors, of $-\Delta_\M$, the Laplace-Beltrami operator in the manifold, in
$\mathrm{L}^2(\piinv)$ \cite[Section 2]{saloff1994precise}.
When the eigenvalues and eigenvectors are known, we rely on an approximation of the
logarithmic gradient of $p_{t|0}$ by truncating the sum in
\cref{eq:infinite_sum1} with $J \in \nset$ terms to obtain for any
$t > 0$ and $x_0,x_t \in \M$
\begin{equation} \label{eq:heat_kernel_trunc}
  \nabla_{x_t} \log p_{t|0}(x_t|x_0) \approx \textstyle{S_{J,t}(x_0,x_t) \triangleq 
  \nabla_{x_t} \log {\sum_{j=0}^J  \rme^{-\lambda_j t} \phi_j(x_0)\phi_j(x_t)}
  . }
\end{equation}    

Under regularity conditions on $\M$ it can be shown that for any $x,y \in
\M$ and $t \geq 0$, $\lim_{J \to +\infty} S_{J,t}(x_0,x_t) = \nabla_{x_t} \log
p_{t|0}(x_t|x_0)$ \cite[Lemma 1]{jones2008Manifold}. In the case of the
$d$-dimensional torus or sphere the eigenvalues and eigenvectors are computable
\cite[Section 2]{saloff1994precise} and we can apply this method to
approximate $p_{t|0}$ for any $t >
0$, see 
\cref{sec:eigenf-eigenv-lapl}

When the eigenvalues and eigenvectors are unknown or not tractable, we
can still derive an approximation of the heat kernel for small times $t$. Using
Varadhan's asymptotics---see \citet[Theorem 3.8]{bismut1984large} or
\citet[Theorem 2.1]{chen2021logarithmic}---for any $x, y \in \M$ with
$y \notin \mathrm{Cut}(x)$ (where $\mathrm{Cut}(x)$ is the cut-locus of $x$
in $\M$ \cite[Chapter 10]{lee2018introduction}) we have that 
\begin{equation}
  \label{eq:varadhan}
  \textstyle{\lim_{t \to 0} t \nabla_{x_t} \log p_{t|0}(x_t|x_0) = \exp^{-1}_{x_t}(x_0) . }
\end{equation}

Using the previously defined score-matching losses and the approximations to the heat kernel above, we highlight three methods to compute $\nabla \log p_t$ in \cref{tab:sm_losses}.
\begin{table}[tbh]
\centering
\small
\renewcommand*{\arraystretch}{1.4}
\renewcommand*{\tabcolsep}{0.3em}
\begin{tabular}{lc|c|cc|c}
\toprule
\multirow{2}{5em}{Loss} &\multirow{2}{6em}{Approximation} & \multirow{2}{6em}{Loss function} &  \multicolumn{2}{c|}{Requirements} & \multirow{2}{5em}{Complexity} \\
& & & $p_{t|0}$ & $\exp^{-1}_{\bfX_t}$ &  \\
\midrule
\multirow{3}{4.5em}{$\ell_{t|0}$ (DSM)} & None & $\frac{1}{2} \E \left[ \| \mathbf{s}(\bfX_t) - \nabla \log  p_{t|0}(\bfX_t | \bfX_0) \|^2 \right]$ & \cmark & \xmark & $\mathcal{O}(1)$ \\
& Truncation~\eqref{eq:heat_kernel_trunc}  &  $\frac{1}{2} \E \left[ \| \mathbf{s}(\bfX_t) - S_{J,t}(\bfX_0,\bfX_t) \|^2 \right]$ &  \begin{tabular}{@{}c@{}}
\def\arraystretch{0.1}
asymptotic\vspace{-0.5em}\\expansion\end{tabular} & \xmark & $\mathcal{O}(1)$ \\
 & Varhadan~\eqref{eq:varadhan}  &  $\frac{1}{2} \E \left[ \| \mathbf{s}(\bfX_t) - \exp_{\bfX_t}^{-1}(\bfX_0) / t \|^2 \right]$ & \xmark & \cmark & $\mathcal{O}(1)$ \\ \midrule
$\ell_{t|s}$ (DSM) & Varhadan~\eqref{eq:varadhan}  &  $\frac{1}{2} \E \left[ \| \mathbf{s}(\bfX_t) - \exp^{-1}_{\bfX_t}(\bfX_s) / (t-s) \|^2 \right]$ & \xmark & \cmark & $\mathcal{O}(1)$  \\ \midrule
\multirow{2}{4.5em}{$\ellim_t$ (ISM)}  & Deterministic & $\E \left[\frac{1}{2} \| \mathbf{s}(\bfX_t) \|^2 + \dive( \mathbf{s})(\bfX_t)  \right]$  &  \xmark & \xmark & $\mathcal{O}(d)$ \\
& Stochastic & $\E \left[\frac{1}{2} \| \mathbf{s}(\bfX_t) \|^2 + \vareps^\top \partial \mathbf{s}(\bfX_t) \vareps  \right]$ &  \xmark & \xmark & $\mathcal{O}(1)$  \\ \bottomrule
\end{tabular} 
\vspace{-0.2cm}
\caption{\small
Computational complexity of score matching losses w.r.t. score network forward and backward passes. $\vareps$ is a random variable on $\mathrm{T}_{\bfX_t}\M$ such that $\expeLigne{\vareps}=0$ and $\expeLigne{\vareps\vareps^\top}=\Id$.
}
\label{tab:sm_losses}
\end{table}

\paragraph{Convergence results in the compact setting}
\label{sec:convergence-results}
We now provide a theoretical analysis of RSGM \red{under the assumption that $\M$ is compact.}
The following result ensures that RSGM generates samples whose distribution is close
to the data distribution $p_0$.
Let us denote $\{Y_k\}_{n \in \{0, \dots, N\}}$ the sequence generated by \Cref{alg:rsgm}. 
This result relies on the following assumption, which is satisfied for a large class of manifolds $\M$
such as the $d$-dimensional sphere and torus, compact matrix groups and
products of these manifolds.
\begin{assumption}{}{manifold}
    There exist $C, \alpha >0$ such that for any $t \in \ocint{0,1}$ and $x \in \M$,
    $p_{t|0}(x|x) \leq C t^{-\alpha/2}$, where $p_{t|0}(\cdot|x_0)$ is the  density of the heat kernel, i.e. the density of $\bfB_t^\M$ with initial condition $x_0$
    \footnote{The diagonal upper-bound is implied by  Sobolev inequalities which control of the growth of some functions by the growth of their gradient. \Cref{assum:manifold} is satisfied in our experiments, see  \textcite{saloff1994precise,gross1992logarithmic}.}.
\end{assumption}

\begin{theorem}{}{weak_qualitative}
    Assume \cref{assum:manifold}, that $p_0$ is smooth and positive and that
    there exists $\Mtt \geq 0$ such that for any $t \in \ccint{0,T}$ and
    $x \in \M$, $\norm{\mathbf{s}_{\theta^{\star}}(t,x) - \nabla \log p_{t}(x)} \leq \Mtt$,  with $\mathbf{s}_{\theta^\star} \in \rmc(\ccint{0,T}, \mathcal{X}(\M))$. Then if $T > 1/2$, there exists $C \geq 0$ independent on $T$ such that 
    \begin{equation}
        \textstyle{
         \wassersteinD[1](\mathcal{L}(Y_N), p_0) = C (  \rme^{-\lambda_1 T} + \sqrt{T/2}   \mathtt{M} + \rme^T \gamma^{1/2})  ,
        }
    \end{equation}
    where $\wassersteinD[1]$ is the Wasserstein distance of order one on the probability measures on $\M$.
\end{theorem}

The proof is postponed to \Cref{sec:convergence-rsgm}.  In particular, for any $\vareps > 0$, choosing $T > 0$ large enough, $\Mtt$ small enough (which can be achieved using the universal property of neural networks) and $\gamma$ small enough, we get that $\wassersteinD[1](\mathcal{L}(Y_N), p_0) \leq \vareps$. This result might seem weaker than the result obtained for Moser flows in \textcite[Theorem 3]{rozen2021moser}, but we emphasize that our bound takes into account the time-discretization contrary to \textcite{rozen2021moser} which considers the continuous-time flow. If we consider the time-reversed continuous-time SDE then we recover a bound in total variation distance, see \Cref{sec:convergence-rsgm}.  \red{Note that the upper bound $\Mtt$ encompasses both the bias introduced by the use of a neural network and the bias introduced by the use of an approximation of the score.}


\section{Related work}
\label{sec:related-works}

In this section we discuss previous work on parametrizing family of distributions for
manifold-valued data. 
Here, the manifold structure is considered to be prescribed, in contrast with methods that jointly learn the manifold structure and density~\parencite[e.g.][]{brehmer2020flows,caterini2021Rectangular}.



\textbf{Push-forward of Euclidean normalizing flows.}
More recently, approaches leveraging the flexibility of normalizing flows
\citep{papamakarios2019normalizing} have been proposed.
Following the wrapping method described above, these methods 
parametrize a normalizing flow in $\mathbb{R}^n$ before being pushed along an
invertible map $\psi: \mathbb{R}^n \rightarrow \M$.  However, to globally
represent the manifold, the map $\psi$ needs to be a homeomorphism, which can
only happen if $\M$ is topologically equivalent to $\mathbb{R}^n$, hence
limiting the scope of that approach.  One natural choice for this map is the
exponential map $\exp_x: \mathrm{T}_x \M \cong
\mathbb{R}^d$. 
This approach has been taken, for instance, by
\textcite{falorsi2019reparameterizing} and \textcite{bose2020latent}, respectively
parametrizing distributions on Lie groups and hyperbolic space.

\textbf{Neural ODE on manifolds.}
To avoid artifacts or numerical instabilities due to the manifold embedding, another line
of work uses tools from Riemannian geometry to define flows directly on the
manifold of interest
\citep{falorsi2020neural,mathieu2020riemannian,falorsi2021Continuous}.
Since these methods do not require a specific embedding mapping, they 
are referred as \emph{Riemannian}.
They extend continuous normalizing flows (CNFs) \citep{grathwohl2019Scalable} to the manifold setting, by implicity parametrizing flows as solutions of Ordinary Differential Equations (ODEs).
As such, the parametric flow is a \emph{continuous} function of time.
This approach has recently been extended by \textcite{rozen2021moser} introducing
Moser flows, whose main appeal being that it circumvents the need to solve an
ODE in the training process.  We refer to \Cref{sec:comp-with-moser} for an
in-depth discussion on the links between our work and Moser flows.

\textbf{Optimal transport on manifolds.}
Another line of work has developed flows on manifolds 
using tools from optimal transport. 
\cite{sei2013jacobian} introduced a flow that is given by $f_\theta: x \mapsto \exp_x(\nabla \psi^c_\theta)$ 
with $\psi^c_\theta$ a $c$-convex function and $c=d^2_\M$ the squared
geodesic distance.  This approach is motivated by the fact that the
optimal transport map takes such an expression
\citep{ambrosio2003Optimal}.  These methods operate directly on the manifold,
similarly to CNFs, yet in contrast they are \emph{discrete} in time.  The
benefits of this approach depend on the specific choice of parametric family of
$c$-convex functions \citep{rezende2021Implicit,cohen2021riemannian},
trading-off expressivity with scalability.


\section{Experiments}
\label{sec:experiments}
\begin{table}[tb]
\centering
\small
\setlength{\tabcolsep}{0.5em}
\renewcommand*{\arraystretch}{1.4}
\begin{tabular}{l|cccc}
\toprule
Method & Training & Likelihood evaluation & Sampling \\
\midrule
RCNF & Solving ODE $\mathcal{O}(dN)$ & Solving augmented ODE $\mathcal{O}(dN)$ & Solving ODE $\mathcal{O}(N)$ \\ 
Moser flow & Computing $\text{div}$ $\mathcal{O}(dk)$ or $\mathcal{O}(k)$ & Solving augmented ODE $\mathcal{O}(dN)$ & Solving ODE $\mathcal{O}(N)$ \\ 
RSGM & Score matching $\mathcal{O}(d)$ or $\mathcal{O}(1)$ & Solving augmented ODE $\mathcal{O}(dN)$ & Solving SDE $\mathcal{O}(N^*)$ \\ \bottomrule
\end{tabular} 

\caption{ \small Summary of computational complexity (w.r.t.\ neural network
  forward and backward passes) for different methods.  $d$ is the manifold
  dimension, $k$ the number of Monte Carlo batches in Moser flow's regularizer,
  $N$ is the number of steps in the (adaptive) ODE solver, whereas $N^*$ is the
  number of steps in the SDE Euler-Maruyama solver--which can usually be lower
  than $N$.  Moser flow and RSGM training complexity varies if the Hutchinson
  stochastic estimator is used.
See \cref{tab:sm_losses} for score matching losses complexity.
}
\label{tab:comparison_methods}
\end{table}
In this section we benchmark the empirical performance of RSGMs along with other manifold-valued methods introduced in \cref{sec:related-works}. \red{We also compare to a `Stereographic` score-based model, introduced in \cref{sec:stereo_exp}}.
First, we assess their modelling capacity on earth and climate science spherical data.
Then, we test the methods scalability with respect to manifold dimensions with a synthetic experiment on the torus $\mathbb{T}^d$.
Eventually, we evaluate the models' regularity and time complexity with a synthetic $\mathrm{SO}_3(\rset)$ target.
Experimental details are provided in \Cref{sec:exp_detail}. 
The code used to run the experiments can be found at
\href{https://github.com/oxcsml/riemannian-score-sde}{github.com/oxcsml/riemannian-score-sde}.


\subsection{Earth and climate science datasets on the sphere}
\label{sec:exp_sphere}
\begin{table}[tb]
    \centering
    \small
    \begin{tabular}{lrrrrr}
    \toprule
     Method & {Volcano} & {Earthquake} & {Flood} & {Fire} \\
    \midrule
    Mixture of Kent & $-0.80_{\pm 0.47}$ & $0.33_{\pm 0.05}$ & $0.73_{\pm 0.07}$ & $-1.18_{\pm 0.06}$ \\
Riemannian CNF            &                      $\bm{-6.05_{\pm 0.61}}$ &                             ${0.14_{\pm 0.23}}$ &                            ${1.11_{\pm 0.19}}$ &                         $\bm{-0.80_{\pm 0.54}}$ \\
Moser Flow                &                         ${-4.21_{\pm 0.17}}$ &                         $\bm{-0.16_{\pm 0.06}}$ &                         $\bm{0.57_{\pm 0.10}}$ &                         $\bm{-1.28_{\pm 0.05}}$ \\
Stereographic Score-Based &  \cellcolor{pearDark!20}${-3.80_{\pm 0.27}}$ &  \cellcolor{pearDark!20}$\bm{-0.19_{\pm 0.05}}$ &  \cellcolor{pearDark!20}$\bm{0.59_{\pm 0.07}}$ &  \cellcolor{pearDark!20}$\bm{-1.28_{\pm 0.12}}$ \\
Riemannian Score-Based    &  \cellcolor{pearDark!20}${-4.92_{\pm 0.25}}$ &  \cellcolor{pearDark!20}$\bm{-0.19_{\pm 0.07}}$ &  \cellcolor{pearDark!20}$\bm{0.45_{\pm 0.17}}$ &  \cellcolor{pearDark!20}$\bm{-1.33_{\pm 0.06}}$ \\
    \midrule 
    Dataset size & 827 & 6120 & 4875 & 12809 \\
    \bottomrule
    \end{tabular}
    \caption{
    Negative log-likelihood scores for each method on the earth and climate science datasets.
    Bold indicates best results (up to statistical significance).
    Means and confidence intervals are computed over 5 different runs.
    Novel methods are shown with blue shading.
    }
    \label{tab:geoscience}
\end{table}

We start by evaluating RSGMs on a collection of simple datasets, each containing an empirical
distribution of occurrences of earth and climate science events on the surface
of the earth. These events are: volcanic eruptions \cite{volcanoe_dataset},
earthquakes \cite{earthquake_dataset}, floods \citep{flood_dataset} and wild
fires \citep{fire_dataset}. We compare to previous baseline methods:
Riemannian Continuous Normalizing Flows \citep{mathieu2020riemannian}, Moser
Flows \citep{rozen2021moser} and a mixture of Kent distributions
\citep{peel2001fitting}.  Additionally, we consider a standard SGM on the 2D plane followed by the
inverse stereographic projection which induces a density on the sphere
\citep{gemici2016normalizing}.
We evaluate the log-likelihood of each model, extending to
the manifold setting the likelihood computation techniques of SGMs, see \Cref{sec:likel-comp}. 
We observe from \cref{tab:geoscience}, 
that all benchmarked methods have comparable performance when evaluated on these simple tasks with RSGM performing marginally better on most datasets. 
However, we empirically notice that Moser flows are slow to train and additionally that both Moser flows and stereographic SGMs are computationally expensive to evaluate.

\begin{figure}[bt]
    \centering
    \begin{subfigure}{0.22\textwidth}
        \includegraphics[width=\textwidth]{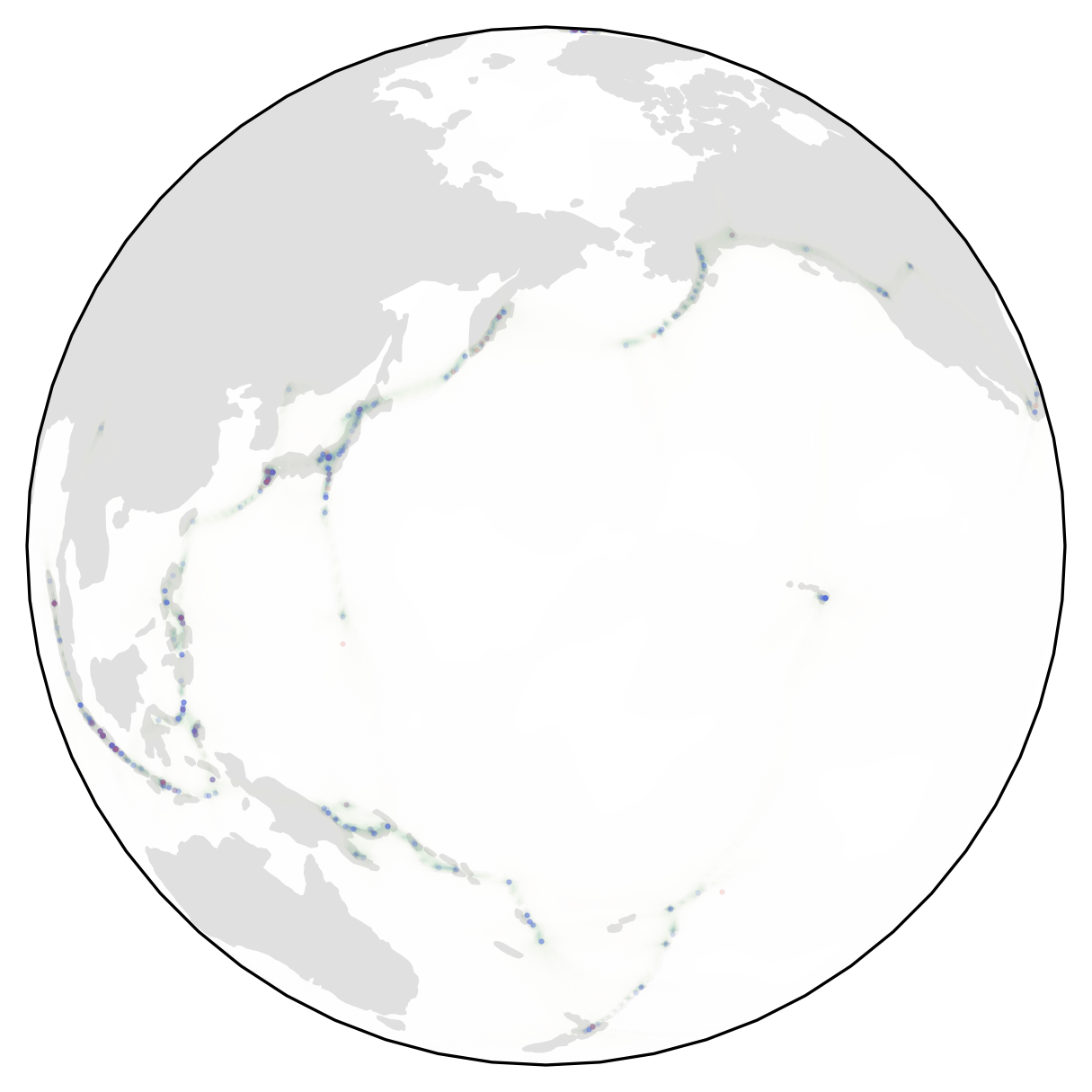}
        \caption{Volcano}
    \end{subfigure}
    \hfill
    \begin{subfigure}{0.22\textwidth}
        \includegraphics[width=\textwidth]{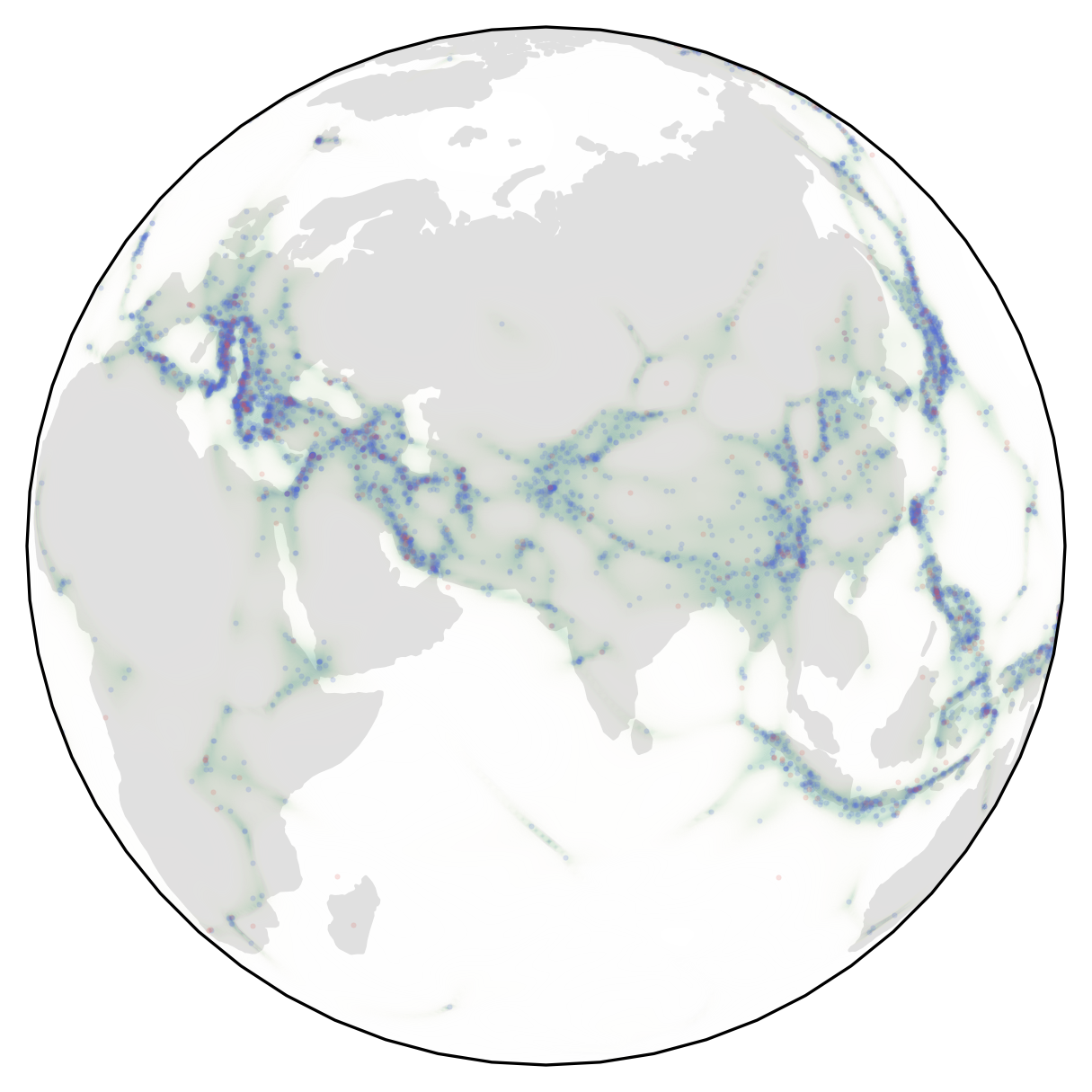}
        \caption{Earthquake}
    \end{subfigure}
    \hfill
    \begin{subfigure}{0.22\textwidth}
        \includegraphics[width=\textwidth]{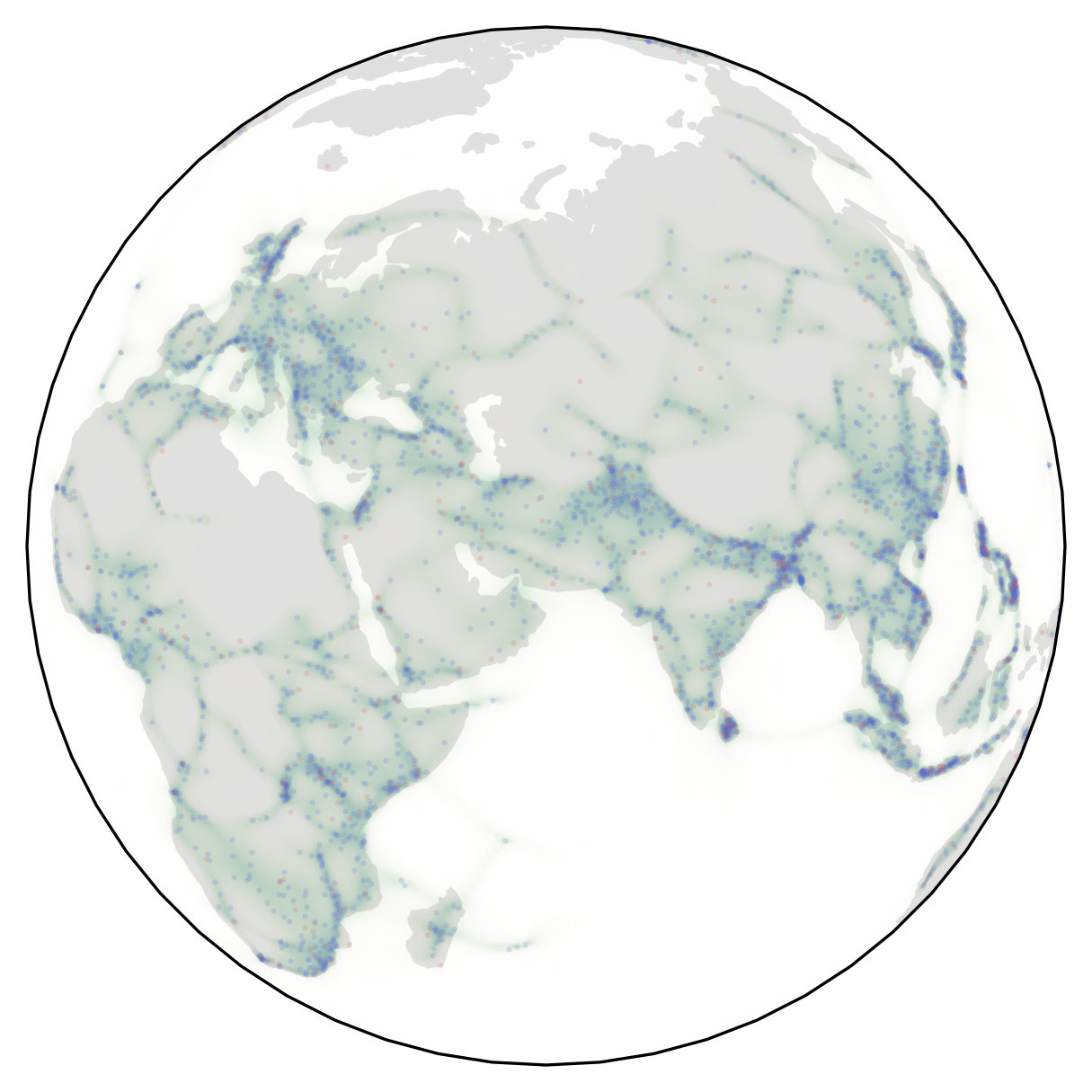}
        \caption{Flood}
    \end{subfigure}
    \hfill
    \begin{subfigure}{0.22\textwidth}
        \includegraphics[width=\textwidth]{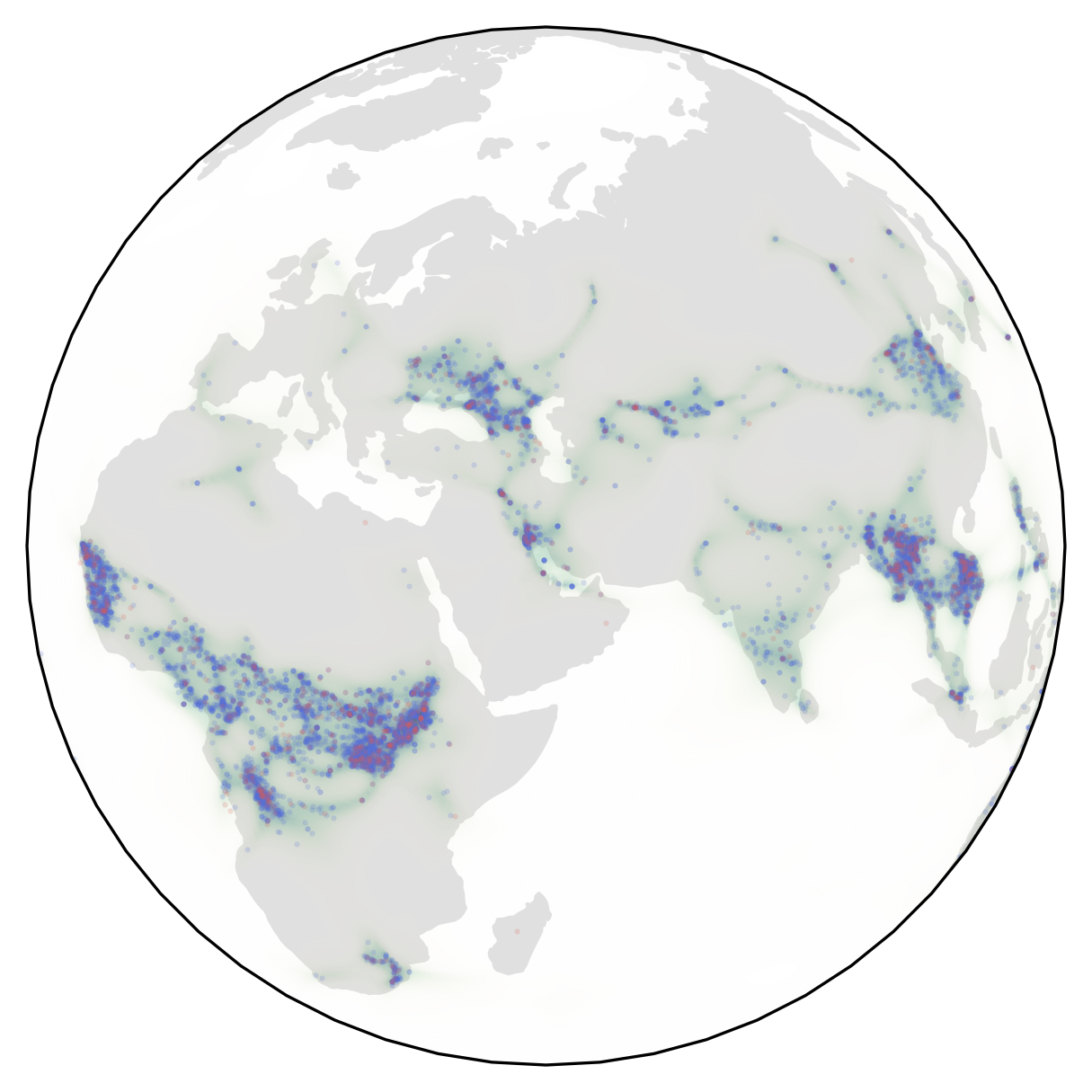}
        \caption{Fire}
    \end{subfigure}
    \caption{
        Trained score-based generative models on earth sciences data.
        The learned density is colored green-blue.
        Blue and red dots represent training and testing datapoints, respectively.
    }
    \label{fig:geoscience}
\end{figure}
\subsection{Synthetic data on \red{tori}}
\label{sec:exp_torus}
\begin{figure}[tb]
    \centering
    \includegraphics[width=\textwidth]{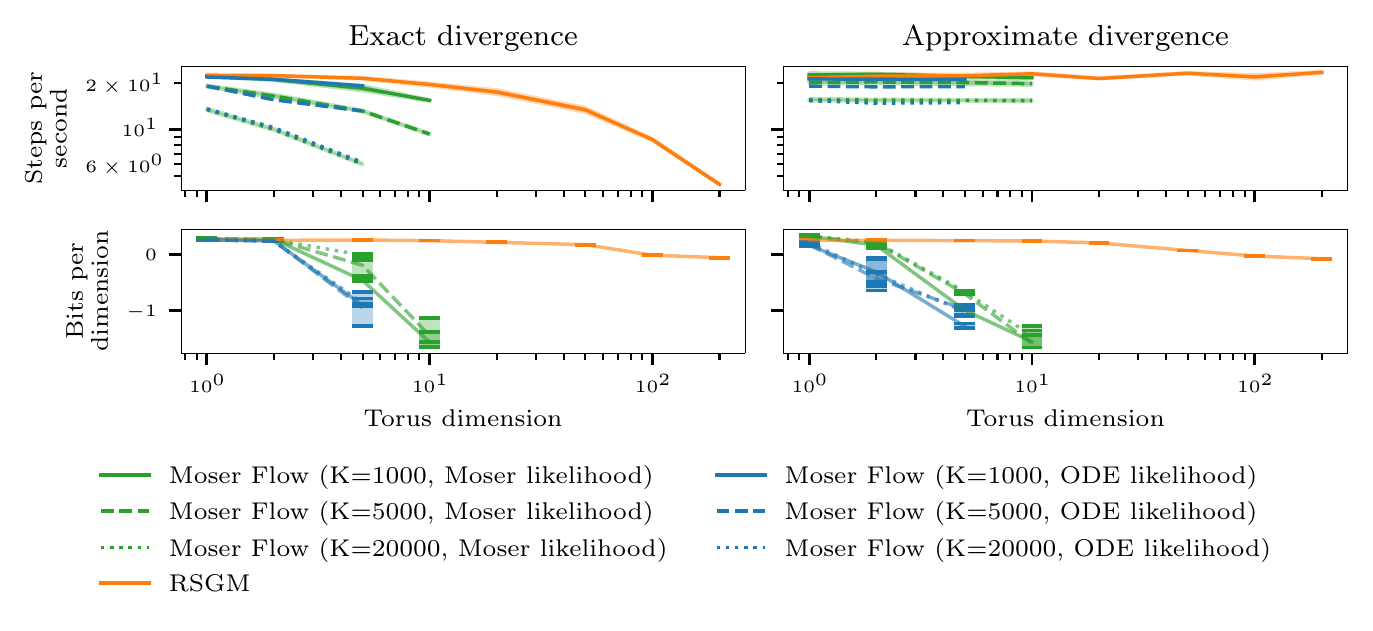}
    \caption{Comparison of Moser flows and RSGMs training speed and performance on the synthetic high-dimension torus task. Moser flows trained with $\lambda_{\min}=1$. We report two likelihoods, the `Moser' closed form density---not guaranteed to be normalized---and the `ODE' likelihood given by solving an augmented ODE (as in CNFs) with the vector field induced by the Moser flow density---which is guaranteed to have unit volume. 
    }
    \label{fig:high-dim}
\end{figure}

We now move to another manifold, that is the torus 
$\mathbb{T}^d = {\mathbb{S}^1 \times \dots \times \mathbb{S}^1}$, so as to
assess the scalability of the different methods with respect to the dimension
$d$. We consider a wrapped Gaussian target distribution on $\mathbb{T}^d$ with a random
mean and unit variance.  Moser flows'~\citep{rozen2021moser} loss involves a
regularization term which involves an integral over the manifold, approximated
by a Monte Carlo (MC) estimator with uniform proposal.
This term regularizes Moser flows towards probability measures, i.e.\ with unit volume.
We thus expect Moser flows to fail in high-dimension as the number of samples $K$ required for the MC estimator to be accurate will
grows as $\mathcal{O}(\rme^d)$, and the memory required to compute this estimator
grows either in $\mathcal{O}(Kd)$ for exact divergences or
$\mathcal{O}(K)$ for approximated divergences (see
\cref{tab:comparison_methods}).

In \Cref{fig:high-dim}, we observe that RSGMs are able to fit well the target distribution even in high dimension, with a linear or constant computational cost---depending on the divergence estimator.
In contrast, Moser flows scale poorly with the dimension, to the extent that we are unable to train them for $d \ge 10$.
This is due to the combination of the complexity which grows linearly with both the dimension $d$ and the number of MC samples $K$, which itself ought to grow exponentially with $d$---as discussed in the previous paragraph.
This is illustrated by the gap between the `Moser' and `ODE' likelihoods which increases with the manifold dimension (see left \Cref{fig:high-dim}).
%
%
%

\subsection{Synthetic data on the Special Orthogonal group}
\label{sec:exp_so3}
\vspace{-.1cm}
In order to demonstrate the broad range of applicability of our model we now
turn to the task of density estimation on the special orthogonal group
$\mathrm{SO}_d(\rset) = \ensembleLigne{\mathrm{Q} \in
  \mathrm{M}_d(\rset)}{\mathrm{Q} \mathrm{Q}^\top = \Id, \
  \det(\mathrm{Q})=1}$.
%
%
We consider the synthetic dataset consisting of samples in
$\mathrm{SO}_3(\rset)$ from a mixture of wrapped normal distributions with $M$ components.

%
%
We compare RSGMs against Moser flows and a wrapped-exponential baseline inspired by \textcite{falorsi2019reparameterizing}---where we parametrize a standard Euclidean SGM on $\mathfrak{so}(3)$ that is then pushed-forward on $\mathrm{SO}_3(\rset)$.
RSGMs are trained using the $\ell_{t|0}$ (DSM) loss with the Varadhan approximation (see \cref{tab:sm_losses}).
%
%
From \cref{tab:so3} we observe that, 
RSGMs perform consistently, whether the target distribution has few or many mixture components $M$, as opposed to Exp-wrapped SGMs and Moser flows which only perform well in some range of $M$.
Similarly to \cref{sec:exp_torus}, we find Moser
flows to be much slower to train due to the large number of Monte Carlo
samples needed in the reguralizer ($K=10^4$).  We also note from \cref{tab:so3} that the number of
score network evaluations (NFE) is significantly lower for RSGMs, and is
particularly detrimental for Moser flows ($\gg 10^3$).

%
%
\begin{figure}[tb]
    \centering
    \begin{subfigure}{0.49\textwidth}
        \includegraphics[width=\textwidth]{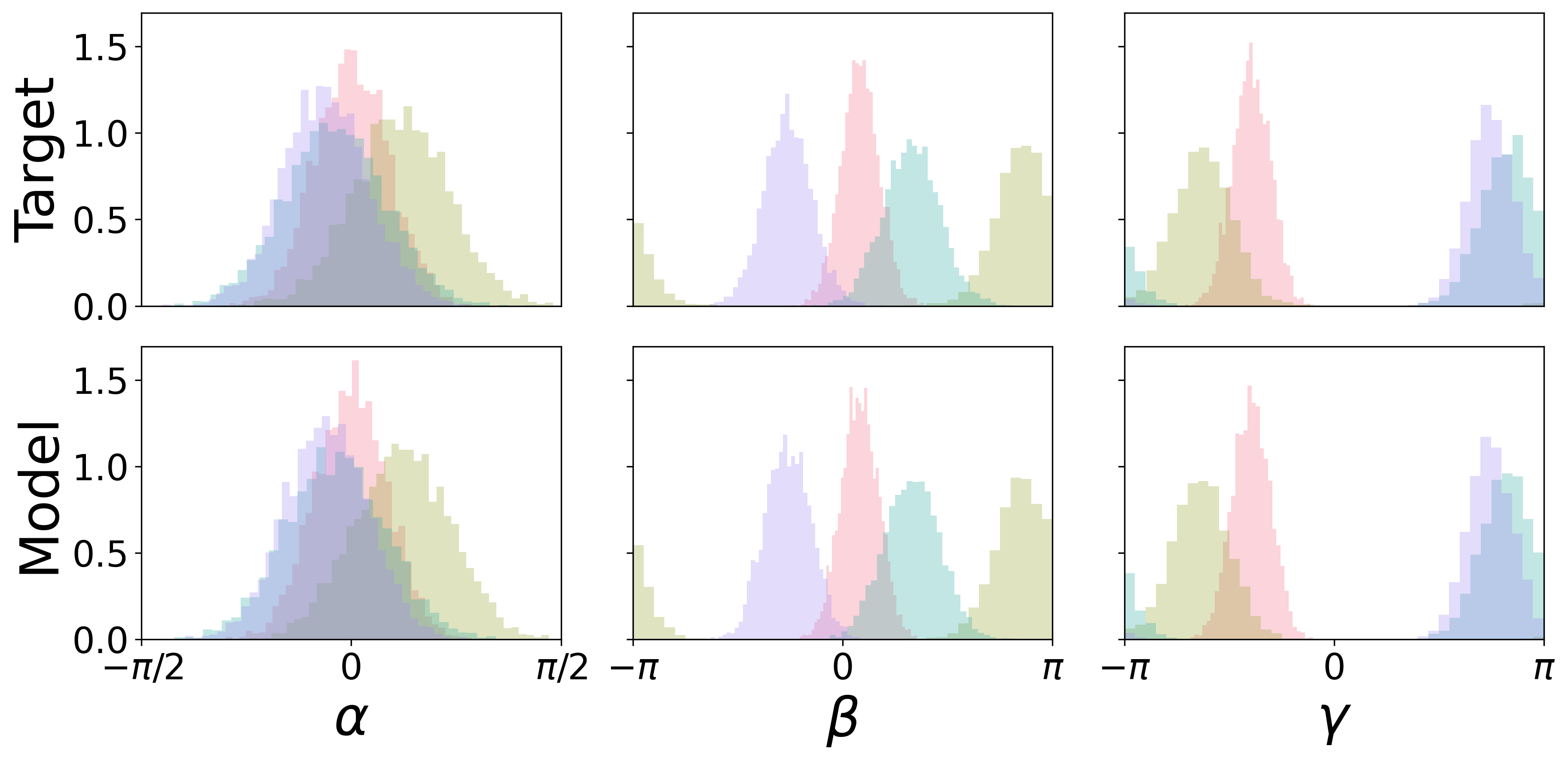}
        \caption{Histograms of $\mathrm{SO}_3(\rset)$ samples from a target mixture distribution with $M=4$ components, represented via their Euler angles.}
        \label{fig:so3_conditional}
    \end{subfigure}
    \hfill
    \begin{subfigure}{0.50\textwidth}
        \includegraphics[width=\textwidth]{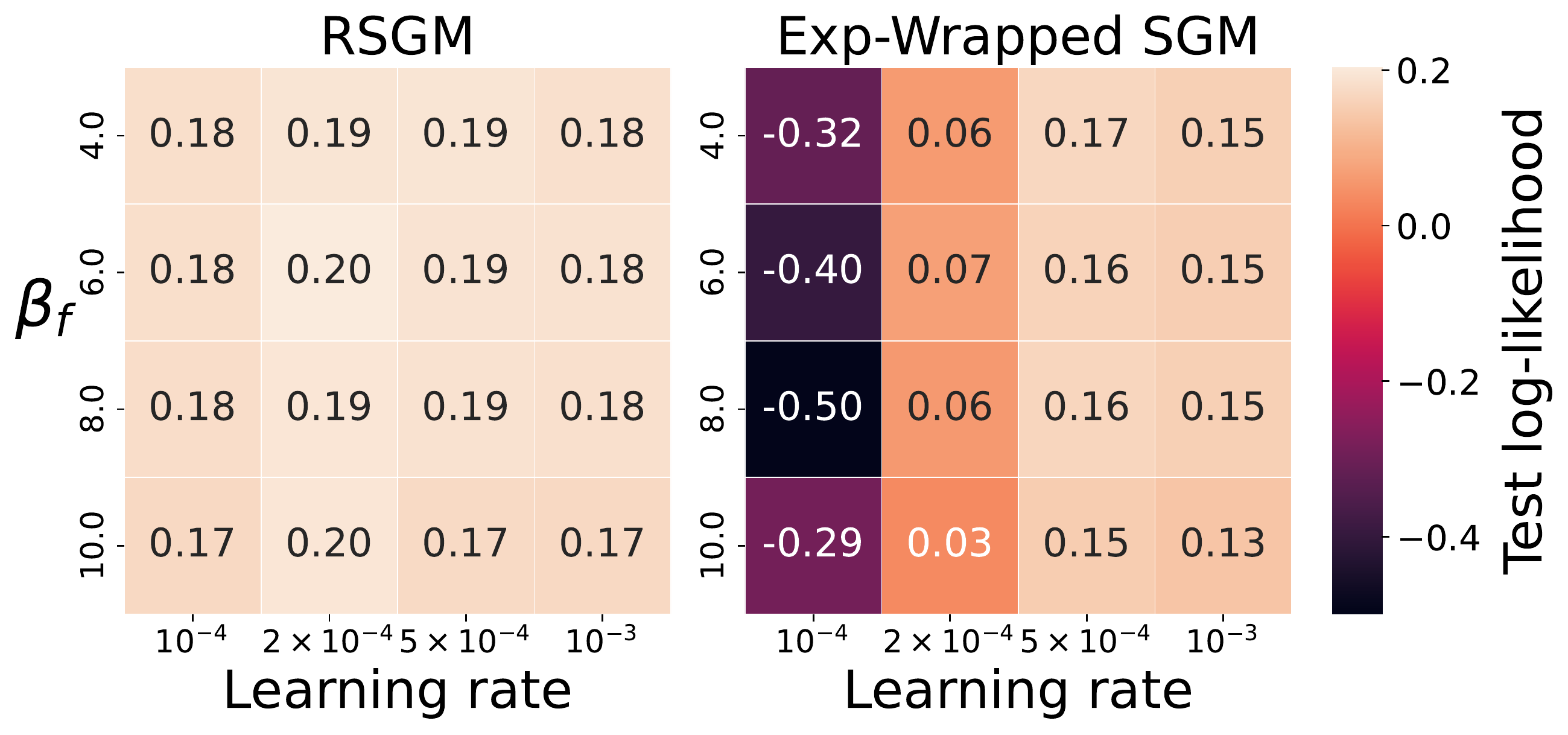}
        \caption{
        RSGMs are much more robust to hyperparameters than Exp-wrapped SGMs.
        The diffusion coefficient is given by $\sigma(t, \bfX_t) = \sqrt{\beta(t)}$, $\beta(t) = \beta_{0} + (\beta_{f} - \beta_{0})t$.
        }
        \label{fig:so3_heamtap}
    \end{subfigure}
    \vspace{-0.5em}
    \caption{
        Trained score-based generative models on synthetic $\mathrm{SO}_3(\rset)$ data.
    }
    \label{fig:so3}
\end{figure}
\begin{table}[htb]
    \setlength\tabcolsep{5pt}
    \centering
    \small
    \begin{tabular}{lrrrrrrrr}
    \toprule
    \multirow{2}{4.5em}{Method} & \multicolumn{2}{c}{$M=16$} & \multicolumn{2}{c}{$M=32$} &\multicolumn{2}{c}{$M=64$} \\ \cmidrule(lr){2-3} \cmidrule(lr){4-5} \cmidrule(lr){6-7}
    & log-likelihood & NFE & log-likelihood & NFE & log-likelihood & NFE\\
     
    \midrule

Moser Flow & ${0.85_{\pm 0.03}}$ & ${2.3_{\pm 0.5}}$ & ${0.17_{\pm 0.03}}$ & ${2.3_{\pm 0.9}}$ & $\bm{-0.49_{\pm 0.02}}$ & ${7.3_{\pm 1.4}}$ \\
Exp-wrapped SGM & \cellcolor{pearDark!20} $\bm{0.87_{\pm 0.04}}$ & \cellcolor{pearDark!20} ${0.5_{\pm 0.1}}$ & \cellcolor{pearDark!20} ${0.16_{\pm 0.03}}$ & \cellcolor{pearDark!20} ${0.5_{\pm 0.0}}$ & \cellcolor{pearDark!20} ${-0.58_{\pm 0.04}}$ & \cellcolor{pearDark!20} ${0.5_{\pm 0.0}}$ \\
RSGM & \cellcolor{pearDark!20} $\bm{0.89_{\pm 0.03}}$ & \cellcolor{pearDark!20} $\bm{0.1_{\pm 0.0}}$ & \cellcolor{pearDark!20} $\bm{0.20_{\pm 0.03}}$ & \cellcolor{pearDark!20} $\bm{0.1_{\pm 0.0}}$ & \cellcolor{pearDark!20} $\bm{-0.49_{\pm 0.02}}$ & \cellcolor{pearDark!20} $\bm{0.1_{\pm 0.0}}$ \\

    \bottomrule
    \end{tabular}
    \caption{
    Test log-likelihood and associated number of function evaluations (NFE) in $10^3$ on the synthetic mixture distribution with $M$ components on $\mathrm{SO}_3(\rset)$.
    Bold indicates best results (up to statistical significance).
    Means and standard deviations are computed over 5 different runs.
    Novel methods are shown with blue shading.
    }
    \label{tab:so3}
\end{table}
%

%


\subsection{Synthetic data on hyperbolic space}

Finally we demonstrate RSGM on a non-compact manifold: the two dimensional hyperbolic space $\mathbb{H}^2$, which is defined as the simply connected space of constant negative curvature.
We use Langevin dynamics as the noising process (\cref{eq:langevin}) and target a wrapped Gaussian as the invariant distribution.
We again consider a synthetic dataset of samples from a mixture of exp-wrapped normal distribution.
From \cref{fig:hyperbolic}, we can qualitatively see that both score-based models are able to fit the target distribution.

\begin{figure}[tb]
    \centering
    \hfill
    \begin{subfigure}[t]{0.28\textwidth}
        \includegraphics[width=\textwidth, trim={0.0em 0.0em 0.0em 5em},clip]{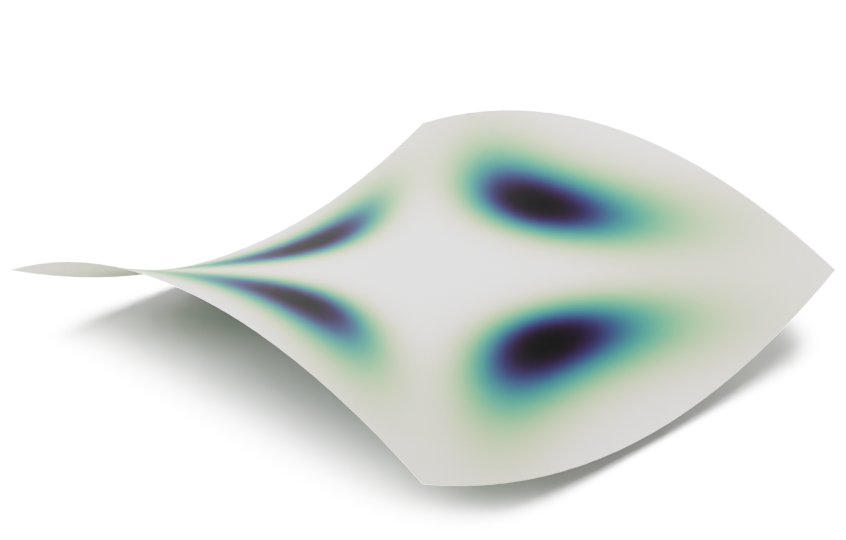}
        \caption{Target distribution.}
    \end{subfigure}
    \hfill
    \begin{subfigure}[t]{0.28\textwidth}
        \includegraphics[width=\textwidth, trim={0.0em 0.0em 0.0em 5em},clip]{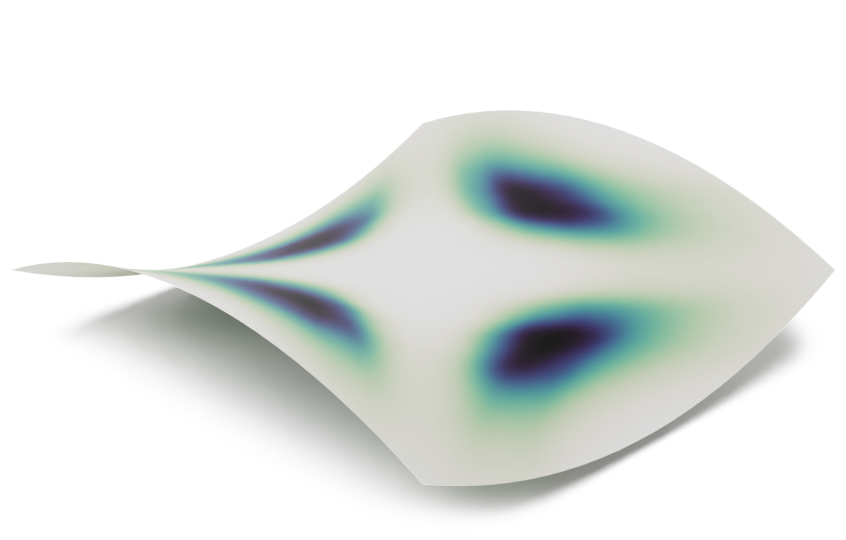}
        \caption{Exp-wrapped SGM.}
    \end{subfigure}
    \hfill
    \begin{subfigure}[t]{0.28\textwidth}
        \includegraphics[width=\textwidth, trim={0.0em 0.0em 0.0em 5em},clip]{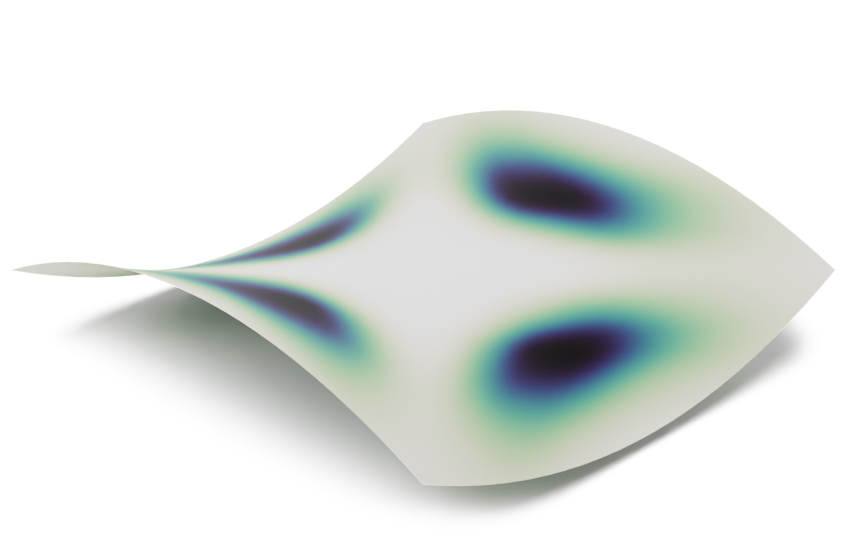}
        \caption{RSGM.}
    \end{subfigure}
    \hfill
     \caption{Samples from different probability distributions on $\mathbb{H}^2$ coloured w.r.t\ their density.}
    \label{fig:hyperbolic}
\end{figure}
\section{Discussion and limitations}
\label{sec:conclusion}

In this paper we introduced Riemannian Score-Based Generative Models (RSGMs), a class of deep generative models that represent target densities supported on manifolds, as the time-reversal of Langevin dynamics.
The main benefits of our method stems from its scalability to high dimensions, its applicability to a broad class of manifolds due to the diversity of available loss functions, its robustness and crucially its capacity to model complex datasets.
We also provided theoretical guarantees on the convergence of RSGMs.
In future work, we would like explore more generic classes of manifolds, such a ones with a boundary, along with alternative noising processes.
Another promising extension concerns stochastic control on manifolds and more precisely, deriving efficient algorithms to solve Schr\"odinger bridges \citep{thornton2022riemannian} in the same spirit as \textcite{debortoli2021neurips} on Euclidean state spaces.

\section*{Acknowledgements}
\label{sec:acknowledgments}

We are grateful to the anonymous reviewers for their insightful comments and the for fruitful discussion more generally.
We thank the $\texttt{hydra}$~\cite{Yadan2019Hydra}, $\texttt{jax}$~\cite{jax2018github} and $\texttt{geomstats}$~\cite{geomstats2020} teams, as our library is built on these great libraries.
EM research leading to these results received funding from the European Research Council under the European Union's Seventh Framework Programme (FP7/2007- 2013) ERC grant agreement no. 617071 and he acknowledges Microsoft Research and EPSRC for funding EM's studentship. MH is funded through the StatML CDT through grant EP/S023151/1. JT is funded through the OxWaSP CDT through grant EP/L016710/1. AD acknowledges support of the UK Defence Science and Technology Laboratory (Dstl) and
and Engineering and Physical Research Council (EPSRC) under grant EP/R013616/1. This is part of the collaboration between US DOD, UK MOD and UK EPSRC under the Multidisciplinary University Research Initiative. AD is also partially supported by the EPSRC grant EP/R034710/1 CoSines.


\printbibliography

\section*{Checklist}

\begin{enumerate}

\item For all authors...
\begin{enumerate}
  \item Do the main claims made in the abstract and introduction accurately reflect the paper's contributions and scope?
    \answerYes{Our main contribution is the extension of diffusion models on Riemannian manifolds.}
  \item Did you describe the limitations of your work?
    \answerYes{See \Cref{sec:conclusion}.}
  \item Did you discuss any potential negative societal impacts of your work?
    \answerNo{The work presented in this paper focuses on the learning of
      score-based models on manifold. We do not foresee any immediate societal
      impact of such a study.}
  \item Have you read the ethics review guidelines and ensured that your paper conforms to them?
    \answerYes{We have read the ethics review guidelines and our paper conforms to them.}
\end{enumerate}

\item If you are including theoretical results...
\begin{enumerate}
\item Did you state the full set of assumptions of all theoretical results?
  \answerYes{Yes, see \Cref{assum:manifold}.}
        \item Did you include complete proofs of all theoretical results?
    \answerYes{Yes, proofs are postponed to the supplementary material.}
\end{enumerate}

\item If you ran experiments...
\begin{enumerate}
  \item Did you include the code, data, and instructions needed to reproduce the main experimental results (either in the supplemental material or as a URL)?
    \answerYes{Experimental details are given in \cref{sec:exp_detail}.}
  \item Did you specify all the training details (e.g., data splits, hyperparameters, how they were chosen)?
    \answerYes{Experimental details are given in \cref{sec:exp_detail}.}
  \item Did you report error bars (e.g., with respect to the random seed after
    running experiments multiple times)?  \answerYes{Error bars are reported for
      each experiment.}
        \item Did you include the total amount of compute and the type of resources used (e.g., type of GPUs, internal cluster, or cloud provider)?
    \answerYes{Experimental details are given in \cref{sec:exp_detail}.}
\end{enumerate}

\item If you are using existing assets (e.g., code, data, models) or curating/releasing new assets...
\begin{enumerate}
  \item If your work uses existing assets, did you cite the creators?
    \answerYes{See \Cref{sec:exp_sphere}.}
  \item Did you mention the license of the assets?
    \answerYes{See \Cref{sec:exp_detail}.}
  \item Did you include any new assets either in the supplemental material or as a URL?
    \answerNo{Not applicable.}
  \item Did you discuss whether and how consent was obtained from people whose data you're using/curating?
    \answerNo{Not applicable.}
  \item Did you discuss whether the data you are using/curating contains personally identifiable information or offensive content?
    \answerNo{Not applicable.}
\end{enumerate}

\item If you used crowdsourcing or conducted research with human subjects...
\begin{enumerate}
  \item Did you include the full text of instructions given to participants and screenshots, if applicable?
    \answerNo{Not applicable.}
  \item Did you describe any potential participant risks, with links to Institutional Review Board (IRB) approvals, if applicable?
    \answerNo{Not applicable.}
  \item Did you include the estimated hourly wage paid to participants and the total amount spent on participant compensation?
    \answerNo{Not applicable.}
\end{enumerate}

\end{enumerate}



\setcounter{equation}{0}
\setcounter{figure}{0}
\setcounter{table}{0}
\setcounter{page}{1}
\makeatletter
\setcounter{tocdepth}{1}

\newpage
\appendixhead
\appendix
\section{Organization of the supplementary}
\label{sec:organ-suppl}

In this supplementary we first introduce notation in \Cref{sec:notation-1}. We
gather the proof of \Cref{thm:time_reversal_manifold} as well as additional
derivations on score-based generative models and Riemannian manifolds. In
\cref{sec:prel-stoch-riem}, we recall basics on stochastic Riemannian geometry
following \textcite{hsu2002stochastic}.
In \cref{sec:likel-comp}, we introduce an extension to the Riemannian setting of
the likelihood computation techniques in diffusion models. Details about
parametric vector fields are given in \Cref{sec:vector_field}.  In
\cref{sec:eigenf-eigenv-lapl}, we recall some basic facts about eigenvalues and
eigenfunctions of the Laplace--Beltrami operator on the $d$-dimensional sphere
and torus.  We present an extension of \Cref{alg:rsgm} using predictor-corrector
schemes in \Cref{sec:predictor-corrector}.  In \cref{sec:time-reversal}, we
prove the extension of the time-reversal formula to manifold in \Cref{thm:time_reversal_manifold}. We prove the convergence of RSGM,
i.e. \Cref{thm:weak_qualitative}, in \Cref{sec:convergence-rsgm}. The proof of
\cref{prop:implicit_der} drawing links between the denoising score matching loss
and the implicit score matching loss is presented \cref{sec:implicit-losses}. We
provide a thorough comparison between our approach and the one of
\textcite{rozen2021moser} in \Cref{sec:comp-with-moser}. We show how our method can
be adapted to perform density estimation in
\Cref{sec:dens-estim-with}. 
Extensions to conditional SGM and Schr\"odinger Bridges are discussed in \Cref{sec:extensions}. In \Cref{sec:non-compact}, we briefly discuss the non compact setting. 
Details on the stereographic SGM are given in \Cref{sec:stereo_exp}. 
Experimental details are given in \cref{sec:exp_detail}.


\section{Notation}
\label{sec:notation-1}

We refer to \Cref{sec:prel-stoch-riem} for more details about the basic concepts
of Riemannian geometry and stochastic processes. In this section, we merely
introduce the notation used in our work. We postpone an introduction to
stochastic processes on manifolds to \Cref{sec:stoch-diff-equat}.

In this work we always consider a smooth, connected and complete manifold $\M$.
We focus on the case of Riemannian manifolds, namely manifolds equipped with a
metric $g$. Metrics $g$ are smooth scalar product on the manifold allowing us to
define the notion of \emph{distance} on a manifold. We refer to
\Cref{sec:prel-stoch-riem} for a precise definition and a discussion on
metrics. Given a smooth map $f \in \rmc^\infty(\M, \rset)$, the gradient
$\nabla f$ is defined for any $f:\M \rightarrow \R$,
$x \in \M, v \in \mathrm{T}_x\M$, $\langle \nabla f, v \rangle_g = \rmd f (v)$.
The distane $d_\M(x,y)$ is defined as the infimum of the length of all the
curves on $\M$ joining $x$ and $y$. Geodesics are path defined on $\M$ by a
second order equation (and a starting point and speed). This second order
equation corresponds to the first order minimization of an \emph{energy}
functional whose minimizers also minimize the length. In
\Cref{sec:prel-stoch-riem}, we introduce the notion of geodesics using parallel
transport. The exponential mapping $\exp_x: \msu \M \to \M$ with
$\msu \subset \mathrm{T}_x \M$ is such that $\exp_x(v) = \gamma(1)$ with
$\gamma(1)$ the geodesics with initial condition $(x,v)$ at time $t=1$. Finally
the volume form is a differentiable form of same degree as the dimension of
$\M$. Since $\M$ is an orientable Riemannian manifold there is a natural volume
form defined using the metric $g$, namely
$\omega(x) = \abs{g(x)}^{1/2} \rmd x_1 \wedge \dots \rmd x_d$. In this paper, we
abuse notation and denote by the volume form this natural volume form. 


\section{Preliminaries on stochastic Riemannian geometry}
\label{sec:prel-stoch-riem}

In this section, we recall some basic facts on Riemannian geometry and
stochastic Riemannian geometry.  We follow
\textcite{hsu2002stochastic,lee2018introduction,lee2006riemannian} and refer to
\textcite{lee2010introduction,lee2013smooth} for a general introduction to
topological and smooth manifolds. Throughout this section $\M$ is a
$d$-dimensional smooth manifold, $\TM$ its tangent bundle and $\TMstar$ it
cotangent bundle. We denote $\rmc^\infty(\M)$ the set of real-valued smooth
functions on $\M$ and $\XM$ the set of vector fields on $\M$.

\subsection{Tensor field, metric, connection and transport}
\label{sec:metr-conn-tens}

\paragraph{Tensor field and Riemannian metric}

For a vector space $V$ let
$\mathrm{T}^{k, \ell}(V) = V^{\otimes k} \otimes (V^\star)^{\otimes \ell}$ with
$k, \ell \in \nset$. For any $k, \ell \in \nset$ we define the space of
$(k,\ell)$-tensors as
$\mathrm{T}^{k,\ell} \M = \sqcup_{p \in \M}
\mathrm{T}^{k,\ell}(\mathrm{T}_p\M)$. Note that
$\Gamma(\M, \mathrm{T}^{0,0}\M) = \mathrm{C}^\infty(\M)$,
$\XM = \Gamma(\M, \mathrm{T}^{1,0} \M)$ and that the space of $1$-form on $\M$
is given by $\Gamma(\M, \mathrm{T}^{0,1} \M)$, where $\Gamma(\M, V(\M))$ is a
section of a vector bundle $V(\M)$ \citep[see][Chapter 10]{lee2013smooth}.  For
any $k \in \nset$, we denote
$\mathrm{T}^{\abs{k}} \M = \sqcup_{j=0}^k \mathrm{T}^{j,k-j} \M$.
$\M$ is said to be
a Riemannian manifold if there exists $g \in \Gamma(\M, \mathrm{T}^{0,2} \M)$ such that for
any $x \in \M$, $g(x)$ is positive definite. $g$ is called the Riemannian metric
of $\M$. Every smooth manifold can be equipped with a Riemannian metric
\cite[see][Proposition 2.4]{lee2018introduction}. In local coordinates we define
$G = \{g_{i,j}\}_{1 \leq i,j \leq d} = \{g(X_i, X_j)\}_{1 \leq i,j \leq d}$,
where $\{X_i\}_{i=1}^d$ is a basis of the tangent space. In what follows we
consider that $\M$ is equipped with a metric $g$ and for any $X, Y \in \XM$ we
denote $\langle X,Y \rangle_{\M} = g(X,Y)$.

\paragraph{Connection}
A connection $\nabla$ is a mapping which allows one to differentiate vector
fields w.r.t other vector fields. $\nabla$ is a linear map
$\nabla: \ \XM \times \XM \to \XM$. In addition, we assume that
\begin{enumerate*}[label=\roman*)]
\item for any $f \in \rmc^\infty(\M)$, $X, Y \in \XM$, $\nabla_{f X}(Y) = f \nabla_X Y$, 
\item for any $f \in \rmc^\infty(\M)$, $X, Y \in \XM$, $\nabla_{X}(fY) = f \nabla_X Y + X(f) Y$.
\end{enumerate*}
Given a system of local coordinates, the Christoffel symbols
$\{\Gamma_{i,j}^k\}_{1 \leq i,j,k\leq d}$ are given for any
$i,j \in \{1, \dots, d\}$ by
$\nabla_{X_i}X_j = \sum_{k=1}^d \Gamma_{i,j}^k X_k$. We
also define the Levi--Civita connection $\nabla$ by considering the additional
two conditions: 
\begin{enumerate*}[label=\roman*)]
\item $\nabla$ is torsion-free, \ie \ for any $X, Y \in \XM$ we have
  $\nabla_X Y - \nabla_Y X = [X,Y]$, where $[X,Y]$ is the Lie bracket between
  $X$ and $Y$,
\item $\nabla$ is compatible with the metric $g$, \ie \ for any $X,Y,Z \in \XM$,
  $X (\langle Y,Z \rangle_\M) = \langle\nabla_X Y, Z\rangle_\M + \langle Y, \nabla_X Z \rangle_\M$.
\end{enumerate*}
We recall that the Levi--Civita connection is uniquely defined since for any
$X,Y,Z \in \XM$ we have
\begin{align}
  2 \prodM{\nabla_X Y}{Z} &= X(\prodM{Y}{Z}) + Y(\prodM{Z}{X}) - Z(\prodM{X}{Y}) \\
  & \qquad \qquad + \prodM{[X,Y]}{Z} - \prodM{[Z,X]}{Y} - \prodM{[Y,Z]}{X}  . 
\end{align}
In this case, the Christoffel symbols are given for any
$i,j,k \in \{1, \dots, d\}$ by
\begin{equation}
  \textstyle{\Gamma_{i,j}^k = \tfrac{1}{2} \sum_{m=1}^d g^{km} (\partial_j g_{m,i} + \partial_i g_{m,j} - \partial_m g_{i,j}) ,}
\end{equation}
where $\{g^{i,j}\}_{1 \leq i,j \leq d} = G^{-1}$. Note that if $\M$ is Euclidean
then for any $i,j,k \in \{1, \dots, d\}$, $\Gamma_{i,j}^k = 0$. We also extend
the connection so that for any $X \in \XM$ and $f \in \rmc^\infty(M)$ we have
$\nabla_X f = X(f)$. In particular, we have that
$\nabla_X f \in \rmc^\infty(\M)$. In addition, we extend the connection such
that for any $\alpha \in \Gamma(\M, \mathrm{T}^{0,1} \M)$, $X,Y \in \XM$ we have
$\nabla_X \alpha (Y) = \alpha(\nabla_X Y) - X(\alpha(Y))$. In particular, we
have that $\nabla_X \alpha \in \Gamma(\M, \mathrm{T}^{1,0} \M)$. Note that for any
$X \in \XM$ and $\alpha, \beta \in \mathrm{T}^{\abs{1}} \M$ we have
$\nabla_X (\alpha \otimes \beta) = \nabla_X \alpha \otimes \beta + \alpha
\otimes \nabla_X \beta$. Similarly, we can define recursively $\nabla_X \alpha$
for any $\alpha \in \Gamma(\M, \mathrm{T}^{k,\ell}\M)$ with $k, \ell \in \nset$. Such an
extension is called a covariant derivative.

\paragraph{Parallel transport, geodesics and exponential mapping} Given a
connection, we can define the notion of parallel transport, which transports
vector fields along a curve. Let $\gamma: \ \ccint{0,1} \to \M$ be a smooth
curve. We define the covariant derivative along the curve $\gamma$ by
$D_{\dot{\gamma}}: \ \Xgamma \to \Xgamma$ similarly to the connection, where
$\Xgamma = \Gamma(\gamma(\ccint{0,1}), \TM)$. In particular if $\dot{\gamma}$
and $X \in \Xgamma$ can be extended to $\XM$ then we define
$D_{\dot{\gamma}}(X) = \nabla_{\dot{\gamma}}X \in \XM$. In what follows, we
denote $D = \nabla$ for simplicity. We say that $X \in \Xgamma$ is parallel to
$\gamma$ if for any $t \in \ccint{0,1}$, $\nabla_{\dot{\gamma}}X(t) = 0$. In
local coordinates, let $X \in \Xgamma$ be given for any $t \in \ccint{0,1}$ by 
$X = \sum_{i=1}^d a_i(t) E_i(t)$ (assuming that $\gamma([0,1])$ is entirely
contained in a local chart), then we have that for any $t \in \ccint{0,1}$ and
$k \in \{1, \dots, d\}$
\begin{equation}
  \label{eq:parallel_transport}
  \textstyle{\dot{a}_k(t) + \sum_{i,j=1}^d \Gamma_{i,j}^k(x(t)) \dot{x}_i(t) a_j(t) = 0  .}
\end{equation}
A curve $\gamma$ on $\M$ is said to be a geodesics if $\dot{\gamma}$ is parallel
to $\gamma$. Using \cref{eq:parallel_transport} we get that
\begin{equation}
  \label{eq:geodesics}
  \textstyle{\ddot{x}_k(t) + \sum_{i,j=1}^d \Gamma_{i,j}^k(x(t)) \dot{x}_i(t) \dot{x}_j(t) = 0  .}
\end{equation}
For more details on geodesics and parallel transport, we refer to \citet[Chapter
4]{lee2018introduction}. 
In addition, we have that parallel transport provides a linear isomorphism
between tangent spaces. Indeed, let $v \in \mathrm{T}_x \M$ and
$\gamma: \ \ccint{0,1} \to \M$ with $\gamma(0) = x$ a smooth curve. Then, there
exists a unique vector field $X^v \in \Xgamma$ such that $X^v(x) = v$ and $X^v$
is parallel to $\gamma$. For any $t \in \ccint{0,1}$, we denote
$\Gamma_0^t: \mathrm{T}_{x} \M \to \mathrm{T}_{\gamma(t)} \M$ the linear
isomorphism such that $\Gamma_0^t(v) = X^v(\gamma(t))$.

For any $x \in \M$ and $v \in \mathrm{T}_x \M$ we denote
$\gamma^{x,v}: \ \ccint{0,\vareps^{x,v}}$ the geodesics (defined on the maximal
interval $\ccint{0, \vareps^{x,v}}$) on $\M$ such that $\gamma(0) = x$ and
$\dot \gamma(0) = v$. We denote
$\msu^x = \ensembleLigne{v \in \mathrm{T}_x \M}{\vareps^{x,v} \geq 1}$. Note
that $0 \in \msu^x$. For any $x \in \M$, we define the exponential mapping
$\exp_x: \ \msu^x \to \M$ such that for any $v \in \msu^x$,
$\exp_x(v) = \gamma^{x,v}(1)$. If for any $x \in \M$,
$\msu^x = \mathrm{T}_x \M$, the manifold is called \emph{geodesically
  complete}. As any connected compact manifold is geodesically complete, there exists a geodesic between any two
points $x, y \in \M$ \cite[see][Lemma 6.18]{lee2018introduction}. For any
$x, y \in \M$, we denote $\mathrm{Geo}_{x,y}$ the sets of geodesics $\gamma$
such that $\gamma(0) = x$ and $\gamma(y) = 1$. For any $x, y \in \M$ we denote
$\Gamma_x^y(\gamma) : \ \mathrm{T}_x \M \to \mathrm{T}_y \M$ the linear
isomorphism such that for any $v \in \mathrm{T}_x \M$,
$\Gamma_x^y(v) = X^v(\gamma(1))$, where $\gamma \in \mathrm{Geo}_{x,y}$. Note
that for any $x \in \M$ there exists $\msv^x \subset \M$ such that
$x \in \msv^x$ and for any $y \in \msv^x$ we have that
$\absLigne{\mathrm{Geo}_{x,y}}=1$.  In this case, we denote
$\Gamma_x^y = \Gamma_x^y(\gamma)$ with $\gamma \in \mathrm{Geo}_{x,y}$.

\paragraph{Orthogonal projection} We will make repeated use of orthonormal
projections on manifolds. Recall that since $\M$ is a closed Riemannian manifold
we can use the Nash embedding theorem \citep{gunther1991isometric}. In the rest
of this paragraph, we assume that $\M$ is a Riemannian submanifold of $\rset^p$
for some $p \in \nset$ such that its metric is induced by the Euclidean
metric. In order to define the projection we introduce
\begin{equation}  
  \mathrm{unpp}(\M) = \ensembleLigne{x \in \rset^d}{\text{there exists a unique $\xi_x$ such that $\normLigne{x - \xi_x} = d(x, \M)$}}  . 
\end{equation}
Let $\mathcal{E}(\M) = \interior(\mathrm{unpp}(\M))$. By \citet[Theorem
1]{leobacher2021existence}, we have $\M \subset \mathcal{E}(\M)$. We define
$\tilde{p}: \ \mathcal{E}(\M) \to \M$ such that for any $x \in \mathcal{E}(\M)$,
$\tilde{p}(x) = \xi_x$. Using \citet[Theorem 2]{leobacher2021existence}, we have
$\tilde{p} \in \rmc^\infty(\rset^p, \M)$ and for any $x \in \M$,
$\tilde{P}(x) = \rmd \tilde{p}(x)$ is the orthogonal projection on
$\mathrm{T}_x\M$. Since $\rset^p$ is normal and $\M$ and
$\mathcal{E}(\M)^\complementary$ are closed, there exists $\msf$ open such that
$\M \subset \msf \subset \mathcal{E}(\M)$. Let
$p \in \rmc^\infty(\rset^p, \rset^p)$ such that for any $x \in \msf$,
$p(x) = \tilde{p}(x)$ (given by Whitney extension theorem for
instance). Finally, we define $P: \ \rset^p \to \rset^p$ such that for any
$x \in \rset^p$, $P(x) = \rmd p(x)$. Note that for any $x \in \M$, $P(x)$ is the
orthogonal projection $\mathrm{T}_x \M$ and that
$P \in \rmc^\infty(\rset^p, \rset^p)$.

\subsection{Stochastic Differential Equations on manifolds}
\label{sec:stoch-diff-equat}

\paragraph{Stratanovitch integral} For reasons that will become clear in the
next paragraph, it is easier to define Stochastic Differential Equations (SDEs)
on manifolds w.r.t the Stratanovitch integral \cite[Part II, Chapter
3]{kloeden:platen:2011}. We consider a filtered probability space
$(\Omega, (\mcf_t)_{t \geq 0}, \Pbb)$. Let $(\bfX_t)_{t \geq 0}$ and
$(\bfY_t)_{t \geq 0}$ be two real continuous semimartingales. We define the
quadratic covariation $([\bfX,\bfY]_t)_{t \geq 0}$ such that for any $t \geq 0$
\begin{equation}
  \textstyle{[\bfX,\bfY]_t = \bfX_t \bfY_t - \bfX_0\bfY_0 - \int_0^t \bfX_s \rmd \bfY_s - \int_0^t \bfY_s \rmd \bfX_s  . }
\end{equation}
We refer to \citet[Chapter IV]{revuz1999continuous} for more details on
semimartingales and quadratic variations. We denote $[\bfX] = [\bfX, \bfX]$. In
particular, we have that $([\bfX, \bfY]_t)_{t \geq 0}$ is an adapted continuous
process with finite-variation and therefore $[[\bfX, \bfY]] = 0$. Let
$(\bfX_t)_{t \geq 0}$ and $(\bfY_t)_{t \geq 0}$ be two real continuous
semimartingales, then we define the Stratanovitch integral as follows for any
$t \geq 0$
\begin{equation}
  \textstyle{ \int_0^t \bfX_s \circ \rmd \bfY_s = \int_0^t \bfX_s \rmd \bfY_s + \tfrac{1}{2} [\bfX, \bfY]_t  . }
\end{equation}
In particular, denoting $(\bfZ_t^1)_{t \geq 0}$ and $(\bfZ_t^2)_{t \geq 0}$ the
processes such that for any $t \geq 0$,
$\bfZ_t^1 = \int_0^t \bfX_s \circ \rmd \bfY_s$ and
$\bfZ_t^2 = \int_0^t \bfX_s \rmd \bfY_s$, we have that $[\bfZ^1] = [\bfZ^2]$. We
refer to \textcite{kurtz1995stratonovich} for more details on Stratanovitch
integrals. Note that if for any $t \geq 0$,
$\bfX_t = \int_0^t f(\bfX_s) \circ \rmd \bfY_s$ with $\rmc^1(\rset, \rset)$,
then $[\bfX, \bfY]_t = \int_0^t f(\bfX_s) f'(\bfX_s) \rmd \bfY_s$. Assuming that
$f \in \rmc^3(\rset, \rset)$ we have that \cite[Chapter IV, Exercise
3.15]{revuz1999continuous}
\begin{equation}
  \label{eq:stratanovitch_lemma}
  \textstyle{ f(\bfX_t) = f(\bfX_0) + \int_0^t f'(\bfX_s) \circ \rmd \bfX_s  .}
\end{equation}
The proof relies on the fact that for any $t \geq 0$,
$\rmd [\bfX, f'(\bfX)]_t = f''(\bfX_t) \rmd [\bfX]_t$.  This result should be
compared with It\^o's lemma. In particular, Stratanovitch calculus satisfies the
ordinary chain rule making it a useful tool in differential geometry which makes
a heavy use of diffeomorphism.  Finally, we have the following correspondence
between Stratanovitch and It\^o SDEs. Assume that $(\bfX_t)_{t \in \ccint{0,T}}$
is a strong solution to
$\rmd \bfX_t = b(t, \bfX_t) \rmd t + \sigma(t, \bfX_t) \circ \rmd \bfB_t$, with
$b \in \rmc^\infty(\rset^d, \rset^d)$ and
$\sigma \in \rmc^\infty(\rset^d, \rset^{d \times d})$. Then, we have that
\begin{equation}
  \label{eq:strata_ito_transfo_sde}
 \rmd \bfX_t = \{ b(t, \bfX_t) + \bar{b}(\bfX_t)\}\rmd t + \sigma(t, \bfX_t)
\rmd \bfB_t  , \qquad \bar{b} = (-1/2) [\dive(\sigma \sigma^\top) - \sigma \dive(\sigma^\top)] .
\end{equation}
where for any $\rmA \in \rmc^\infty(\rset^d, \rset^{d \times d})$ we have that
$\dive(\rmA) \in \rmc^\infty(\rset^d, \rset^d)$ and for any
$i \in \{1, \dots, d\}$ and $x \in \rset^d$,
$\dive(\rmA)_i(x) = \sum_{j=1}^d \partial_j \rmA_{i,j}(x)$. In particular, note
that if for $x_0 \in \rset^d$, $\sigma(x_0)$ is an orthogonal projection, then
$\sigma(x_0) \bar{b}(x_0) = 0$.

\paragraph{SDEs on manifolds}
We define semimartingales and SDEs on manifold through the lens of their actions
on functions. A continuous $\M$-valued stochastic process $(\bfX_t)_{t \geq 0}$
is called a $\M$-valued semimartingale if for any $f \in \rmc^\infty(\M)$ we
have that $(f(\bfX_t))_{t \geq 0}$ is a real valued semimartingale. Let
$\ell \in \nset$, $V^{1:\ell} = \{ V_i\}_{i=1}^\ell \in \XM^\ell$ and
$Z^{1:\ell} = \{Z^i\}_{i=1}^\ell$ a collection of $\ell$ real-valued
semimartingales. A $\M$-valued semimartingale $(\bfX_t)_{t \geq 0}$ is said to
be the solution of $\SDE(V^{1:\ell}, Z^{1:\ell}, \bfX_0)$ up to a stopping
$\tau$ with $\bfX_0$ a $\M$-valued random variable if for all
$f \in \rmc^\infty(\M)$ and $t \in \ccint{0, \tau}$ we have 
\begin{equation}
  \textstyle{f(\bfX_t) = f(\bfX_0) + \sum_{i=1}^\ell \int_0^t V_i(f)(\bfX_s) \circ \rmd \bfZ^i_s  . } 
\end{equation}
Since the previous SDE is defined w.r.t the Stratanovitch integral we have that
if $(\bfX_t)_{t \geq 0}$ is a solution of $\SDE(V^{1:\ell}, Z^{1:\ell}, \bfX_0)$
and $\Phibf: \M \to \mathcal{N}$ is a diffeomorphism then $(\Phibf(\bfX_t))_{t \geq 0}$
is a solution of $\SDE(\Phibf_\star V^{1:\ell}, Z^{1:\ell}, \Phibf(\bfX_0))$,
where $\Phibf_\star$ is the pushforward operation \cite[see][Proposition
1.2.4]{hsu2002stochastic}. Because the vector fields $\{V_i\}_{i=1}^\ell$ are
smooth we have that for any $\ell \in \nset$,
$V^{1:\ell} = \{ V_i\}_{i=1}^\ell \in \XM^\ell$ and
$Z^{1:\ell} = \{Z^i\}_{i=1}^\ell$ a collection of $\ell$ real-valued
semimartingales, there exists a unique solution to
$\SDE(V^{1:\ell}, Z^{1:\ell}, \bfX_0)$ \cite[see][Theorem
1.2.9]{hsu2002stochastic}.

\subsection{Brownian motion on manifolds}
\label{sec:brown-moti-manif}

In this section, we introduce the notion of Brownian motion on manifolds. We
derive some of its basic convergence properties and provide alternative
definitions (stochastic development, isometric embedding, random walk
limit). These alternative definitions are the basis for our alternative
methodologies to sample from the time-reversal. To simplify our discussion, we
assume that $\M$ is a connected compact orientable Riemannian manifold equipped
with the Levi--Civita connection $\nabla$. We denote $p_{\textup{ref}}^m$ the
Haussdorff measure of the manifold (which coincides with the measure associated
with the Riemannian volume form \citep[see][Theorem
2.10.10]{federer2014geometric} and
$p_{\textup{ref}} = p_{\textup{ref}}^m / p_{\textup{ref}}(\M)$ the associated
probability measure.

\paragraph{Gradient, divergence and Laplace operators}
Let $f \in \rmc^{\infty}(\M)$. We define $\nabla f \in \XM$ such that for any
$X \in \XM$ we have $\langle X, \nabla f \rangle_{\M} = X(f)$. Let
$\{X_i\}_{i=1}^d \in \XM^d$ such that for any $x \in \M$, $\{X_i(x)\}_{i=1}^d$
is an orthonormal basis of $\mathrm{T}_x \M$. Then, we define
$\dive: \ \XM \to \rmc^\infty(\M)$ (linear) 
such that for any $X \in \XM$,
$\dive(X) = \sum_{i=1}^d \prodM{\nabla_{X_i}X}{X_i}$. The following Stokes
formula (also called divergence theorem, see \citet[p.51]{lee2018introduction})
holds for any $f \in \rmc^\infty(\M)$ and $X \in \XM$,
$\int_{M} \dive(X)(x) f(x) \rmd p_{\textup{ref}}(x) = - \int_M X(f)(x) \rmd
p_{\textup{ref}}(x)$. Let $X = \sum_{i=1}^d a_i X_i$ in local coordinates.  Using the
Stokes formula and the definition of the gradient we get that in local
coordinates
\begin{equation}
\textstyle{  \nabla f = \sum_{i,j=1}^d g^{i,j} \partial_i f X_j  ,  \qquad \dive(X) = \det(G)^{-1/2} \sum_{i=1}^d \partial_i(\det(G)^{1/2} a_i)  . }
\end{equation}
The Laplace--Beltrami operator is given by
$\Delta_{\M} : \ \rmc^\infty(M) \to \rmc^\infty(M)$ and for any
$f \in \rmc^\infty(M)$ by $\Delta_{\M}(f) = \dive(\grad(f))$. In local
coordinates we obtain
$\Delta_{\M}(f) = \det(G)^{-1/2} \sum_{i=1}^d \partial_i (\det(G)^{1/2}
\sum_{j=1}^d g^{i,j} \partial_j f)$. Using the Nash isometric embedding theorem
\citep{gunther1991isometric} we will see that $\Delta_{\M}$ can always be
written as a sum of squared operators. However, this result requires an
\emph{extrinsic} point of view as it relies on the existence of projection
operators. In contrast, if we consider the orthonormal bundle $\OM$, see
\cite[Chapter 2]{hsu2002stochastic}, we can define the Laplace--Bochner operator
$\Delta_{\OM}: \ \rmc^\infty(\OM) \to \rmc^\infty(\OM)$ as
$\Delta_{\OM} = \sum_{i=1}^d H_i^2$, where we recall that for any
$i \in \{1, \dots, d\}$, $H_i$ is the horizontal lift of $e_i$. In this case,
$\Delta_{\OM}$ is a sum of squared operators and we have that for any
$f \in \rmc^\infty(\M)$, $\Delta_{\OM}(f \circ \pi) = \Delta_{\M}(f)$
\cite[see][Proposition 3.1.2]{hsu2002stochastic}. Being able to express the
various Laplace operators as a sum of squared operators is key to express the
associated diffusion process as the solution of an SDE.

\paragraph{Alternatives definitions of Brownian motion}

We are now ready to define a Brownian motion on the manifold $\M$. Using the
Laplace--Beltrami operator, we can introduce the Brownian motion through the lens
of diffusion processes.

\begin{definition}{Brownian motion}{}
  Let $(\bfB_t^\M)_{t \geq 0}$ be a $\M$-valued semimartingale.
  $(\bfB_t^\M)_{t \geq 0}$ is a Brownian motion on $\M$ if for any
  $f \in \rmc^\infty(\M)$, $(\bfM_t^f)_{t \geq 0}$ is a local martingale where
  for any $t \geq 0$
  \begin{equation}
    \textstyle{\bfM_t^f = f(\bfB_t^\M) - f(\bfB_0^\M) - \tfrac{1}{2}\int_0^t \Delta_{\M}f(\bfB_s^\M) \rmd s  .}
  \end{equation}
\end{definition}

Note that this definition is in accordance with the definition of the Brownian
motion as a diffusion process in the Euclidean space $\rset^d$, since in this
case $\Delta_{\M} = \Delta$. A key property of frame bundles and orthonormal
bundles is that any semimartingale on $\M$ can be associated to a process on
$\FM$ (or $\OM$) and a process on $\rset^d$. The proof of the following result
can be found in \citet[Propositions 3.2.1 and 3.2.2]{hsu2002stochastic}.

\begin{proposition}{Intrinsic view of Brownian motion}{intrinsic_brownian}
  Let $(\bfB_t^\M)_{t \geq 0}$ be a $\M$-valued semimartingales. Then
  $(\bfB_t^\M)_{t \geq 0}$ is a Brownian motion on $\M$ if and only on the
  following conditions hold:
  \begin{enumerate}[label=\alph*)]
  \item The horizontal lift $(\bfU_t)_{t \geq 0}$ is a $\Delta_{\OM}/2$
    diffusion process, \ie \ for any $f \in \rmc^\infty(\OM)$, we have that
    $(\bfM_t^f)_{t \geq 0}$ is a local martingale where for any $t \geq 0$
  \begin{equation}
    \textstyle{\bfM_t^f = f(\bfU_t) - f(\bfU_0) - \tfrac{1}{2}\int_0^t \Delta_{\OM}f(\bfU_s) \rmd s  .}
  \end{equation}    
\item The stochastic antidevelopment of $(\bfB_t^\M)_{t \geq 0}$ is a
  $\rset^d$-valued Brownian motion $(\bfB_t)_{t \geq 0}$.
  \end{enumerate}
\end{proposition}

In particular the previous proposition provides us with an \emph{intrisic} way
to sample the Brownian motion on $\M$ with initial condition $\bfB_0^\M$. First
sample $(\bfU_t)_{t \geq 0}$ solution of $\SDE(H^{1:d}, \bfB^{1:d}, \bfU_0)$
with $H^{1:d} = \{H_i\}_{i=1}^d$ and $\pi(\bfU_0) = \bfB_0^\M$ and $\bfB^{1:d}$ the
Euclidean $d$-dimensional Brownian motion. Then, we recover the $\M$-valued
Brownian motion $(\bfB_t^\M)_{t \geq 0}$ upon letting
$(\bfB_t^\M)_{t \geq 0} = (\pi(\bfU_t))_{t \geq 0}$.

We now consider an \emph{extrinsic} approach to the sampling of Brownian motions
on $\M$. Using the Nash embedding theorem \citep{gunther1991isometric}, there
exists $p \in \nset$ such that without loss of generality we can assume that
$\M \subset \rset^p$. For any $x \in \M$, we denote
$\rmP(x): \ \rset^p \to \mathrm{T}_x \M$ the projection operator. In addition for
any $x \in \M$, we denote $\{\rmP_i(x)\}_{i=1}^p = \{\rmP(x) e_i\}_{i=1}^p$, where
$\{e_i\}_{i=1}^p$ is the canonical basis of $\rset^p$. For any
$i \in \{1, \dots, p\}$, we smoothly extend $\rmP_i$ to $\rset^p$. In this case, we
have the following proposition \cite[Theorem 3.1.4]{hsu2002stochastic}:

\begin{proposition}{Extrinsic view of Brownian motion}{extrinsic_brownian}
  For any $f \in \rmc^{\infty}(\M)$ we have that
  $\Delta_\M(f) = \sum_{i=1}^p \rmP_i(\rmP_i(f))$. Hence, we have that
  $(\bfB_t^\M)_{t \geq 0}$ solution of
  $\SDE(\{\rmP_i\}_{i=1}^{p}, \bfB^{1:p}, \bfB_0^\M)$ with $\bfB_0^\M$ a $\M$-valued
  random variable and $\bfB^{1:p}$ a $\rset^p$-valued Brownian motion.
\end{proposition}

The second part of this proposition, stems from the fact that any solution of
$\SDE(\{V_i\}_{i=1}^{\ell}, \bfB^{1:\ell}, \bfX_0)$, where $\bfX_0$ is a
$\M$-valued random variable and $\bfB^{1:\ell}$ a $\rset^\ell$-valued Brownian
motion is a diffusion process with generator $\generator$ such that for any
$f \in \rmc^\infty(\M)$, $\generator(f) = \sum_{i=1}^\ell V_i(V_i(f))$. The
\emph{extrinsic} approach is particularly convenient since the SDE appearing in 
\cref{prop:extrinsic_brownian} can be seen as an SDE on the Euclidean space
$\rset^p$. 

We finish this paragraph, by investigating the behaviour of the Brownian motion
in local coordinates. For simplicity, we assume here that we have access to a
system of global coordinates. In the case where the coordinates are strictly
local then we refer to \citet[Chapter 5, Theorem 1]{ikeda1989sto} for a
construction of a global solution by patching local solutions. We denote
$\{X_k, X_{i,j}\}_{1 \leq i,j,k \leq d}$ such that for any $u \in \FM$,
$\{X_k(u), X_{i,j}(u)\}_{1 \leq i,j,k \leq d}$ is a basis of $\mathrm{T}_u \FM$,
Using properties of the horizontal lift, see \cite[Chapter
2]{hsu2002stochastic}, we get that
$(\bfU_t)_{t \geq 0} = (\{\bfX^k_t, \bfE_t^{i,j}\}_{1 \leq i,j,k \leq d})$
obtained in \cref{prop:intrinsic_brownian} is given in the global coordinates
for any $i,j,k \in \{1, \dots, d\}$ by
\begin{equation}
  \textstyle{
    \rmd \bfX_t^k = \sum_{j=1}^d \bfE_t^{k,j} \circ \rmd \bfB_t^k  , \qquad \rmd \bfE_t^{i,j} = - \sum_{n=1}^d \{\sum_{\ell, m=1}^d \bfE_t^{\ell,n}\bfE_t^{m,j} \Gamma_{\ell,m}^{i}(\bfX_t)\} \circ \rmd \bfB_t^n  . 
    }
  \end{equation}
  By definition of the Stratanovitch integral we have that for any $k \in \{1, \dots, d\}$
  \begin{equation}
    \textstyle{
      \rmd \bfX_t^k = \sum_{j=1}^d \{ \bfE_t^{k,j} \rmd \bfB_t^k +\tfrac{1}{2} \rmd [\bfE_t^{k,j}, \bfB_t^j]_t \}  .
      }
    \end{equation}
    Let $(\bfM_t)_{t \geq 0} = (\{\bfM_t^k\}_{k=1}^d)_{t \geq 0}$ such that for
    any $t \geq 0$ and $k \in \{1, \dots, d\}$
    $\bfM_t^k = \sum_{j=1}^d \int_0^t \bfE_t^{k,j} \rmd \bfB_t^k$. We obtain
    that $\rmd \bfM_t = G(\bfX_t)^{-1/2} \rmd \bfB_t$ for some $d$-dimensional
    Brownian motion $(\bfB_t)_{t \geq 0}$, using L\'evy's characterization of
    Brownian motion. In addition, we have that for any
    $k, j \in \{1, \dots, d\}$
    \begin{equation}
      \textstyle{[\bfE^{k,j}, \bfB^j]_t = -\sum_{\ell, m=1}^d \int_0^t \bfE_t^{\ell, j} \bfE_t^{m,j} \Gamma_{\ell, m}^k(\bfX_t) \rmd t }
    \end{equation}
    Hence, using this result and the fact that
    $\sum_{j=1}^d \bfE_t^{\ell, j} \bfE_t^{m,j} = g^{\ell,m}(\bfX_t)$, we get
    that for any $k \in \{1, \dots, d\}$
    \begin{equation}
      \textstyle{\rmd \bfX_t^k =-  \tfrac{1}{2} \sum_{\ell, m=1}^d g^{\ell,m}(\bfX_t) \Gamma_{\ell, m}^k(\bfX_t) \rmd t + (G(\bfX_t)^{-1/2} \rmd \bfB_t)^k  . }
    \end{equation}
    Note that this result could also have been obtained using the expression of
    the Laplace--Beltrami in local coordinates.

    \paragraph{Brownian motion and random walks}

    In the previous paragraph we consider three SDEs to obtain a Brownian motion
    on $\M$ (stochastic development, isometric embedding and local
    coordinates). In this section, we summarize results from
    \textcite{jorgensen1975central} establishing the limiting behaviour of Geodesic
    Random Walks (GRWs) when the stepsize of the random walk goes to $0$. This will be
    of particular interest when considering the time-reversal process. We start
    by defining the geodesic random walk on $\M$, following \citet[Section
    2]{jorgensen1975central}.

    Let $\{ \nu_x \}_{x \in \M}$ such that for any $x \in \M$,
    $\nu_x: \mcb{\mathrm{T}_x \M} \to \ccint{0,1}$ with
    $\nu_x(\mathrm{T}_x \M) =1$, \ie \ for any $x \in \M$, $\nu_x$ is a
    probability measure on $\mathrm{T}_x \M$. Assume that for any $x \in \M$,
    $\int_{\M} \normLigne{v}^3 \rmd \nu_x(v)< +\infty$. In addition assume that
    there exists $\mu^{(1)} \in \XM$ and $\mu^{(2)} \in \XMdeux$, where
    $\XMdeux$ is the section
    $\Gamma(\M, \sqcup_{x \in \M} \mathcal{L}(\mathrm{T}_x \M))$, such that for
    any $x \in \M$, $\int_{\M} v \rmd \nu_x(v) = \mu^{(1)}(x)$ and
    $\int_{\M} v \otimes v \rmd \nu_x(v) = \mu^{(2)}(x)$. In addition, we assume
    that for any $x \in \M$,
    $\Sigma(x) = \mu^{(2)}(x) - \mu^{(1)}(x) \otimes \mu^{(1)}(x)$ is strictly
    positive definite and that there exists $\Ltt \geq$ such that for any
    $x, y \in \M$, $\tvnorm{\nu_x - \nu_y} \leq \Ltt d_\M(x,y)$. Where we have that
    for any $\nu_1 \in \Pens(\mathrm{T}_x \M)$ and $\nu_2 \in \Pens(\mathrm{T}_y \M)$,
    \begin{equation}
      \tvnorm{\nu_x - \nu_y} = \sup \ensembleLigne{\nu_1[f] - \Gamma_{x}^y(\gamma)_\# \nu_2[f]}{\gamma \in \mathrm{Geo}_{x,y}, \ f \in \rmc(\mathrm{T}_x \M)}  . 
    \end{equation}
    Note that if $d_\M(x,y) \leq \vareps$ then for some $\vareps > 0$ we have that $\abs{\mathrm{Geo}_{x,y}}=1$.

    \begin{definition}{Geodesic random walk}{}
      Let $X_0$ be a $\M$-valued random variable.  For any $\gamma > 0$, we
      define $(\bfX_t^{\gamma})_{t \geq 0}$ such that $\bfX_0^\gamma = X_0$ and
      for any $n \in \nset$ and $t \in \ccint{0, \gamma}$,
      $\bfX_{n\gamma + t} = \exp_{\bfX_{n \gamma}}[t\gamma \{ \mu_n +
      (1/\sqrt{\gamma}) (V_n - \mu_n)\}]$, where $(V_n)_{n \in \nset}$ is a sequence
      of random variables in such that for any $n \in \nset$, $V_n$
      has distribution $\nu_{\bfX_{n \gamma}}$ conditionally to $\bfX_{n \gamma}$.
    \end{definition}

    For any $\gamma > 0$, the process
    $(X_n^\gamma)_{n \in \nset} = (\bfX_{n \gamma}^\gamma)_{n \in \nset}$ is
    called a geodesic random walk. In particular, for any $\gamma>0$ we denote
    $(\Rker_n^{\gamma})_{n \in \nset}$ the sequence of Markov kernels such that
    for any $n \in \nset$, $x \in \M$ and $\msa \in \mcb{\M}$ we have that
    $\updelta_x \Rker(\msa) = \Pbb(X_n^\gamma \in \msa)$, with $X_0^\gamma =
    x$. The following theorem establishes that the limiting dynamics of a
    geodesic random walk is associated with a diffusion process on $\M$ whose
    coefficients only depends on the properties of $\nu$ \cite[see][Theorem
    2.1]{jorgensen1975central}.

    \begin{theorem}{Convergence of geodesic random walks}{thm:jorgensen_appendix}
      For any $t \geq 0$, $f \in \rmc(\M)$ and $x \in \M$ we have that
      $\lim_{\gamma \to 0} \normLigne{ \Rker_{\gamma}^{\ceil{t/\gamma}}[f] -
        \Pker_t[f]}_{\infty} = 0$, where $(\Pker_t)_{t \geq 0}$ is the
      semi-group associated with the infinitesimal generator
      $\generator: \ \rmc^\infty(\M) \to \rmc^\infty(\M)$ given for any
      $f \in \rmc^\infty(\M)$ by
      $\generator(f) = \langle \mu^{(1)}, \nabla f \rangle_{\M} + \tfrac{1}{2} \langle
      \Sigma, \nabla^2f \rangle_{\M}$.
    \end{theorem}   

    In particular if $\mu^{(1)} = 0$ and $\mu^{(2)} = \Id$ then the random walk
    converges towards a Brownian motion on $\M$ in the sense of the convergence
    of semi-groups. For any $x \in \M$ in local coordinates we have that
    $\Phi_\# \nu_x$ has zero mean and covariance matrix $G(x)$, where $\Phi$ is
    a local chart around $x$ and $G(x) = (g_{i,j}(x))_{1 \leq i,j \leq d}$ the
    coordinates of the metric in that chart.

\paragraph{Convergence of Brownian motion}

We finish this section with a few considerations regarding the convergence of
the Brownian motion on $\M$. Since we have assumed that $\M$ is compact we have
that there exist $(\Phi_k)_{k \in \nset}$ an orthonormal basis of $-\Delta_\M$ in
$\mathrm{L}^2(p_{\textup{ref}})$, $(\lambda_k)_{k \in \nset}$ such that for any
$i, j \in \nset$, $i \leq j$, $\lambda_i \leq \lambda_j$ and $\lambda_0 = 0$, $\Phi_0=1$ and
for any $k \in \nset$, $\Delta_\M \Phi_k = -\lambda_k \Phi_k$. For any $t \geq 0$
and $x,y \in \M$,
$p_{t|0}(y|x) = \sum_{k \in \nset} \rme^{-\lambda_k t} \Phi_k(x) \Phi_k(y)$ where for
any $f \in \rmc^\infty$ we have
\begin{equation}
  \textstyle{\expeLigne{f(\bfB_t^{\M,x})} = \int_\M p_{t|0}(x,y) f(y) \rmd p_{\textup{ref}}(y)  , }
\end{equation}
where $(\bfB_t^{\M,x})_{t \geq 0}$ is the Brownian motion on $\M$ with $\bfB_0^{\M,x} = x$
and $p_{\textup{ref}}$ is the probability measure associated with the Haussdorff measure on
$\M$. We also have the following result \cite[see][Proposition
2.6]{urakawa2006convergence}.

\begin{proposition}{Convergence of Brownian motion}{brownian_conv_repeat}
  For any $t > 0$, $\Pker_t$ admits a density $p_{t|0}$ w.r.t $p_{\textup{ref}}$ and
  $p_{\textup{ref}} \Pker_t = p_{\textup{ref}}$, \ie \ $p_{\textup{ref}}$ is an invariant measure for
  $(\Pker_t)_{t \geq 0}$. In addition, if there exists $C, \alpha \geq 0$ such
  that for any $t \in \ocint{0,1}$, $p_{t|0}(x|x) \leq C t^{-\alpha /2}$ then 
  for any $\pizero \in \Pens(\M)$ and for any $t \geq 1/2$ we have 
  \begin{equation}
    \textstyle{\tvnorm{\pizero \Pker_t - p_{\textup{ref}}} \leq C^{1/2} \rme^{\lambda_1 /2} \rme^{-\lambda_1 t}  ,}
  \end{equation}
  where $\lambda_1$ is the first non-negative eigenvalue of $-\Delta_\M$ in
  $\mathrm{L}^2(p_{\textup{ref}})$ and we recall that $(\Pker_t)_{t \geq 0}$ is the
  semi-group of the Brownian motion.
\end{proposition}
A review on lower bounds on the first positive eigenvalue
of the Laplace--Beltrami operator can be found in \citep{he2013lower}. These lower
bounds usually depend on the Ricci curvature of the manifold or its diameter. We
conclude this section by noting that in the non-compact case \citep{li1986large}
establishes similar estimates in the case of a manifold with non-negative Ricci
curvature and maximal volume growth.


\section{Likelihood computation}
\label{sec:likel-comp}
\subsection{ODE likelihood computation}

Similarly to \textcite{song2020score}, once the score is learned we can use it
in conjunction with an Ordinary Differential Equation (ODE) solver to compute
the likelihood of the model. Let $(\Phi_t)_{t \in \ccint{0,T}}$ be a family of vector
fields. We define $(\bfX_t)_{t \in \ccint{0,T}}$ such that $\bfX_0$ has
distribution $p_0$ (the data distribution) and satisfying
$\rmd \bfX_t = \Phi_t(\bfX_t) \rmd t$. Assuming that  $p_0$ admits a density
w.r.t.\ $\piinv$ then for any $t \in \ccint{0,T}$, the distribution of $\bfX_t$
admits a density w.r.t.\ $\piinv$ and we denote $p_t$ this density.  We recall that
$\rmd \log p_t(\bfX_t) = -\dive(\Phi_t)(\bfX_t) \rmd t$, see \citet[Proposition 
2]{mathieu2020riemannian} for instance.

Recall that we consider a Brownian motion on the manifold as a forward process
$(\bfB_t^\M)_{t \in \ccint{0,T}}$ with $\{p_t\}_{t \in \ccint{0,T}}$ the associated family
of densities. Thus we have that for any $t \in \ccint{0,T}$ and $x \in \M$
\begin{equation}
  \partial_t p_t(x) = \tfrac{1}{2} \Delta_\M p_t(x) = \dive\left(\tfrac{1}{2} p_t \nabla \log p_t \right)(x)  . 
\end{equation}
Hence, we can define $(\bfX_t)_{t \in \ccint{0,T}}$ satisfying
$\rmd \bfX_t = -\tfrac{1}{2} \nabla \log p_t(\bfX_t) \rmd t$ such that $\bfX_0$ has
distribution $p_0$.
Defining
$(\bfhX_t)_{t \in \ccint{0,T}} = (\bfX_{T-t})_{t \in \ccint{0,T}}$, it follows
that $\bfhX_0$ has distribution $\mathcal{L}(\bfX_T)$ and satisfies
\begin{equation}
  \label{eq:backward_flow}
 \rmd \bfhX_t =\tfrac{1}{2} \nabla \log p_{T-t}(\bfhX_t) \rmd t  . 
\end{equation}
Finally, we introduce $(\bfY_t)_{t \in \ccint{0,T}}$ satisfying
\eqref{eq:backward_flow} but such that $\bfY_0 \sim \piinv$.  Note
that if $T \geq 0$ is large then the two processes
$(\bfY_t)_{t \in \ccint{0,T}}$ and $(\bfhX_t)_{t \in \ccint{0,T}}$ are close
since $\mathcal{L}(\bfX_T)$ is close to $\piinv$.

Therefore, using the score network and a manifold ODE solver \citep[as
in][]{mathieu2020riemannian}, we are able to approximately solve the following
ODE
\begin{equation}
  \rmd\log q_t(\bfhX_t^\theta) = -\tfrac{1}{2}\dive(\mathbf{s}_\theta(T-t,\cdot))(\bfhX_t^\theta) \rmd t ,
\end{equation}
with $q_t$ the density of $\bfY_t^\theta$ w.r.t.\ $\piinv$ and
$\log q_0(\bfY_0) = 0$ with
$\rmd \bfY_t^\theta = \tfrac{1}{2}\dive(\mathbf{s}_\theta(T-t,\bfY_t^\theta))
\rmd t$ and $\bfY_0^\theta \sim \piinv$. The likelihood approximation of the
model is then given by
$\expeLigne{\log q_T(\bfhX_T^\theta)} = \int_{\M} \log q_T(x) \rmd \pdata(x)$,
where
$(\bfhX_t^\theta)_{t \in \ccint{0,T}} = (\bfX_{T-t}^\theta)_{t \in \ccint{0,T}}$
with
$\rmd \bfX_t^\theta = -\tfrac{1}{2}\dive(\mathbf{s}_\theta(t,\bfX_t^\theta))
\rmd t$ and $\bfX_0 \sim \pdata$. In \cref{sec:diff-betw-ode}, we highlight that
this is \emph{not} the likelihood of the SDE model.


    \subsection{Difference between ODE and SDE likelihood computations}
    \label{sec:diff-betw-ode}
    
    In this section, we show that the likelihood computation from
    \textcite{song2020score} does not coincide with the likelihood computation
    obtained with the SDE model. We present our findings in the Riemannian setting
    but our results can be adapted to the Euclidean setting with arbitrary
    forward dynamics. Recall that we consider a Brownian motion on the manifold as a forward process
    $(\bfB_t^\M)_{t \in \ccint{0,T}}$ with $(p_t)_{t \in \ccint{0,T}}$ the associated family
    of densities. We have that for any $t \in \ccint{0,T}$ and $x \in \M$
    \begin{equation}
      \label{eq:forward}
      \partial_t p_t(x) = \tfrac{1}{2} \Delta_\M p_{t|0}(x) = \dive(\tfrac{1}{2}p_t \nabla \log p_t )(x)  . 
    \end{equation}

    \paragraph{ODE model.}
    In the case of the ODE model, we define $(\bfX_t)_{t \in \ccint{0,T}}$ such that
    $\bfX_0 \sim p_0$ and satisfies
    $\rmd \bfX_t = -\tfrac{1}{2}  \nabla \log p_t(\bfX_t) \rmd t$. The family of densities $(q_t)_{t \in \ccint{0,T}}$ associated with $(\bfX_t)_{t \in \ccint{0,T}}$ also
    satisfies \eqref{eq:forward}. Now consider
    $(\bfhX_t)_{t \in \ccint{0,T}} = (\bfX_{T-t})_{t \in \ccint{0,T}}$, this satisfies $\bfhX_0 \sim p_T$ with 
    \begin{equation}
      \label{eq:backward_flow_appendix}
     \rmd \bfhX_t = \tfrac{1}{2}  \nabla \log p_{T-t}(\bfhX_t) \rmd t  .
    \end{equation}
    Finally, we consider $(\bfY_t^{\mathrm{ODE}})_{t \in \ccint{0,T}}$ which also satisfies
    \cref{eq:backward_flow_appendix} and such that $\bfY_0^{\mathrm{ODE}}\sim \piinv$. Denoting
    $(q_t^{\mathrm{ODE}})_{t \in \ccint{0,T}}$ the densities of
    $(\bfY_t^{\mathrm{ODE}})_{t \in \ccint{0,T}}$ w.r.t. $\piinv$ we have for any $t \in \ccint{0,T}$ and $x \in \M$
    \begin{equation}
      \label{eq:proba_flow_ode}
     \partial_t q_t^{\mathrm{ODE}}(x) =  -\dive( \tfrac{1}{2} q_t^{\mathrm{ODE}} \nabla\log p_{T-t} )(x)  . 
    \end{equation}
    
    \paragraph{SDE model.}
    When sampling we consider a process $(\bfY^{\mathrm{SDE}}_t)_{t \in \ccint{0,T}}$ such that
    $\bfY^{\mathrm{SDE}}_0$ has distribution $\piinv$ and whose family of densities
    $(q_t^{\mathrm{SDE}})_{t \in \ccint{0,T}}$ satisfies for any $t \in \ccint{0,T}$ and $x \in \M$
    \begin{align}
      \partial_t q_t^{\mathrm{SDE}}(x) &= -\dive(\nabla \log p_{T-t} q_t^{\mathrm{SDE}}(x)) +\tfrac{1}{2}\Delta_\M q_t^{\mathrm{SDE}}(x) \nonumber \\ &=- \dive(q_t^{\mathrm{SDE}}\{\nabla\log p_{T-t} - \tfrac{1}{2}\nabla\log q_t^{\mathrm{SDE}}\})(x). \label{eq:proba_flow_sde}
    \end{align}
    Hence, \cref{eq:proba_flow_ode} and \cref{eq:proba_flow_sde} do not agree,
    except if $q_t^{\mathrm{SDE}} = q_t^{\mathrm{ODE}} = p_{T-t}$ which is the case if and only if $\bfY^{\mathrm{SDE}}_0$ and
    $\bfY_0^{\mathrm{ODE}}$ have the same distribution as $\bfX_T$. Note that it is possible to
    evaluate the likelihood of the SDE model using that
    \begin{equation}
      \partial_t \log q_t^{\mathrm{SDE}}(\bfY^{\mathrm{SDE}}_t) = \big\{\nabla\log p_{T-t}(\bfY^{\mathrm{SDE}}_t) -\tfrac{1}{2}\nabla\log q_t^{\mathrm{SDE}}(\bfY^{\mathrm{SDE}}_t) \big\}\rmd t  . 
    \end{equation}
    We can use the score approximation $\bm{s}_\theta(t,x)$ to approximate
    $\nabla \log p_t(x)$ for any $t \in \ccint{0,T}$ and $x \in \M$. In order to
    approximate $\nabla \log q_t^{\mathrm{SDE}}$, one can consider another neural network
    $\bm{t}_\theta(t,x)$ approximating $\nabla \log q_t^{\mathrm{SDE}}(x)$ for any $t \in \ccint{0,T}$
    and $x \in \M$. This approximation can be obtained using the implicit score loss
    presented in \Cref{sec:riem-score-appr}.


\section{Parametric family of vector fields}
\label{sec:vector_field}

We approximate $(\nabla \log p_t)_{t \in \ccint{0,T}}$ by a
family of functions $\{\mathbf{s}_\theta\}_{\theta \in \Theta}$ where $\Theta$ is a
set of parameters and for any $\theta \in \Theta$,
$\mathbf{s}_\theta: \ \ccint{0,T} \to \XM$. In this work, we consider several
parameterisations of vector fields:
\begin{itemize}[wide]
\item \textbf{Projected vector field}. We define
  $\mathbf{s}_\theta(t, x) = \text{proj}_{T_{x}\M}(\tilde{\mathbf{s}}_\theta(t, x)) =
  P(x) \tilde{\mathbf{s}}_\theta(t, x) $ for any $t \in \ccint{0,T}$ and $x \in \M$,
  with $\tilde{\mathbf{s}}_\theta: \ \rset^p \times \ccint{0,T} \to \rset^p$ an
  ambient vector field and $P(x)$ the orthogonal projection over
  $\mathrm{T}_x\M$ at $x \in M$.  According to \citet[Lemma 2]{rozen2021moser},
  then $\dive(\mathbf{s}_\theta)(x,t) = \dive_E(\mathbf{s}_\theta)(x,t)$ for any
  $x \in \M$, where $\dive_E$ denotes the standard Euclidean divergence.

\item \textbf{Divergence-free vector fields}: For any Lie group $G$, any basis
  of the Lie algebra $\mathfrak{g} = \mathrm{T}_e G$ yields a global
  frame. Indeed, let $v \in \mathfrak{g}$ and define the flow
  $\Phi: \ \rset \times \M \to \M$ given for any $t \in \rset$ and $x \in M$ by
  $\Phi_t^v(x) = x \exp_e(t v)$. Then defining
  $\{E_i\}_{i=1}^d = \{\partial_t \Phi_0^{v_i}\}_{i=1}^d$, where
  $\{v_i\}_{i=1}^d$ is a basis of $\mathfrak{g}$, we get that $\{E_i\}_{i=1}^d$
  is a left-invariant global frame. As a result, we have that for any
  $i \in \{1, \dots, d\}$, $\dive(E_i)=0$ (for the classical left invariant
  metric). This result simplifies the computation of $\dive(\mathbf{s}_\theta)$
  where $\mathbf{s}_\theta(t,x) = \sum_{i=1}^d {s}^i_\theta(t,x) E_i(x)$ for any
  $t \in \ccint{0,T}$ and $x \in \M$ since we have that 
  $\dive(\mathbf{s}_\theta)(t, x) 
   = \sum_{i=1}^d E_i({s}^i_\theta)(t, x) + \sum_{i=1}^d {s}^i_\theta(t,x) \dive(E_i)(x)
     = \sum_{i=1}^d d{s}^i_\theta (E_i)(t, x)
  $ \cite[see][]{falorsi2020neural}. Note that
  this approach can be extended to any homogeneous space $(G,H)$.
\item \textbf{Coordinates vector fields}. We define
  $\mathbf{s}_\theta(t, x) = \sum_{i=1}^d \mathbf{s}^i_\theta(t,x) E_i(x)$ for any
  $t \in \ccint{0,T}$ and $x \in \M$, with
  $\{E_i\}_{i=1}^d = \{\partial_i \varphi(\varphi^{-1}(x))\}_{i=1}^d$ the vector fields
  induced by a choice of local coordinates, where $\varphi$ is a local
  parameterization $\varphi: \ \msu \to \M$ and $z \in \msu \subset
  \rset^d$. Then the divergence can be computed in these local coordinates
  $\dive(\mathbf{s}_\theta)(t, \varphi(z)) =\absLigne{\det G}^{-1/2} \sum_{i=1}^d
  \partial_i \{ \absLigne{\det G}^{1/2} \mathbf{s}^i_\theta(t,
  \varphi(\cdot))\}(z)$. In the case of the sphere, one recovers the standard
  divergence in spherical coordinates using this formula. Note that $\{E_i\}_{i=1}^d$ does not span the tangent bundle except if the manifold is parallelizable.
   The sphere is a well-known example of non-parallelizable manifold, as per the \emph{hairy ball theorem}. 
\end{itemize}%
\section{Eigensystems of the Laplace--Beltrami operator and heat kernels}
\label{sec:eigenf-eigenv-lapl}

In this section, we recall the eigenfunctions and eigenvalues of the
Laplace--Beltrami operator in two specific cases: the $d$-dimensional torus and
the $d$-dimensional sphere. We also highlight that the heat kernel on compact
manifold can be written as an infinite series using the Sturm--Liouville
decomposition.

\paragraph{The case of the torus}
Let $\{b_i\}_{i=1}^d$ be a basis of $\rset^d$.  We consider the associated
lattice on $\rset^d$, i.e.
$\Gamma = \ensembleLigne{\sum_{i=1}^d \alpha_i b_i}{\{\alpha_i\}_{i=1}^d \in
  \zset^d}$. Finally, the associated $d$-dimensional torus is defined as
$\tset_\Gamma = \rset^d / \Gamma$. Denote
$\rmB = (b_1, \dots, b_d) \in \rset^{d \times d}$. Let
$\{\bar{b}_i\}_{i=1}^d \in (\rset^d)^d$ such that
$(\rmB^{-1})^\top = (\bar{b}_1, \dots, \bar{b}_d)$. We define
$\Gamma^\star = \ensembleLigne{\sum_{i=1}^d \alpha_i
  \bar{b}_i}{\{\alpha_i\}_{i=1}^d \in \zset^d}$, the dual lattice. Note that for
any $x \in \Gamma$ and $y \in \Gamma^\star$ we have that
$\langle x, y \rangle \in \zset$ and that if $\{b_i\}_{i=1}^d$ is an orthonormal
basis then $\Gamma = \Gamma^\star$. The torus $\rset^d/\Gamma$ is a (flat)
compact Riemannian manifold. The set of eigenvalues of the Laplace--Beltrami
operator is given by
$\ensembleLigne{-4 \uppi^2 \normLigne{y}^2}{y \in \Gamma^\star}$. The
eigenfunctions of the Laplace--Beltrami operator are given by
$\ensembleLigne{x \mapsto \sin(2 \uppi \langle x, y \rangle)}{y \in
  \Gamma^\star}$ and
$\ensembleLigne{x \mapsto \cos(2 \uppi \langle x, y \rangle)}{y \in
  \Gamma^\star}$.

\paragraph{The case of the sphere} Next, we investigate the case of the
$d$-dimensional sphere \citep[see][]{saloff1994precise}. The set of eigenvalues of
the Laplace--Beltrami operator is given by
$\ensembleLigne{-k(k+d-1)}{k \in \nset}$. Note that $\lambda_k = k(k+d-1)$ has
multiplicity $d_k = (k+d-2)!/\{(d-1)!k\}(2k+d-1)$. The eigenfunctions of the
Laplace--Beltrami operator are known as the spherical harmonics and can be
defined in terms of Legendre polynomials. When investigating the heat kernel on
the $d$-dimensional sphere, we are interested in the product
$(x,y) \mapsto \sum_{\phi \in \Phi_n} \phi(x)\phi(y)$, where $\Phi_n$ is the set
of eigenfunctions associated with the eigenvalue $\lambda_n$ for $n \in
\nset$. This function can be described using the Gegenbauer polynomials
\cite[see][Theorem 2.9]{atkinson2012spherical}. More precisely, we have that for any
$n \in \nset$ and $x,y \in \mathbb{S}^d$
\begin{align}
  &G_n(x,y) = \textstyle{ \sum_{\phi \in \Phi_n} \phi(x) \phi(y)} \\
  &\quad = \textstyle{n! \Gamma((d-1)/2) \sum_{k=0}^{\floor{n/2}} (-1)^k (1- \langle x,y \rangle^2)\langle x,y \rangle^{n-2k} / (4^k k! (n -2k)! \Gamma(k + (d-1)/2) ) ,}
\end{align}
where here $\Gamma: \ \rset_+ \to \rset$ is given for any $v > 0$ by
$\Gamma(v) = \int_0^{+\infty} t^{v-1} \rme^{-t} \rmd t$.  In the special case
where $d=1$, then the heat kernel coincide with the wrapped Gaussian density and
can be easily evaluated.

\paragraph{Heat kernel on compact Riemannian manifolds.}

\begin{figure}[t]
\centering
  \includegraphics[width=\linewidth]{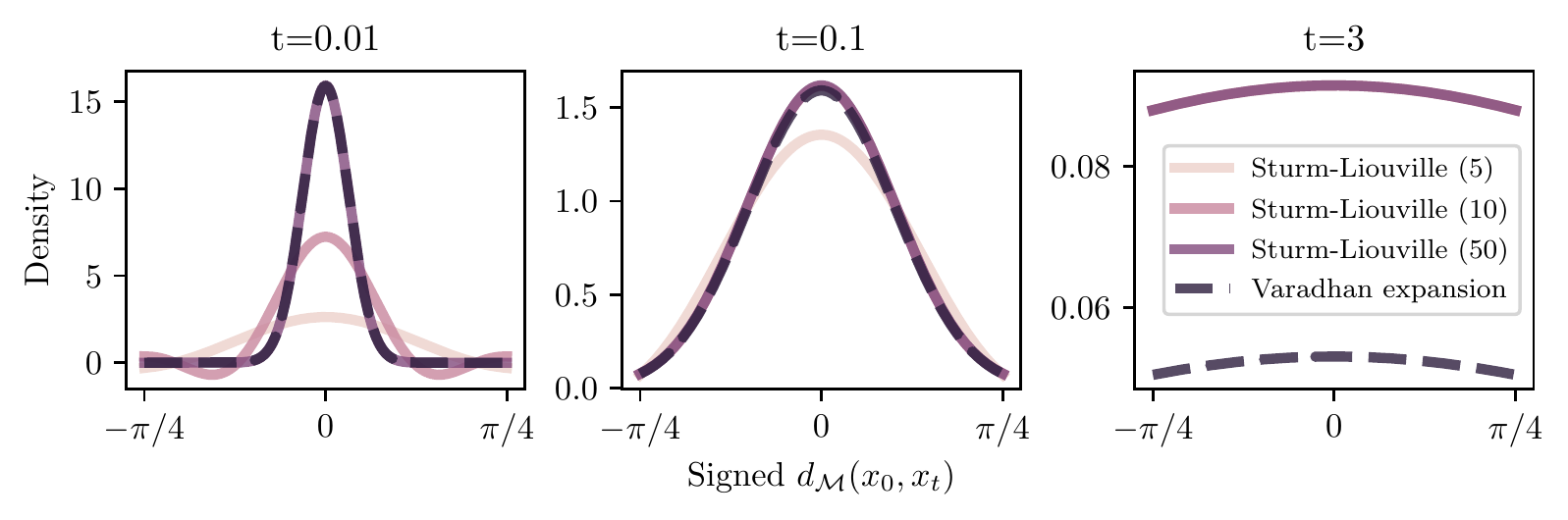}
\vspace{-1.8em}
  \caption{
    Slice of heat kernel $p_{t|0}(x_t|x_0)$ on $\mathbb{S}^2$ for different approximations.
  }
  \label{fig:heat_kernel}
\end{figure}

We recall that in the case of compact manifolds the heat kernel is given by the Sturm--Liouville decomposition \cite{chavel1984eigenvalues} given for any
$t > 0$ and $x, y \in \M$ by 
\begin{equation}
  \label{eq:infinite_sum2}
  \textstyle{p_{t|0}(y|x) = \sum_{j \in \nset} \rme^{-\lambda_j t} \phi_j(x)\phi_j(y),}
\end{equation}
where the convergence occurs in $\mathrm{L}^2(\piinv \otimes \piinv)$,
$(\lambda_j)_{j \in \nset}$ and $(\phi_j)_{j \in \nset}$ are the
eigenvalues, respectively the eigenvectors, of $-\Delta_\M$ in
$\mathrm{L}^2(\piinv)$ \cite[see][Section 2]{saloff1994precise}. When the eigenvalues and eigenvectors are known, we approximate the
logarithmic gradient of $p_{t|0}$ by truncating the sum in
\eqref{eq:infinite_sum2} with $J \in \nset$ terms.
%
Another possibility to approximate $\nabla \log p_{t|0}$ is to rely on the
so-called Varadhan approximation, see \Cref{sec:riem-score-appr}, which is valid
for small $t >0$ .  \Cref{fig:heat_kernel} illustrates these different
approximations of the heat kernel and \Cref{tab:sm_losses_bis} compares the
different loss functions.

\begin{table}[ht]
\centering
\small
\setlength{\tabcolsep}{0.3em}
\renewcommand*{\arraystretch}{1.4}
\begin{tabular}{cc|c|ccc}
Loss & Approximation & Loss function & Unbiased & Consistent & Variance \\
\midrule
\multirow{2}{4.5em}{$\ell_{t|0}$ (DSM)} & Truncation~\eqref{eq:heat_kernel_trunc}  &  $\frac{1}{2} \E \left[ \| s(\bfX_t) - S_{J,t}(\bfX_0,\bfX_t) \|^2 \right]$ & \xmark & \cmark ($J \rightarrow \infty$) & 0 \\
 & Varhadan~\eqref{eq:varadhan}  &  $\frac{1}{2} \E \left[ \| s(\bfX_t) - \log_{\bfX_t}(\bfX_0) / t \|^2 \right]$ & \xmark & \cmark ($t \rightarrow 0$) & 0 \\ \midrule
$\ell_{t|s}$ (DSM) & Varhadan~\eqref{eq:varadhan}  &  $\frac{1}{2} \E \left[ \| s(\bfX_t) - \log_{\bfX_t}(\bfX_s) / (t-s) \|^2 \right]$ & \xmark & \cmark($t \rightarrow s$) & 0  \\ \midrule
\multirow{2}{4.5em}{$\ellim_t$ (ISM)}  & Deterministic & $\E \left[\frac{1}{2} \| s(\bfX_t) \|^2 + \dive( s)(\bfX_t)  \right]$  &  \cmark & \cmark & 0 \\
& Stochastic & $\E \left[\frac{1}{2} \| s(\bfX_t) \|^2 + \vareps^\top \partial s(\bfX_t) \vareps  \right]$ & \cmark & \cmark & $2 \| \partial s \|_F$\\ \bottomrule
\end{tabular} 
\vspace{0.2cm}
\caption{\small Riemannian score matching losses.
}
\label{tab:sm_losses_bis}
\end{table}

%


\section{Predictor-corrector schemes}
\label{sec:predictor-corrector}

In this section, we present a predictor-corrector scheme, adapting the
techniques of \textcite{allgower2012numerical,song2020score} to the manifold setting. 
Changes between \Cref{alg:grw}, \Cref{alg:rsgm} and \Cref{alg:grw-c},
\Cref{alg:rsgm-c} are highlighted in red. Let $t \in \ccint{0,T}$, $\gamma >0$
and $k = \floor{t / \gamma}$. We remark that \Cref{alg:grw-c}
corresponds to the recursion associated with
$(\bfY^j_{t})_{j \in \nset}$
such that for any $j \in \nset$
\begin{equation}
  \bfY^{j+1}_{t} = \exp_{\bfY^j_{t}}[\tfrac{\gamma}{2} \nabla \log p_{T-j\gamma}(\bfY^j_{t}) + \sqrt{\gamma} \bfZ^{j+1}]  ,
\end{equation}
where $\{\bar{\bfZ}^j\}_{j \in \nset}$ is a family of i.i.d Gaussian random
variables with zero mean and identity covariances matrix in $\rset^p$ and for
any $j \in \nset$, $\bfZ^j = \mathrm{P}(\bfY^j_{t}) \bar{\bfZ}^j$. Note that here
$k \in \{0, N-1\}$ is fixed. Letting $\gamma \to 0$, we obtain that under mild
assumptions, see \cite[Theorem 3.1]{kuwada2012convergence},
$(\bfY^j_{t})_{j \in \nset}$ converges to $(\bfY^s_t)_{s \geq 0}$ such that
\begin{equation}  
  \rmd \bfY^s_t = \tfrac{1}{2} \nabla \log p_{T-t}(\bfY^s_t) \rmd s + \rmd \bfB_s^\M  . 
\end{equation}
We have that $p_{T-t}$ is the invariant measure of $(\bfY^s_t)_{s \geq
  0}$. Hence, the role of the corrector step is to project the distribution back
onto $p_{T-t}$ for all times $t \in \ccint{0,T}$, see
\Cref{fig:illustration_pred_correc}.  


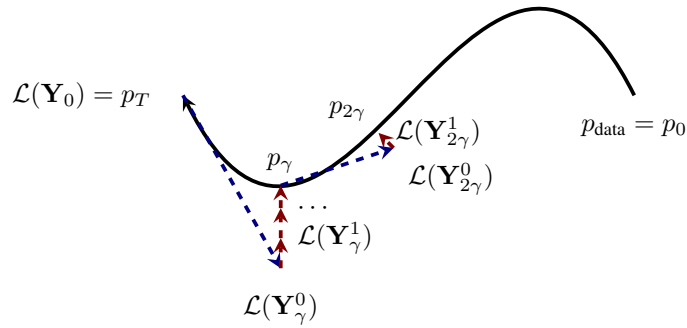
\begin{figure}[htb]
  \centering
\begin{tikzpicture}
  \draw [stealth-, line width = .05cm] (0,0) .. controls  (2,-4) and (4,4) .. (6,0);
  \draw [-stealth, line width = .05cm, blue!50!black, dashed] (0,0) -- (1.3,-2.3);
  \draw [-stealth, line width = .05cm, red!50!black, dashed] (1.3,-2.3) -- (1.3,-1.9);
  \draw [-stealth, line width = .05cm, red!50!black, dashed] (1.3,-1.9) -- (1.3,-1.5);
  \draw [-stealth, line width = .05cm, red!50!black, dashed] (1.3,-1.5) -- (1.3,-1.2);
  \draw [-stealth, line width = .05cm, blue!50!black, dashed] (1.3,-1.2) -- (2.8,-0.7);
  \draw [-stealth, line width = .05cm, red!50!black, dashed] (2.8,-0.7) -- (2.6, -0.5);
  \node [below=0.2cm] at (6,0) {$p_{\textrm{data}}=p_0$};
  \node [left=0.3cm] at (0,0) {$\mathcal{L}(\bfY_0) = p_T$};
  \node [below=0.2cm] at (1.3,-2.3) {$\mathcal{L}(\bfY_\gamma^0)$};
  \node [right=0.1cm] at (1.3,-1.9) {$\mathcal{L}(\bfY_\gamma^1)$};
  \node [right=0.1cm] at (1.3,-1.5) {$\cdots$};
  \node [above=0.05cm] at (1.3,-1.2) {$p_{\gamma}$};
  \node [below right=0.1cm] at (2.8,-0.7) {$\mathcal{L}(\bfY_{2\gamma}^0)$};
  \node [right=0.1cm] at (2.6, -0.5) {$\mathcal{L}(\bfY_{2\gamma}^1)$};
  \node [above left=0.05cm] at (2.6, -0.5) {$p_{2\gamma}$};
\end{tikzpicture}
\vspace{-.8cm}
\caption{Illustration of the effect of the corrector step on RSGM. The black
  line corresponds to the dynamics of the noising process 
  $(p_t)_{t \in \ccint{0,T}}$. The blue dashed lines correspond to the
  predictor step (going backward in time) and the red dashed lines correspond to
  the corrector step (projecting back onto the initial dynamics). Note that
  $\mathcal{L}(\bfY^s_\gamma) \approx p_{T- \gamma}$ and
  $\mathcal{L}(\bfY^s_{2\gamma}) \approx p_{T-2\gamma}$.}
\label{fig:illustration_pred_correc}
\end{figure}

\begin{algorithm}[!ht]
\caption{\small GRW-c (Geodesic Random Walk with corrector)}
\label{alg:grw-c}
\begin{algorithmic}[1]
 \small
\Require $T, N, \bfY_0, b, \sigma, \mathrm{P}$
  \State $\gamma = T / N$ \Comment Step-size
  \For{$k \in \{0, \dots, N-1\}$}
  \State \textcolor{blue!50!black}{/// PREDICTOR STEP}
  \State \textcolor{blue!50!black}{$\bar{\bfZ}_{k+1} \sim \mathrm{N}(0, \Id)$} \Comment Standard Gaussian noise in ambient space $\rset^p$
  \State \textcolor{blue!50!black}{${\bfZ}_{k+1} = \mathrm{P}(\bfY_k) \bar{\bfZ}_{k+1}$} \Comment Projection in the tangent space $\mathrm{T}_x \M$ 
  \State \textcolor{blue!50!black}{$\bfY_{k+1} = \bfY_{k} + \gamma \left[ -b(T-k \gamma, \bfY_k) + \sigma(T-k\gamma)^2 \nabla \log p_{T-k\gamma}(\bfY_{k}) \right] + \sqrt{\gamma} \sigma(T-k\gamma) \bfZ_{k+1}$} \Comment E-M step 
  \State \textcolor{red!50!black}{/// CORRECTOR STEP}
  \State \textcolor{red!50!black}{${\bfY}_{k+1}^0 = {\bfY}_{k+1}$}
  \For{\textcolor{red!50!black}{$s \in \{0, \dots, S-1\}$}}
  \State \textcolor{red!50!black}{$\bar{\bfZ}_{k+1}^s \sim \mathrm{N}(0, \Id)$} \Comment Standard Gaussian noise in ambient space $\rset^p$
  \State \textcolor{red!50!black}{${\bfZ}_{k+1}^s =  \mathrm{P}(\bfY_{k+1}^s) \bar{\bfZ}_{k+1}^s$} \Comment Projection in the tangent space $\mathrm{T}_x \M$ 
  \State \textcolor{red!50!black}{$\bfY_{k+1}^{s+1} = \bfY_{k+1}^{s} + \gamma_s \tfrac{1}{2} \nabla \log p_{T-k\gamma}(\bfY_{k+1}^{s}) + \sqrt{\gamma_s} \bfZ_{k+1}^s$} \Comment Langevin step
  \EndFor
  \State \textcolor{red!50!black}{${\bfY}_{k+1} = {\bfY}_{k+1}^S$}
  \EndFor
  \State {\bfseries return} $\{\bfY_k\}_{k=0}^{N}$
\end{algorithmic}
\end{algorithm}

\begin{algorithm}[!t]
\caption{\small RSGM-c (Riemannian Score-Based Generative Model with corrector)}
\label{alg:rsgm-c}
\begin{algorithmic}[1]
  \small
  \Require $\vareps, T, N, \{X_0^m\}_{m=1}^M, \mathrm{loss}, \mathbf{s}, \theta_0, N_{\textrm{iter}}, \piinv, \mathrm{P}$
\State{{/// TRAINING ///}}

\For{$n \in \{0, \dots, N_{\textrm{iter}}-1\}$}  
\State $X_0 \sim (1/M) \sum_{m=1}^M \updelta_{X_0^m}$ \Comment Random mini-batch from dataset
\State $t \sim U(\ccint{\vareps, T})$ \Comment Uniform sampling between $\vareps$ and $T$
\State $\bfX_t = \textrm{GRW}(t, N, X_0, 0, \Id, \mathrm{P})$ \Comment Approximate forward diffusion with \cref{alg:grw}
\State
$\ell(\theta_n) = \ell_t(T, N, X_0, \bfX_t, \mathrm{loss}, \mathbf{s}_{\theta_n})$
\Comment Compute score matching loss from \cref{tab:sm_losses} \State
$\theta_{n+1} = \verb|optimizer_update|(\theta_n, \ell(\theta_n))$ \Comment ADAM
optimizer step 
\EndFor \State $\theta^\star = \theta_{N_{\textrm{epoch}}}$
\State{{/// SAMPLING ///}} \State $Y_0 \sim \piinv$ \Comment
Sample from uniform distribution \State
$b_\theta^\star(t, x) = \mathbf{s}_{\theta^\star}(T-t, x)$ for any
$t \in \ccint{0,T}$, $x \in \M$ \Comment Reverse process drift \State
$\{Y_k\}_{k=0}^{N} = \textrm{GRW-c}(T, N, Y_0, b_{\theta^\star}, \Id,
\mathrm{P})$ \Comment Approximate reverse diffusion with \Cref{alg:grw-c} \State
{\bfseries return} $\theta^\star, \{Y_k\}_{k=0}^{N}$
\end{algorithmic}
\end{algorithm}


\section{Time-reversal formula: extension to  Riemannian manifolds}
\label{sec:time-reversal}

In this section, we provide the proof of \Cref{thm:time_reversal_manifold}.  The
proof follows the arguments of \citet[Theorem 4.9]{cattiaux2021time}. We could
have also applied the abstract results of \citet[Theorem 5.7]{cattiaux2021time}
to obtain our results. Note that the time-reversal on manifold could also be
obtained by readily extending arguments from \citet{haussmann1986time}, however
the entropic conditions found by \citet{cattiaux2021time} are more natural when
it comes to the study of the Schr\"odinger Bridge problem. For the interested
reader we provide an informal derivation of the time-reversal formula obtained
by \citet{haussmann1986time} in \cref{sec:informal-derivation}. The proof of
\Cref{thm:time_reversal_manifold} is given in
\cref{sec:proof-crefthm:t}. Finally, we emphasize that \citet{garcia2021brenier}
have developed a Girsanov theory for stochastic processes defined on compact
manifolds with boundary in order to study the Brenier-Schr\"odinger problem.

\subsection{Informal derivation}
\label{sec:informal-derivation}

In this section, we provide a non-rigorous derivation of
\Cref{thm:time_reversal_manifold} following the approach of
\citet{haussmann1986time}. Let $(\bfX_t)_{t \in \ccint{0,T}}$ be a continuous
process such that for any $f \in \rmc^2(\M)$ we have that
$(\bfM_t^{\bfX, f})_{t \in \ccint{0,T}}$ is a $\bfX$-martingale where for any
$t \in \ccint{0,T}$
  \begin{equation}
    \label{eq:martingale_forward}
    \textstyle{ \bfM_t^{\bfX, f} = f(\bfX_t) - \int_0^t \{ \langle b(\bfX_s), \nabla f(\bfX_s) \rangle + \tfrac{1}{2} \Delta_\M f(\bfX_s) \} \rmd s  . }
  \end{equation}
  Let $(\bfY_t)_{t \in \ccint{0,T}} = (\bfX_{T-t})_{t \in \ccint{0,T}}$. Our goal is to show that for any $f \in \rmc^2(\M)$, 
  $(\bfM_t^{\bfY, f})_{t \in \ccint{0,T}}$ is a $\bfY$-martingale where for any
  $t \in \ccint{0,T}$
  \begin{equation}
    \textstyle{ \bfM_t^{\bfY, f} = f(\bfY_t) - \int_0^t \{ \langle - b(\bfY_s) + \nabla \log p_{T-s}(\bfY_s), \nabla f(\bfY_s) \rangle + \tfrac{1}{2} \Delta_\M f(\bfY_s) \} \rmd s  . }
  \end{equation}
  Note that here we implicitly assume that for any $t \in \ccint{0,T}$, $\bfX_t$
  admits a smooth positive density w.r.t. $\piinv$ denoted $p_t$.  In other
  words, we want to show that for any $g \in \rmc^2(\M)$ and
  $s, t \in \ccint{0,T}$ with $t \geq s$ we have
  \begin{align}
    \label{eq:time_reversal_manifold_haussman}    
   &\textstyle{\expeLigne{g(\bfY_s)(f(\bfY_t) - f(\bfY_s))}} \\
      & \qquad \textstyle{= \expeLigne{g(\bfY_s)\int_s^t \{ \langle - b(\bfY_u) + \nabla \log p_{T-u}(\bfY_u), \nabla f(\bfY_u) \rangle + \tfrac{1}{2} \Delta_\M f(\bfY_u) \} \rmd u}  . }
  \end{align}
  We introduce the infinitesimal generator
  $\generator: \  \rmc^2(\M) \to \rmc(\M)$ given for any $f \in \rmc^2(\M)$ and $x \in \M$ by
  \begin{equation}
    \generator (f)(x) = \langle b(x) , \nabla f(x) \rangle + \tfrac{1}{2} \Delta_\M f(x)  . 
  \end{equation}
  Similarly, we introduce the infinitesimal generator
  $\generatort: \  \ccint{0,T} \times \rmc^2(\M) \to \rmc(\M)$ given for any $f \in \rmc^2(\M)$, $t \in \ccint{0,T}$ and $x \in \M$ by
  \begin{equation}
    \generatort (t, f)(x) = \langle -b(x) + \nabla \log p_{T-t}(x), \nabla f(x) \rangle + \tfrac{1}{2} \Delta_\M f(x)  . 
  \end{equation}
  With these notations, \eqref{eq:time_reversal_manifold_haussman1} can be written as follows:  we want to show that for any $g \in \rmc^2(\M)$ and
  $s, t \in \ccint{0,T}$ with $t \geq s$ we have 
  \begin{equation}
    \label{eq:time_reversal_manifold_haussman1}    
    \textstyle{\expeLigne{g(\bfY_s)(f(\bfY_t) - f(\bfY_s))} = \expeLigne{g(\bfY_s)\int_s^t \generatort(u, \bfY_u) \rmd u}  . }
  \end{equation}
  The rest of this section follows the first part of the proof of \citet[Theorem 2.1]{haussmann1986time}.
  Let $t, s \in \ccint{0,T}$ with $t \geq s$. We have
  \begin{align}
    \textstyle{\expeLigne{g(\bfY_s)(f(\bfY_t) - f(\bfY_s))}} &= \textstyle{\expeLigne{g(\bfX_{T-s})(f(\bfX_{T-t}) - f(\bfX_{T-t}))}} \\
                                                             &= \textstyle{\expeLigne{\CPELigne{g(\bfX_{T-s})}{\bfX_{T-t}}f(\bfX_{T-t})} - \expeLigne{g(\bfX_{T-s})f(\bfX_{T-s})}} \\
                                                             &= \textstyle{\expeLigne{v(T-t,\bfX_{T-t})f(\bfX_{T-t})} - \expeLigne{v(T-s,\bfX_{T-s})f(\bfX_{T-s})}}  ,
                                                               \label{eq:first_der}
  \end{align}
  with $v: \ \ccint{0,T-s} \times \M \to \rset$ given for any $u \in \ccint{0,T-s}$
  and $x \in \M$ by $v(u,x) = \CPELigne{g(\bfX_{T-s})}{\bfX_u=x}$. We have that $v$
  satisfies the backward Kolmogorov equation, i.e. we have for any
  $u \in \ccint{0,T-s}$ and $x \in \M$
  \begin{equation}
    \label{eq:backward_kolmogorov}
    \partial_u v(u,x) = -\generator v(u,x) . 
  \end{equation}
  Note that it is not trivial to show that $v$ is regular enough to satisfy the
  backward Kolmogorov equation. In this informal derivation, we assume that $v$
  is regular enough and will provide a different rigorous proof of the
  time-reversal formula in \cref{sec:proof-crefthm:t}. However, note that it is
  possible to show that $v$ indeed satisfies the backward Kolmogorov equation by
  adapting arguments from \citet{haussmann1986time} to the manifold framework.

  Let $h: \ \ccint{0,T-s} \times \M \to \rset$ given for any
  $u \in \ccint{0,T-s}$ and $x \in \M$ by $h(u,x) = v(u,x) f(x)$. Using
  \eqref{eq:backward_kolmogorov}, we have for any $u \in \ccint{0,T-s}$ and
  $x \in \M$
  \begin{align}
    \partial_u h(u,x) + \generator h(u, x) &= f(x) \partial_u v(u,x)  + f(x) \generator v(u,x) + v(u,x) \generator f(x) +  \langle \nabla f(x), \nabla v(u,x)\rangle \\
    &=  v(u,x) \generator f(x) + \langle \nabla f(x), \nabla v(u,x)\rangle  .     \label{eq:def_h}
  \end{align}
  In addition, using the
      divergence theorem \citep[see][p.51]{lee2018introduction}, we have for any $u \in \ccint{0,T-s}$
  \begin{align}
    &\expeLigne{\langle \nabla f(\bfX_u), \nabla v(u,\bfX_u)\rangle} = \textstyle{\int_{\M} \langle \nabla f(x_u), \nabla v(u,x_u) p_u(x_u) \rangle \rmd \piinv(x_u) } \\
                                                                    & \qquad \qquad = - \textstyle{\int_{\M} v(u,x_u) \dive(p_u \nabla f) (x_u) \rmd \piinv(x_u) } \\
    & \qquad \qquad = - \textstyle{\int_{\M} v(u,x_u) \Delta_\M f(x_u) p_u(x_u) \rmd \piinv(x_u)} \\
    & \qquad \qquad \qquad \textstyle{- \int_{\M} v(u,x_u) \langle \nabla f(x_u), \nabla \log p_u(x_u) \rangle p_u(x_u) \rmd \piinv(x_u) } \\
                                                                    & \qquad \qquad = - \textstyle{\expeLigne{ v(u,\bfX_u) \Delta_\M f(\bfX_u)}  - \expeLigne{ v(u,\bfX_u) \langle \nabla f(\bfX_u), \nabla \log p_u(\bfX_u) \rangle} }  .
  \end{align}
  Therefore, using this result and \eqref{eq:def_h} we get that for any
  $u \in \ccint{0,T-s}$
  \begin{align}
    &\expeLigne{\partial_u h(u,\bfX_u) + \generator h(u, \bfX_u)} \\ &= \expeLigne{v(u,\bfX_u)\{ \langle b(\bfX_u) - \nabla \log p_u(\bfX_u), \nabla f(\bfX_u) \rangle -\tfrac{1}{2} \Delta_\M f(\bfX_u)\}} \\
    &= -\expeLigne{v(u,\bfX_u)\generatort(T-u,f)(\bfX_u)}  . 
  \end{align}
  Combining this result and \eqref{eq:martingale_forward} and that for any
  $u \in \ccint{0,T-s}$ and $x \in \M$,
  $v(u,x) = \CPELigne{g(\bfX_{T-s})}{\bfX_u=x}$ we get
  \begin{align}
    &\expeLigne{v(T-t,\bfX_{T-t})f(\bfX_{T-t})} - \expeLigne{v(T-s,\bfX_{T-s})f(\bfX_{T-s})}\\
    & \qquad \qquad \qquad = \expeLigne{h(T-t, \bfX_{T-t}) - h(T-s, \bfX_{T-s})} \\
                                                                                            & \qquad \qquad \qquad = \textstyle{\int_{T-t}^{T-s} \expeLigne{v(u,\bfX_u)\generatort(T-u, \bfX_u)} \rmd u } \\
    &\qquad \qquad \qquad= \textstyle{\expeLigne{g(\bfX_{T-s})\int_{T-t}^{T-s} \generatort(T-u, \bfX_u) \rmd u } .}\\
  \end{align}
  Using this result, \eqref{eq:first_der} and the change of variable $u \mapsto T-u$ we obtain 
  \begin{equation}
    \expeLigne{g(\bfY_s)(f(\bfY_t) - f(\bfY_s))} = \textstyle{\expeLigne{g(\bfX_{T-s})\int_{T-t}^{T-s} \generatort(T-u, \bfX_u) \rmd u } } = \textstyle{\expeLigne{g(\bfY_{s})\int_{s}^{t} \generatort(u, \bfY_u) \rmd u } }  .
  \end{equation}
  Hence, \eqref{eq:time_reversal_manifold_haussman} holds and we have proved
  \Cref{thm:time_reversal_manifold}. Again, we emphasize that in order to make
  the proof completely rigourous one needs to derive regularity properties of $v$.

\subsection{Proof of \Cref{thm:time_reversal_manifold}}
\label{sec:proof-crefthm:t}

In this section, we follow another approach to prove the time-reversal
formula. We are going to use the integration by part formula of \citet[Theorem
3.17]{cattiaux2021time} in a similar spirit as \citet[Theorem
4.9]{cattiaux2021time} in the Euclidean setting. In order to adapt arguments
from \citet{cattiaux2021time} to our Riemannian setting, we use the Nash
embedding theorem in order to embed our processes in a Euclidean space and
leverage tools from Girsanov theory. The rest of the section is organized as
follows. First in \cref{sec:diff-proc-stoch}, we recall basic properties of
infinitesimal generators and recall the integration by part formula of
\citet[Theorem 3.17]{cattiaux2021time}. Then in \cref{sec:girs-theory-comp}, we
extend some Girsanov theory to compact Riemannian manifolds using the Nash
embedding theorem. We conclude the proof in \cref{sec:concluding-proof}.

\subsubsection{Diffusion processes and integration by part formula}
\label{sec:diff-proc-stoch}

In this section, we state a simplified version of \citet[Theorem
3.17]{cattiaux2021time} for Markov continuous path (probability) measure on
Polish spaces. Let $(\msx, \mcx)$ be a Polish space. We say that $\Pbb$ is a
path measure if $\Pbb \in \Pens(\rmc(\ccint{0,T}, \msx))$. Let
$(\bfX_t)_{t \in \ccint{0,T}}$ with distribution $\Pbb$. We denote
$(\mcf_t)_{t \in \ccint{0,T}}$ the filtration such that for any
$t \in \ccint{0,T}$, $\mcf_t = \sigma(\bfX_s, \ s \in \ccint{0,t})$. Let
$(\bfM_t)_{t \in \ccint{0,T}}$ be a Polish-valued stochastic process. We say that
$(\bfM_t)_{t \in \ccint{0,T}}$ is a $\Pbb$-local martingale if it is a local
martingale w.r.t. the filtration $(\mcf_t)_{t \in \ccint{0,T}}$. A function
$u: \ \ccint{0,T} \times \msx \to \rset$ is said to be in the domain of the
extended generator of $\Pbb$ if there exists a process
$(\generatorb_\Pbb u(t, \bfX_{\ccint{0,t}}))_{t \in \ccint{0,T}}$ such that:
\begin{enumerate}[label= (\alph*),  wide, labelindent=0pt]
\item $(\generatorb_\Pbb u(t, \bfX_{\ccint{0,t}}))_{t \in \ccint{0,T}}$ is adapted w.r.t. $(\mcf_t)_{t \in \ccint{0,T}}$.
\item $\int_0^T \absLigne{\generatorb_\Pbb u(t, \bfX_{\ccint{0,t}})} \rmd t < +\infty$, $\Pbb$-a.s.
\item The process $(\bfM_t)_{t \in \ccint{0,T}}$ is a $\Pbb$-local martingale,
  where for any $t \in \ccint{0,T}$
  \begin{equation}
    \textstyle{\bfM_t = u(t,\bfX_t) - u(0, \bfX_0) - \int_0^t \generatorb_\Pbb u(s, \bfX_{\ccint{0,s}}) \rmd s   .}
  \end{equation}
\end{enumerate}
The domain of the extended generator is denoted $\dom(\generatorb_\Pbb)$. We say
that $(u,v)$ with $u,v : \ \ccint{0,T} \times \msx \to \rset$ is in the domain
of the carr\'e du champ if $u,v, uv \in \dom(\generatorb_\Pbb)$. In this case, we
define the carr\'e du champ $\carrechampb_\Pbb$ as
\begin{equation}
  \carrechampb_\Pbb(u,v) = \generatorb_\Pbb(uv) - \generatorb_\Pbb(u)v - \generatorb_\Pbb(v)u  . 
\end{equation}
Note that if $\msx = \M$ is a Riemannian manifold,
$\rmc^2(\M) \subset \dom(\generatorb_\Pbb)$ and for any $u \in \rmc^2(\M)$
$\generatorb_\Pbb(u) = \langle \nabla u, X \rangle + \tfrac{1}{2}\Delta_\M u$ with
$X \in \Gamma(\TM)$  then we have that $\rmc^2(\M) \times \rmc^2(\M) \subset \dom(\carrechampb_\Pbb)$
and for any $u, v \in \rmc^2(\M)$,
$\carrechampb_\Pbb(u,v) = \langle \nabla u, \nabla v \rangle$. Assume that there exists
$\mathcal{U}_\Pbb \subset \dom(\generatorb_\Pbb) \cap \rmc_b(\msx)$ such that
$\mathcal{U}_\Pbb$ is an algebra. We denote $\mathcal{U}_{\Pbb,2}$ such that
\begin{equation}
  \mathcal{U}_{\Pbb,2} = \ensembleLigne{u \in \mathcal{U}_\Pbb}{\generatorb_\Pbb u \in \mathrm{L}^2(\Pbb), \ \carrechampb_\Pbb(u,u) \in \mathrm{L}^1(\Pbb)}  . 
\end{equation}
Finally we denote $R(\Pbb)$ the time-reverse path measure, i.e. for any
$\msa \in \mcb{\rmc(\ccint{0,T}, \msx)}$ we have
$R(\Pbb)(\msa) = \Pbb(R(\msa))$, where
$R(\msa) = \ensembleLigne{t \mapsto \omega_{T-t}}{\omega \in \msa}$.  In what
follows, we assume $\Pbb$ is Markov. It is well-known, see \citep[Theorem
1.2]{leonard2014reciprocal} for instance, that in this case $R(\Pbb)$ is also
Markov. In addition, since $\Pbb$ is Markov, for any $u \in \mathrm{dom}(\generatorb_\Pbb)$ and
$t \in \ccint{0,T}$ there exists $\generator_\Pbb$ such that
$\generatorb_\Pbb u(t, \bfX_{\ccint{0,t}}) = \generator_\Pbb u(t, \bfX_t)$ with
$\generator_\Pbb u: \ \ccint{0,T} \times \msx \to \rset$. Similarly, we define
$\carrechamp_\Pbb(u,v): \ \ccint{0,T} \times \msx \to \rset$ from $\carrechampb_\Pbb(u,v)$.

We are now ready to state the integration by part formula,
\citep[Theorem 3.17]{cattiaux2021time}. 

\begin{theorem}{}{ibp_cattiaux}
  Let $u, v \in \mathcal{U}_{\Pbb, 2}$. The following hold:
  \begin{enumerate}[label= (\alph*),  wide, labelindent=0pt]    
  \item If
  $u \in \dom(\generator_{R(\Pbb)})$ and
  $\generator_{R(\Pbb)}u \in \mathrm{L}^1(\Pbb)$ then for almost any $t \in \ccint{0,T}$
  \begin{equation}
    \expeLigne{\{\generator_\Pbb u(t, \bfX_t) + \generator_{R(\Pbb)} u (T-t, \bfX_t)\}v(\bfX_t) + \carrechamp_\Pbb(u,u)(t, \bfX_t)} = 0  .       
  \end{equation}  
\item If the following hold:
  \begin{enumerate}[label=\roman*)]
  \item $\carrechamp_\Pbb(u,v) \in \rmc(\ccint{0,T} \times \msx, \rset)$.
  \item $\mathcal{U}_{2, \Pbb}$ determines the weak convergence of Borel measures.
  \item $\mu$ defines a finite measure on $\ccint{0,T} \times \msx$ where for
    any $\omega \in \bar{\mathcal{U}}_{2, \Pbb}$ we have
    \begin{equation}
      \textstyle{\mu[\omega] = \expeLigne{\int_0^T \carrechamp_\Pbb(u,\omega_t)(t, \bfX_t) \rmd t }},
    \end{equation}
    where
    $\bar{\mathcal{U}}_{2, \Pbb} = \ensembleLigne{\omega \in \rmc(\ccint{0,T}
      \times \msx, \rset)}{\omega(t, \cdot) \in \mathcal{U}_{2, \Pbb}\ \
      \text{for any $t \in \ccint{0,T}$}}$.
  \end{enumerate}
  Then $u \in \dom(\generator_{R(\Pbb)})$ and
  $\generator_{R(\Pbb)}u \in \mathrm{L}^1(\Pbb)$.
  \end{enumerate}
\end{theorem}

Note that this theorem is a simplified version of \citet[Theorem
3.17]{cattiaux2021time} where we restrict ourselves to the case of Markov path
measures. In what follows, we wish to apply \Cref{thm:ibp_cattiaux} to diffusion
processes on manifolds. To do so, we will verify that under a finite entropy
assumption, the conditions $u \in \dom(\generator_{R(\Pbb)})$ and
$\generator_{R(\Pbb)}u \in \mathrm{L}^1(\Pbb)$ are fullfilled for a class of
regular functions $u$. These integrability results are obtained using Girsanov
theory.

\subsubsection{Girsanov theory on compact Riemannian manifolds}
\label{sec:girs-theory-comp}

In this section, we will consider two types of martingale problems: one on
Euclidean spaces and one on the compact Riemannian manifold $\M$. Let
$\Pbb \in \Pens(\rmc(\ccint{0,T}, \rset^p))$. We say that $\Pbb$ satisfies the
(Euclidean) martingale problem with infinitesimal generator
$\generator: \ \ccint{0,T} \times \rmc^2(\rset^p) \times \rset^p \to \rset$ if
for any $u \in \rmc_c^2(\rset^p)$, $(\bfM_t)_{t \in \ccint{0,T}}$ is a
$\Pbb$-martingale where for any $t \in \ccint{0,T}$ we have
\begin{equation}
  \textstyle{
    \bfM_t = \bfM_0 + \int_0^t \generator(t, u)(\bfX_s) \rmd s  ,
    }
  \end{equation}
  where $(\bfX_t)_{t \in \ccint{0,T}}$ has distribution $\Pbb$ and
  $\int_0^T \absLigne{\generator(t, u)(\bfX_s) \rmd t} <+\infty$, $\Pbb$-a.s.
  Let $\Pbb \in \Pens(\rmc(\ccint{0,T}, \M))$. We say that $\Pbb$ satisfies the
  (Riemannian) martingale problem with infinitesimal generator
  $\generatort: \ \ccint{0,T} \times \rmc^2(\M) \times \M \to \rset$ if for any
  $u \in \rmc^2(\M)$, $(\bfM_t)_{t \in \ccint{0,T}}$ is a $\Pbb$-martingale
  where for any $t \in \ccint{0,T}$ we have
\begin{equation}
  \textstyle{
    \bfM_t = \bfM_0 + \int_0^t \generatort(t, u)(\bfX_s) \rmd s  ,
    }
  \end{equation}
  where $(\bfX_t)_{t \in \ccint{0,T}}$ has distribution $\Pbb$ and 
  $\int_0^T \absLigne{\generatort(t, u)(\bfX_s) \rmd t} <+\infty$, $\Pbb$-a.s.
  We now prove the following theorem.

  \begin{proposition}{}{girsanov_manifold}
    Let $\Qbb$ be the path measure of a Brownian motion on $\M$. Let $\Pbb$ be a
    Markov path measure on $\rmc(\ccint{0,T}, \M)$ such that
    $\KL{\Pbb}{\Qbb} < +\infty$. Then there exists $\beta$ such that for any
    $t \in \ccint{0,T}$ and $x \in \M$, $\beta(t,x) \in \mathrm{T}_x \M$ and we
    have that $\Pbb$ satisfies the martingale problem with infinitesimal
    generator $\generator$ where for any $t \in \ccint{0,T}$, $u \in \rmc^2(\M)$
    and $x \in \M$ we have
    \begin{equation}
      \generator(t,u)(x) = \langle \beta(t,x), \nabla u(x) \rangle + \tfrac{1}{2} \Delta_\M u(x)  . 
    \end{equation}
    In addition, we have that
    \begin{equation}
      \textstyle{\KL{\Pbb}{\Qbb} = \KL{\Pbb_0}{\Qbb_0} + \tfrac{1}{2} \int_0^T \expeLigne{\norm{\beta(t, \bfX_t)}^2} \rmd t  ,}
    \end{equation}
    where $(\bfX_t)_{t \in \ccint{0,T}}$ has distribution $\Pbb$.
  \end{proposition}

  \begin{proof}{}{}
    First, we extend $(\bfB_t^\M)_{t \in \ccint{0,T}}$ to $\rset^p$ using the
    Nash embedding theorem \citep[see][]{gunther1991isometric}.
    $(\bfB_t^\M)_{t \in \ccint{0,T}}$ can be seen as a process on $\rset^p$ (for
    some $p \in \nset$) which satisfies in a weak sense
    \begin{equation}
      \textstyle{
        \rmd \bfB_t^\M = \sum_{i=1}^p \rmP_i(\bfB_t^\M) \circ \rmd \bfB_t^i  = P(\bfB_t^\M) \circ \rmd \bfB_t  ,
        }
    \end{equation}
    where $(\bfB_t)_{t \in \ccint{0,T}}$ is a $p$-dimensional Brownian motion
    and $\rmP \in \rmc^\infty(\rset^p, \rset^{p\times p})$ is such that for any
    $x \in \M$, $\rmP(x)$ is the projection onto $\mathrm{T}_x \M$ and for any
    $i \in \{1, \dots, p\}$, $\rmP_i \in \rmc^\infty(\rset^p, \rset^p)$ with
    $\rmP_i = \rmP e_i$ where $\{e_j\}_{j=1}^p$ is the canonical basis of
    $\rset^p$. We refer to \Cref{sec:metr-conn-tens} for more details on the
    projection operator and its extension to $\rset^p$.  Using the link between Stratanovitch
    and It\^o integral, there exists $\bar{b} \in \rmc^\infty(\rset^p, \rset^p)$
    such that $(\bfB_t^\M)_{t \in \ccint{0,T}}$ can be seen as a process on
    $\rset^p$ which satisfies in a weak sense
    \begin{equation}
      \textstyle{
        \rmd \bfB_t^\M = \bar{b}(\bfB_t^\M) \rmd t +  \rmP(\bfB_t^\M)  \rmd \bfB_t  ,
        }
      \end{equation}
      where $\bar{b}$ is given in \eqref{eq:strata_ito_transfo_sde} and
      satisfies $\mathrm{P}\bar{b}(x) = 0$ for any $x \in \M$, see the remark
      following \eqref{eq:strata_ito_transfo_sde}.  For any $u, v \in \rmc^2_c(\M)$,
      we consider $\bar{u}, \bar{v}$ extensions to $\rmc^2_c(\rset^p)$ and we have for
      any $s, t \in \ccint{0,T}$
      \begin{align}
        &\textstyle{\expeLigne{\bar{v}(\bfB_s^\M) \int_s^t \tfrac{1}{2} \Delta_\M u(\bfB_u^\M) \rmd u}} \\
        & \qquad =  \textstyle{\expeLigne{\bar{v}(\bfB_s^\M) \int_s^t \{ \langle \nabla \bar{u}(\bfB_w^\M), \bar{b}(\bfB_w^\M) \rangle + \tfrac{1}{2} \langle \rmP(\bfB_w^\M), \nabla^2 \bar{u}(\bfB_w^\M) \rangle \} \rmd w}  . }
      \end{align}
      In particular, we get that for any $x \in \M$,
      $\Delta_\M u(x) = 2 \langle \nabla \bar{u}(x), \bar{b}(x) \rangle +
      \langle \rmP(x), \nabla^2 \bar{u}(x) \rangle$. Note that
      $(\bfB_t^\M)_{t \in \ccint{0,T}}$ (seen as a process on $\rset^p$)
      satisfies the condition $\mathrm{(U)}$ in \textcite{leonard2012girsanov},
      i.e. uniqueness of the trajectories given an initial condition. Therefore
      applying \cite[Theorem 2.1]{leonard2012girsanov}, \citep[Claim
      4.5]{cattiaux2021time}, there exists
      $\bar{\beta}: \ \ccint{0,T} \times \rset^p \to \rset^p$ such that
      \begin{equation}
        \label{eq:KL_ineq1}
      \textstyle{\KL{\Pbb}{\Qbb} = \KL{\Pbb_0}{\Qbb_0} + \tfrac{1}{2} \int_0^T \expeLigne{\normLigne{\rmP(\bfX_t) \bar{\beta}(t, \bfX_t)}^2} \rmd t  .}
    \end{equation}
    In addition, $\Pbb$ (seen as a process on $\rset^p$) satisfies a martingale
    problem with infinitesimal generator
    $\generatorb: \ \ccint{0,T} \times \rmc^2_c(\rset^p) \times \rset^p \to \rset$ such that
    for any $t \in \ccint{0,T}$, $\bar{u} \in \rmc^2_c(\rset^p)$ and $x \in \rset^p$
    \begin{equation}
      \generatorb(t,\bar{u})(x) = \langle \bar{b}(x) + \rmP(x)\bar{\beta}(t,x), \nabla \bar{u}(x) \rangle + \tfrac{1}{2} \langle \rmP(x), \nabla^2 \bar{u}(x)\rangle  . 
    \end{equation}
    Let $\beta: \ \ccint{0,T} \times \M$ such that for any $t \in \ccint{0,T}$
    and $x \in \M$ we have $\beta(t,x) = \rmP(x) \bar{\beta}(t,x)$. In particular,
    we have that for any $x \in \M$, $\beta(t,x) \in \mathrm{T}_x\M$. Let
    $u \in \rmc^2_c(\M)$ and consider an extension $\bar{u}$ to
    $\rmc^2_c(\rset^p)$. For any $t \in \ccint{0,T}$ and $x \in \M$ we have
    \begin{align}
      \generatorb(t,\bar{u})(x) &= \langle \bar{b}(x) + \rmP(x)\bar{\beta}(t,x), \nabla \bar{u}(x) \rangle + \tfrac{1}{2} \langle \rmP(x), \nabla^2 \bar{u}(x)\rangle \\
                               &= \langle  \beta(t,x), \nabla \bar{u}(x) \rangle + \tfrac{1}{2} \Delta_\M u(x) \\
                               &=\langle \beta(t,x), \nabla u(x) \rangle + \tfrac{1}{2} \Delta_\M u(x)  . 
    \end{align}
    In particular, we have that $\Pbb$ (seen as a process on $\M$) satisfies a
    martingale problem with infinitesimal generator
    $\generator: \ \ccint{0,T} \times \rmc^2_c(\M) \times \M \to \rset$ such that
    for any $t \in \ccint{0,T}$, $u \in \rmc^2(\rset^p)$ and $x \in \M$
    \begin{equation}
      \generator(t,\bar{u})(x) = \langle \beta(t,x), \nabla u(x) \rangle + \tfrac{1}{2} \Delta_\M u(x)  . 
    \end{equation}
    In addition, rewriting \eqref{eq:KL_ineq1} we have
      \begin{equation}
        \label{eq:KL_ineq}
      \textstyle{\KL{\Pbb}{\Qbb} = \KL{\Pbb_0}{\Qbb_0} + \tfrac{1}{2} \int_0^T \expeLigne{\normLigne{\beta(t, \bfX_t)}^2} \rmd t  ,}
    \end{equation}
    which concludes the proof.
  \end{proof}

  We also derive the following useful lemma, which will be used in the proof of convergence of RSGM.
  
  \begin{corollary}{}{KL_diff}
    Assume \rref{assum:manifold}. 
    Let $\Pbb^1, \Pbb^2$ be a Markov path measure on $\rmc(\ccint{0,T}, \M)$
    with $\Pbb_0^1 = \Pbb_0^2$. In addition, assume that there exist
    $b_1,b_2 \in \rmc^\infty(\ccint{0,T}, \XM)$ such that
    $(\bfX_t^1)_{t \in \ccint{0,T}}$ and $(\bfX_t^2)_{t \in \ccint{0,T}}$ are
    associated to $\Pbb^1$ and $\Pbb^2$ respectively and satisfy weakly
    $\rmd \bfX_t^i = b_1(t,\bfX_t^i) \rmd t + \rmd \bfB_t$ for $i \in
    \{1,2\}$. Then, we have that
    \begin{equation}
      \textstyle{\KLLigne{\Pbb^1}{\Pbb^2} = \tfrac{1}{2} \int_0^T \expeLigne{\normLigne{b_1(t,\bfX_t^1) - b_2(t,\bfX_t^1)}^2} \rmd t  . }
    \end{equation}
  \end{corollary}

  \begin{proof}{}{}
    Upon, using the Nash embedding theorem \citep[see][]{gunther1991isometric},
    we can assume that $\M$ is a sub-manifold of $\rset^p$ with $p \in \nset$
    such that the Riemannian metric on $\M$ is induced by the Euclidean metric
    on $\rset^p$. Since $\M$ is compact, there exists $R > 0$ such that
    $\M \subset \cball{0}{R}$. Let $\varphi \in \rmc^\infty(\rset^p, \ccint{0,1})$
    such that for any $x \in \cball{0}{R}$, $\varphi(x) = 1$ and for any
    $x \in \rset^p$ with $\normLigne{x} \geq R+1$, $\varphi(x) =0$. Consider
    $\bar{b}_1, \bar{b}_2 \in \rmc_c^2(\ccint{0,T} \times \rset^p, \rset^p)$
    such that for any $t \in \ccint{0,T}$ and $x \in \M$,
    $\bar{b}_i(x) = b_i(x)$ with $i \in \{1,2\}$. Consider
    $(\barbfX_t^i)_{t \in \ccint{0,T}}$ such that for any $i \in \{1,2\}$
    \begin{equation}
      \rmd \barbfX_t^i = \varphi(\barbfX_t^i)\{\rmP(\barbfX_t^i) \bar{b}^i(t, \barbfX_t^i) + \bar{b}(\barbfX_t)\} \rmd t + \varphi(\barbfX_t^i) \rmP(\barbfX_t^i) \rmd \bfB_t, 
    \end{equation}
    where $\bar{b} \in \rmc^\infty(\rset^p, \rset^p)$ is defined in the proof of
    \Cref{prop:girsanov_manifold}.  Let $\barbfX_0^i \sim \Pbb_0^1$ for any
    $i \in \{1,2\}$ then for any $i \in \{1, 2\}$,
    $(\barbfX_t^i)_{t \in \ccint{0,T}}$ (seen as a process on $\M$) is such that
    $\mathcal{L}((\barbfX_t^i)_{t \in \ccint{0,T}}) = \Pbb^i$. Indeed, denote
    $\{\generatorb^i_t\}_{t \in \ccint{0,T}}$ the generator of
    $(\barbfX_t^i)_{t \in \ccint{0,T}}$ for any $i \in \{1,2\}$. Let
    $f \in \rmc^\infty(\M, \rset)$ and $\bar{f} \in \rmc^\infty(\rset^p, \rset)$
    an extension to $\rset^p$. We have that for any $i \in \{1,2\}$, $x \in \M$ and $t \in \ccint{0,T}$
    \begin{align}
      \generatorb^i_t(\bar{f})(x) &= \langle \bar{b}^i(t, x) + \bar{b}(x), \nabla \bar{f}(x) \rangle + \tfrac{1}{2} \langle \rmP(x), \nabla^2 \bar{f}(x) \rangle \\
      &= \langle b^i(t,x), \nabla f(x) \rangle + \tfrac{1}{2} \Delta_\M f(x).
    \end{align}
    Hence, for any $i \in \{1,2\}$, $(\barbfX_t^i)_{t \in \ccint{0,T}}$ (seen as
    a process on $\M$) and $(\bfX_t^i)_{t \in \ccint{0,T}}$ have the same
    infinitesimal generators. Hence,
    $\mathcal{L}((\barbfX_t^i)_{t \in \ccint{0,T}}) = \Pbb^i$ for any
    $i \in \{1, 2\}$.  For any $i \in \{1, 2\}$, denote
    $\bar{\Pbb}^i = \mathcal{L}((\barbfX_t^i)_{t \in \ccint{0,T}})$ (seen as a
    process on $\rset^p$). Note that since for any $x \in \rset^p$ with
    $\normLigne{x} \geq R+1$, $\varphi(x) = 0$ we have that \cite[Equation
    (7.137)]{lipster2001statistics} is satisfied. In addition, since for any
    $x \in \rset^p$ with $\normLigne{x} \geq R+1$,
    $\varphi(x) + \normLigne{\nabla \varphi(x)}= 0$, we have that \cite[Equation
    (4.110), Equation (4.111)]{lipster2001statistics} are satisfied. In
    addition, letting for any $t \in \ccint{0,T}$ and $x \in \rset^p$,
    $\alpha(t,x) = \bar{b}^1(t, x) - \bar{b}^2(t, x) = \mathrm{P}(x)
    (\bar{b}^1(t, x) - \bar{b}^2(t, x))$, we have that for any
    $t \in \ccint{0,T}$,
    $\mathrm{P}(x) \alpha(t,x) = \mathrm{P}(x)(\bar{b}^1(t, x) - \bar{b}^2(t, x))$. Therefore,
    we can apply \cite[Section 7.6.4]{lipster2001statistics} and using that
    $\mathrm{P}(x) \bar{b}(x) = 0$ for any $x \in \M$ (see the proof of
    \Cref{prop:girsanov_manifold}), we have that
    \begin{align}     
      &(\rmd \bar{\Pbb}^1 / \rmd \bar{\Pbb}^2)((\barbfX_t^1)_{t \in \ccint{0,T}}) =\textstyle{ \exp\left.[\int_0^T \langle \bar{b}^1(t, \barbfX_t^1) - \bar{b}^2(t, \barbfX_t^1), \mathrm{P}(\barbfX_t^1) \rmd \barbfX_t^1 \rangle \right. }\\
                                                                              & \qquad \qquad  \quad  \textstyle{ - \tfrac{1}{2} \int_0^T \langle \bar{b}^1(t, \barbfX_t^1) - \bar{b}^2(t, \barbfX_t^1), \mathrm{P}(\barbfX_t^1)(\bar{b}^1(t, \barbfX_t^1) + \bar{b}^2(t, \barbfX_t^1)) \rangle  \rmd t ] }\\
                                                                              & \qquad \qquad =\textstyle{ \exp\left.[\int_0^T \langle \bar{b}^1(t, \barbfX_t^1)  - \bar{b}^2(t, \barbfX_t^1), \mathrm{P}(\barbfX_t^1) \{\bar{b}^1(t, \barbfX_t^1) + \bar{b}(\barbfX_t^1)\} \rangle \rmd t \right.} \\
                                                                                & \qquad \qquad  \quad \textstyle{ \left. + \int_0^T \langle \bar{b}^1(t, \barbfX_t^1) - \bar{b}^2(t, \barbfX_t^1), \mathrm{P}(\barbfX_t^1) \rmd \bfB_t \rangle\right. }\\
                                                                                 & \qquad \qquad  \quad  \textstyle{ - \tfrac{1}{2} \int_0^T \langle \bar{b}^1(t, \barbfX_t^1) - \bar{b}^2(t, \barbfX_t^1), \mathrm{P}(\barbfX_t^1)(\bar{b}^1(t, \barbfX_t^1) + \bar{b}^2(t, \barbfX_t^1)) \rangle  \rmd t ] }\\
      & \qquad \qquad =\textstyle{ \exp[\tfrac{1}{2}\int_0^T \normLigne{ \bar{b}^1(t, \barbfX_t^1)  - \bar{b}^2(t, \barbfX_t^1)}^2 \rmd t + \int_0^T \langle \bar{b}^1(t, \barbfX_t^1) - \bar{b}^2(t, \barbfX_t^1), \mathrm{P}(\barbfX_t^1) \rmd \bfB_t \rangle ]}.
  \end{align}
  Therefore, we have that
  \begin{equation}
    \textstyle{
      \KLLigne{\bar{\Pbb}^1}{\bar{\Pbb}^2} = \tfrac{1}{2} \int_0^T \expeLigne{ \normLigne{\bar{b}^1(t, \barbfX_t^1)  - \bar{b}^2(t, \barbfX_t^1)}^2} \rmd t .
      }
    \end{equation}
    Hence, we get
    \begin{equation}
    \textstyle{
      \KLLigne{\bar{\Pbb}^1}{\bar{\Pbb}^2} = \tfrac{1}{2} \int_0^T \expeLigne{ \normLigne{b^1(t, \bfX_t^1)  - b^2(t, \bfX_t^1)}^2} \rmd t .
      }
    \end{equation}
    which concludes the proof.
  \end{proof}  
  Once \Cref{prop:girsanov_manifold} is established, we can obtain the following
  straightforward extension of \citet[Proposition 4.6]{cattiaux2021time}.

  \begin{proposition}{}{hyp_317}
    Assume \rref{assum:manifold}. 
    Let $\Qbb$ be a Brownian motion with $\Qbb_0 = \piinv$ and $\Pbb$ a path
    measure on $\rmc(\ccint{0,T}, \M)$ such that $\KL{\Pbb}{\Qbb} <
    +\infty$. Then, there exist $\beta_\Pbb, \beta_{R(\Pbb)}: \ \ccint{0,T} \times \M \to $
    such that for any $t \in \ccint{0,T}$ and $x \in \M$,
    $\beta_\Pbb(t,x), \beta_{R(\Pbb)}(t,x) \in \mathrm{T}_x \M$. In addition, we have that
    $\Pbb$ and $R(\Pbb)$ satisfy martingale problems with infinitesimal generator
    $\generator_{\Pbb}$, respectively $\generator_{R(\Pbb)}$ where for any $t \in \ccint{0,T}$, $u \in \rmc^2(\M)$ and
    $x \in \M$ we have
    \begin{align}
      &\generator_{\Pbb}(t,u)(x) = \langle \beta_\Pbb(t,x), \nabla u(x) \rangle + \tfrac{1}{2} \Delta_\M u(x)  , \\
      &\generator_{R(\Pbb)}(t,u)(x) = \langle \beta_{R(\Pbb)}(t,x), \nabla u(x) \rangle + \tfrac{1}{2} \Delta_\M u(x)  . 
    \end{align}
    Finally, we have that
    \begin{equation}
      \textstyle{
        \int_0^T \expeLigne{\norm{\beta_\Pbb(t, \bfX_t)}^2} \rmd t + \int_0^T \expeLigne{\norm{\beta_{R(\Pbb)}(t, \bfX_{T-t})}^2} \rmd t < +\infty  ,
        }
    \end{equation}
    where $(\bfX_t)_{t \in \ccint{0,T}}$ has distribution $\Pbb$.
  \end{proposition}

  \begin{proof}{}{}
    The proof is straightforward upon combining \cref{prop:girsanov_manifold}
    and the fact that
    $\KL{\Pbb}{\Qbb} = \KL{R(\Pbb)}{R(\Qbb)} = \KL{R(\Pbb)}{\Qbb} < +\infty$,
    using that $\Qbb$ is stationary.
  \end{proof}

  We conclude this section, with the following application of \Cref{thm:ibp_cattiaux}.

  \begin{proposition}{}{cattiaux_spec}
    For any $u, v \in \rmc^\infty_c(\M)$, we have that for almost any $t \in \ccint{0,T}$
    \begin{equation}
      \label{eq:equalitu}
      \expeLigne{v(\bfX_t) (\langle \beta_\Pbb(t, \bfX_t) + \beta_{R(\Pbb)}(T-t, \bfX_t), \nabla u(\bfX_t) \rangle +\Delta_\M u(\bfX_t)) + \langle \nabla u(\bfX_t), \nabla v(\bfX_t) \rangle} = 0  . 
    \end{equation}
  \end{proposition}

  \begin{proof}{}{}
  Remark that $\rmc^2_c(\M) \subset \dom(\carrechamp_\Pbb)$ and
  $\rmc^2_c(\M) \subset \dom(\carrechamp_{R(\Pbb)})$. In addition, we have that for any
  $u,v \in \rmc^2_c(\M)$,
  $\carrechamp_\Pbb(u,v) = \carrechamp_{R(\Pbb)}(u,v) = \langle u, v \rangle$. Note that
  by \Cref{prop:hyp_317} and \Cref{thm:ibp_cattiaux} we have that
  for any $u, v \in \rmc^\infty_c(\M)$, \eqref{eq:equalitu} holds.    
  \end{proof}
\subsubsection{Concluding the proof}
\label{sec:concluding-proof}

Using \cref{prop:cattiaux_spec} we can now conclude the proof of \Cref{thm:time_reversal_manifold}.
First, remark that we can identify $\beta_\Pbb = b$. Let $u, v \in \rmc^\infty(\M)$, we have that 
    \begin{equation}
      \label{eq:equality_fin}
      \expeLigne{v(\bfX_t) \langle b(\bfX_t) + \beta_{R(\Pbb)}(T-t, \bfX_t), \nabla u(\bfX_t) \rangle + \Delta_\M u(\bfX_t) v(\bfX_t)+ \langle \nabla u(\bfX_t) , \nabla v(\bfX_t) \rangle } = 0  . 
    \end{equation}
    Using that for any $t \in \ccint{0,T}$, $\Pbb_t$ admits a smooth positive
    density w.r.t. $\piinv$ denoted $p_t$ and the divergence theorem, see
    \citep[p.51]{lee2018introduction}, we have that for any $t \in \ccint{0,T}$,
\begin{align}
  &    \textstyle{\int_{\M} \{ \langle \beta_{R(\Pbb)}(T-t, x), \nabla u(x) \rangle + \langle b(x), \nabla u(x) \rangle \} v(x) p_t(x) \rmd \piinv(x)} \\
    & \qquad \qquad \qquad \qquad = - \textstyle{\int_\M \langle \nabla u(x) p_t(x), \nabla v(x) \rangle \rmd \piinv(x) - \int_\M \Delta_\M u (x) v(x)  p_t(x) \rmd \piinv(x) } \\
   & \qquad \qquad \qquad \qquad = \textstyle{\int_\M  \langle \nabla \log p_t(x), \nabla u(x) v(x) p_t(x)\rmd \piinv(x) }  . 
\end{align}
Therefore, we get that for any $t \in \ccint{0,T}$ and $x \in \M$,
$\langle \beta_{R(\Pbb)}(T-t, x), \nabla u(x) \rangle = \langle 
-b(x) + \nabla\log p_t(x), \nabla u(x) \rangle$, which concludes the proof.


\section{Convergence of RSGM}
\label{sec:convergence-rsgm}

In this section, we study the convergence of RSGM and prove
\Cref{thm:weak_qualitative}.  We state our main results in
\Cref{sec:main-results} and give discretization bounds following the recent work
of \textcite{cheng2022theorymanifold} in {sec:discr-bounds-grw}.

\subsection{Main results}
\label{sec:main-results}

In this section, we prove \Cref{thm:weak_qualitative}. We start by recalling the
sequence considered in RSGM.  Let $(Y_k)_{k \in \{0, \dots, N\}}$ be given by
$Y_0 \sim p_{\mathrm{ref}}$ and for any $k \in \{0, \dots, N-1\}$
  \begin{equation}
    Y_{k+1} = \exp_{Y_k}[\gamma \mathbf{s}_{\theta^\star}(T - n\gamma, Y_k) + \sqrt{2} Z_{k+1}]  ,
  \end{equation}
  where $\{Z_k\})_{n \in \nset}$ is a sequence of independent square integrable
  random variables with zero mean and identity covariance matrix. For ease of
  reading, we restate \Cref{thm:weak_qualitative}. 

    \begin{theorem}{}{weak_qualitative_appendix}
    Assume \rref{assum:manifold}, that $p_0$ is smooth and positive and that
    there exists $\Mtt \geq 0$ such that for any $t \in \ccint{0,T}$ and
    $x \in \M$, $\norm{\mathbf{s}_{\theta^{\star}}(t,x) - \nabla \log p_{t}(x)} \leq \Mtt$, 
              with
              $\mathbf{s}_{\theta^\star} \in \rmc(\ccint{0,T},
              \mathcal{X}(\M))$. Then if $T > 1/2$, there exists $C \geq 0$ independent on $T$ such
              that 
              \begin{equation}
                \textstyle{
                 \wassersteinD[1](\mathcal{L}(Y_N), p_0) = C (  \rme^{-\lambda_1 T} + \sqrt{T/2}   \mathtt{M} + \rme^T \gamma^{1/2})  ,
                }
              \end{equation}
              where $\wassersteinD[1]$ is the Wasserstein distance of order one on the probability measures on $\M$.
            \end{theorem}

            \begin{proof}{}{}
              For any $k \in \{1, \dots, N\}$, denote $\Rker_k$ such that for
              any $x \in \rset^d$, $\msa \in \mcb{\rset^d}$ and
              $k \in \{0, \dots, N-1\}$ we have
    \begin{equation}
      \textstyle{\expeLigne{\Rker_{k+1}(Y_k, \msa)} = \expeLigne{\1_{\msa}(Y_{k+1})}  . }
    \end{equation}
    Define for any $k_0, k_1 \in \{1, \dots, N\}$ with $k_1 \geq k_0$
    $\Qker_{k_0, k_1} = \prod_{\ell = k_0}^{k_1} \Rker_{k_1 + k_0 -
      \ell}$. Finally, for ease of notation, we also define for any
    $k \in \{1, \dots, N\}$, $\Qker_k = \Qker_{k+1,N}$. Note that for any
    $k \in \{1, \dots, N\}$, $Y_k$ has distribution $\pi_{\infty} \Qker_k$,
    where $\pi_{\infty}\in \Pens(\M)$ with density w.r.t. the Hausdorff measure
    $p_{\mathrm{ref}}$.  Let $\Pbb \in \Pens(\contspace)$ be the probability
    measure associated with $(\bfB_t)_{t \in \ccint{0,T}}$ with
    $\bfB_0 \sim \pi_0$, where $\pi_0 \in \Pens(\M)$ admits a density w.r.t. the
    Hausdorff measure given by $p_0$. We denote $(\bfhY_t)_{t \in \ccint{0,T}}$
    the process defined by the diffusion
    $\rmd \bfhY_t = \mathbf{s}_{\theta^\star}(T-t, \bfhY_t) \rmd t + \rmd
    \bfB_t$ and $\bfhY_0 \sim \pi_\infty$. We also denote
    $\hat{\Pbb}^R \in \Pens(\contspace)$ the probability measure associated with
    $(\bfhY_t)_{t \in \ccint{0,T}}$.  First note that using that
    $\Pbb_0 = \pi_0$ we have for any $\msa \in \mcb{\M}$
    \begin{equation}
      \pi_0 \Pbb_{T|0} (\Pbb^R)_{T|0} (\msa) = \Pbb_T (\Pbb^R)_{T|0} (\msa) = (\Pbb^R)_0 (\Pbb^R)_{T|0} (\msa)  = (\Pbb^R)_T(\msa) = \pi_0(\msa)  . 
    \end{equation}
    Hence we have that
    \begin{equation}
      \label{eq:equality_backward}
     \pi_0 = \pi_0 \Pbb_{T|0} (\Pbb^R)_{T|0}  .  
   \end{equation}
   Let $\varphi \in \rmc(\M)$ with is $1$-Lipschitz, i.e. for any
   $x,y \in \M$, $\absLigne{\varphi(x) - \varphi(y)} \leq d(x,y)$. Since $\M$
   is compact, we have that $\varphi$ is bounded. Using this result,
   \eqref{eq:equality_backward}, the data processing theorem \cite[Theorem
   4.1]{kullback1997information} and Pinsker's inequality \cite[Equation
   5.2.2]{bakry:gentil:ledoux:2014} we have
   \begin{align}
     & \textstyle{\absLigne{\expeLigne{\varphi(Y_N)} - \int_\M \varphi(x) p_0(x) \rmd \mu(x)}}  \\
     & \leq \absLigne{\expeLigne{\varphi(\bfB_0)} - \expeLigne{\varphi(\bfY_T)}} 
                                                                                                             + \absLigne{\expeLigne{\varphi(\bfhY_T)} - \expeLigne{\varphi(\bfY_T)}} \absLigne{\expeLigne{\varphi(\bfhY_T)} - \expeLigne{\varphi(Y_N)}} \\
      &  \leq \normLigne{\varphi}_\infty \tvnorm{\pi_0 - \pi_{\infty} (\mathbb{P}^R)_{T|0}}  + \absLigne{\expeLigne{\varphi(\bfhY_T)} - \expeLigne{\varphi(\bfY_T)}}  + \absLigne{\expeLigne{\varphi(\bfhY_T)} - \expeLigne{\varphi(Y_N)}} \\
      &  \leq \normLigne{\varphi}_\infty \tvnorm{\pi_0 \Pbb_{T|0} (\Pbb^R)_{T|0} - \pi_{\infty} (\mathbb{P}^R)_{T|0}} \\& \qquad + \absLigne{\expeLigne{\varphi(\bfhY_T)} - \expeLigne{\varphi(\bfY_T)}}  + \absLigne{\expeLigne{\varphi(\bfhY_T)} - \expeLigne{\varphi(Y_N)}} \\
     & \leq \normLigne{\varphi}_\infty \tvnorm{\pi_0 \Pbb_{T|0} - \pi_{\infty} }  + \absLigne{\expeLigne{\varphi(\bfhY_T)} - \expeLigne{\varphi(\bfY_T)}} + \absLigne{\expeLigne{\varphi(\bfhY_T)} - \expeLigne{\varphi(Y_N)}} \\
      & \leq \normLigne{\varphi}_\infty \tvnorm{\pi_0 \Pbb_{T|0} - \pi_{\infty} }  + \sqrt{2}\normLigne{\varphi}_\infty\KLLignesqrt{\pi_{\infty}\Pbb^R_{|0}}{\pi_{\infty}\hat{\Pbb}^R_{|0}} + \absLigne{\expeLigne{\varphi(\bfhY_T)} - \expeLigne{\varphi(Y_N)}}  .                
   \end{align}

   We now control each one of these terms. The first term can be easily
   controlled using the geometric ergodicity of the Brownian motion on compact
   manifolds. The second term can be controlled using the Girsanov theory on
   isometrically embedded manifolds. For the last term, we rely on the
   convergence of the GRW to its associated diffusion as presented in
   \Cref{sec:discr-bounds-grw}. We now control each one of these terms.
 \begin{enumerate}[wide, labelindent=0pt, label=(\alph*)]
 \item Using \Cref{prop:brownian_conv_repeat}, we have that
   $\tvnorm{\pi_0 \Pbb_{T|0} - \pi_{\infty} } \leq C^{1/2} \rme^{\lambda_1/2}
   \rme^{-\lambda_1 T}$ where $\lambda_1$ is the first positive eigenvalue of
   $-\Delta_\M$ in $\mathrm{L}^2(\pi_\infty)$. Therefore, we get that
   \begin{equation}
     \label{eq:convergence_brownian}
     \normLigne{\varphi}_\infty \tvnorm{\pi_0 \Pbb_{T|0} - \pi_{\infty} } \leq C^{1/2} \rme^{\lambda_1/2} \normLigne{\varphi}_\infty 
   \rme^{-\lambda_1 T}  . 
   \end{equation}
 \item Recall that we have that $\Pbb^R_{|0}$ is associated with the process
   $\rmd \bfY_t = \nabla \log p_{T-t}(\bfY_t) \rmd t + \rmd \bfB_t^\M$ and that
   $\hat{\Pbb}^R_{|0}$ is associated with the process
   $\rmd \bfhY_t = \mathbf{s}_{\theta^\star}(T-t, \bfhY_t) \rmd t + \rmd
   \bfB_t^\M$. Using \Cref{col:KL_diff} we have that
   \begin{equation}
     \label{eq:approx_score}
     \textstyle{
       \KLLigne{\pi_{\infty}\Pbb^R_{|0}}{\pi_{\infty}\hat{\Pbb}^R_{|0}} = \tfrac{1}{2} \int_{0}^T \expeLigne{\normLigne{\mathbf{s}_{\theta^\star}(T-t, \bfY_t) - \nabla \log p_{T-t}(\bfY_t)}^2} \leq M^2 T  .
       }
     \end{equation}
   \item Let us define $\{\bar{\bfY}^k\}_{k=0}^N$ such that for any
     $k \in \{0, \dots, N\}$, $\bar{\bfY}^k_0 = \hat{\bfY}_0 = Y_0$ and for any
     $t \in \ccint{0, k \gamma}$ we have that $\bar{\bfY}_t^0 = \hat{\bfY}_t$. For
     any $t \in \ccint{k \gamma, T}$, we have that $\bar{\bfY}_t^k = Y_{t,k}$,
     where $Y_{k \gamma, k} = \hat{\bfY}_{k \gamma}$ and for any
     $j \in \{k, \dots, N-1\}$ and $t \in \ccint{0, \gamma}$
     \begin{equation}
       Y_{j \gamma +t, k} = \exp_{Y_{j \gamma, k}}[t \mathbf{s}_{\theta^\star}(T - j \gamma, Y_{j \gamma, k}) + \sqrt{t} E_j^k Z_j ]  , 
     \end{equation}
     where $\{Z_j\}_{j=0}^{N-1}$ are independent Gaussian random variables with
     identity covariance matrix and zero mean and $E_j^k$ is a frame of
     $\mathrm{T}_{Y_{j \gamma, k}} \M$ such that for any
     $j \in \{k+1, \dots, N-1\}$,
     $E_{j}^{k+1} = \Gamma_{Y_{j \gamma, k}}^{{Y_{j \gamma, k+1}}} E_j^k$ and
     $\{E_j^0\}_{j=0}^{N-1}$ is such that for any $j \in \{0, \dots, N-1\}$,
     $E_j^0$ is a frame of $\mathrm{T}_{Y_{j \gamma}} \M$. One
     $\ccint{0, k\gamma}$, we define
     $(\hat{\bfY}_t^k)_{t \in \ccint{0, k\gamma}}$ as follows. 
     For any $k \in \{0, \dots, N-1\}$, we set
     $(\bfY_t^{k+1})_{t \in \ccint{0, k\gamma}} = (\bfY_t^k)_{t \in \ccint{0,
         k\gamma}}$. For any $k \in \{0, \dots, N-1\}$, we set
     $(\bfY_t)_{k\gamma, (k+1)\gamma}$ as in \Cref{prop:borne_one_step} (taking
     the notations of \Cref{prop:borne_one_step},
     $X_1^0 = \hat{\bfY}_{(k+1)\gamma}^k$ and
     $\bfX_\gamma = \hat{\bfY}_{k \gamma}^k$). Note that we have 
     $\{\bar{\bfY}_{j\gamma,0}^N\}_{j=0}^N = \{Y_j^N\}_{j=0}^N$ and
     $\{\bar{\bfY}_{t, N}\}_{t \in \ccint{0,T}} = \{\hat{\bfY}_t\}_{t \in
       \ccint{0,T}}$. Therefore, we have that
     \begin{align}
       \textstyle{ \absLigne{\varphi(\hat{\bfY}_T) - \varphi(Y_N)}} &=\absLigne{\varphi(\bar{\bfY}_{T}^0) - \varphi(\bar{\bfY}_{T}^N)} \\
       &\textstyle{\leq \sum_{k=0}^{N-1} \absLigne{\varphi(\bar{\bfY}_{T}^k) - \varphi(\bar{\bfY}_{T}^{k+1})} \leq \normLigne{\nabla \varphi}_\infty \sum_{k=0}^{N-1} d(\bar{\bfY}_{T}^k, \bar{\bfY}_{T}^{k+1}) }  . 
     \end{align}
     In addition, using \Cref{prop:borne_one_step} and \Cref{prop:explosion}, we
     have that there exists $C \geq 0$ such that for any
     $k \in \{0, \dots, N-1\}$ 
     \begin{equation}
       \expeLigne{d(\bar{\bfY}_{k,T}, \bar{\bfY}_{k+1, T})} \leq C \exp[(N-k) \gamma] \gamma^{3/2}  . 
     \end{equation}
     Therefore, we get that there exists $C \geq 0$ such that 
     \begin{equation}
       \absLigne{\expeLigne{\varphi(\bfhY_T)} - \expeLigne{\varphi(Y_N)}} \leq C \exp[T] \gamma^{1/2}  ,
     \end{equation}
   \end{enumerate}

   Therefore, we get that there exists $C \geq 0$ such that for any
   $\varphi \in \rmc(\M)$ which is $1$-Lipschitz, we have
   \begin{equation}
     \label{eq:inter}
     \textstyle{
       \expeLigne{\varphi(Y_N)} - \int_{\M} \varphi(x) p_{0}(x) \rmd \piinv(x) \leq C (\rme^{\lambda_1/2}\normLigne{\varphi}_\infty \rme^{-\lambda_1 T} + \sqrt{T/2}  \normLigne{\varphi}_\infty \mathtt{M} + \rme^T \gamma^{1/2}) .
       }
   \end{equation}   



              Let $x_0 \in \M$. Let $\Lip(\M)$ the set of Lipschitz functions on
              $\M$ with Lipschitz constant equal to $1$. Let $\Lip(\M)_0$ the
              set of Lipschitz functions on $\M$ with Lipschitz constant equal
              to $1$ and such that for any $\varphi \in \Lip(\M)_0$,
              $\varphi(x_0) =0$. Note that in this case, we have that
              $\normLigne{\varphi}_\infty \leq 
              \mathrm{diam}(\M)$.  Using \eqref{eq:inter}, we have 
              \begin{align}
                \wassersteinD[1](\mathcal{L}(Y_N), p_0) &= \textstyle{\sup \ensembleLigne{\expeLigne{\varphi(Y_N)} - \int_{\M} \varphi(x) p_{0}(x) \rmd \piinv(x)}{\varphi \in \Lip(\M)}} \\
                                                        &= \textstyle{\sup \ensembleLigne{\expeLigne{\varphi(Y_N)} - \int_{\M} \varphi(x) p_{0}(x) \rmd \piinv(x)}{\varphi \in \Lip(\M)_0}} \\
                &\leq C (\rme^{\lambda_1/2}\mathrm{diam}(\M) \rme^{-\lambda_1 T} + \sqrt{T/2}  \mathrm{diam}(\M) \mathtt{M} + \rme^T \gamma^{1/2})   , 
              \end{align}
              which concludes the proof.
            \end{proof}

            We now state a result regarding the continuous-time process (i.e. we
            now longer consider discretization errors). We recall that we denote 
            $(\bfhY_t)_{t \in \ccint{0,T}}$ the process defined by the diffusion
            $\rmd \bfhY_t = \mathbf{s}_{\theta^\star}(T-t, \bfhY_t) \rmd t +
            \rmd \bfB_t$ and $\bfhY_0 \sim \pi_\infty$.

    \begin{theorem}{}{weak_qualitative_continuous}
    Assume \rref{assum:manifold}, that $p_0$ is smooth and positive and that
    there exists $\Mtt \geq 0$ such that for any $t \in \ccint{0,T}$ and
    $x \in \M$, $\norm{\mathbf{s}_{\theta^{\star}}(t,x) - \nabla \log p_{t}(x)} \leq \Mtt$, 
              with
              $\mathbf{s}_{\theta^\star} \in \rmc(\ccint{0,T},
              \mathcal{X}(\M))$. Then if $T > 1/2$, there exists $C \geq 0$ independent on $T$ such
              that 
              \begin{equation}
                \textstyle{
                 \tvnorm{\mathcal{L}(\bfhY_T) - p_0} = C (  \rme^{-\lambda_1 T} + \sqrt{T/2}   \mathtt{M} ) .
                }
              \end{equation}
            \end{theorem}

            \begin{proof}
              The proof is identical to the one of
              \Cref{thm:weak_qualitative_appendix}, except that we do not have
              to deal with the discretization error. We use that for any $\mu, \nu \in \Pens(\M)$
              \begin{equation}
                \tvnorm{\mu - \nu} = \sup \ensembleLigne{\mu[f] - \nu[f]}{f \in \rmc(\M), \normLigne{f}_\infty \leq 1} .
              \end{equation}
            \end{proof}

            The result of \Cref{thm:weak_qualitative_continuous} should be
            compared with the one of \cite[Theorem 3]{rozen2021moser}. With our
            result we control a $\mathrm{L}^1$ bound between the density of
            $\hat{\bfY}_T$ and the one of $p_0$. In \cite[Theorem
            3]{rozen2021moser} a $\mathrm{L}^\infty$ bound between the densities
            is recovered. It can be shown that
            $\hat{p}_T = \mathcal{L}(\hat{\bfY}_T)$. Let $\kappa$ be the modulus
            of continuity of $\hat{p}_T - p_0$, i.e. for any $\vareps \geq 0$
            \begin{equation}
              \kappa(\vareps) = \sup \ensembleLigne{\absLigne{\hat{p}_T(x) - p_0(x) - \hat{p}_T(y) + p_0(y)}}{x, y \in \M, \ d(x,y) \leq \vareps} .
            \end{equation}
            Let $x_0 \in \M$ such that
            \begin{equation}
              \absLigne{\hat{p}_T(x_0) - p_0(x_0)} = M = \sup \ensembleLigne{\absLigne{\hat{p}_T(x) - p_0(x)}}{x \in \M} .
            \end{equation}
            For any $x \in \cball{x_0}{\kappa(M/2)}$, we have
            $\absLigne{\hat{p}_T(x) - p_0(x)} \geq M/2$. Hence, denoting
            $\mathrm{Vol}_\kappa = \int_{\cball{x_0}{\kappa(M/2)}} \rmd
            \piinv(x) > 0$, we have
            \begin{equation}
              \textstyle{
                (2/\mathrm{Vol}_\kappa) \int_{\M} \absLigne{\hat{p}_T(x) - p_0(x)} \rmd \piinv(x) \geq \normLigne{\hat{p}_T - p_0}_\infty \eqsp . 
                }
              \end{equation}
              Hence, there exists $C \geq 0$ such that for any $T > 1/2$
              \begin{equation}
                \normLigne{\hat{p}_T - p_0}_\infty \leq C (  \rme^{-\lambda_1 T} + \sqrt{T/2}   \mathtt{M} ) .
              \end{equation}
              Therefore, we recover the same guarantees as
              \Cref{thm:weak_qualitative_continuous} (note that $\mathtt{M}$ is
              not explicitly controlled using network properties in our work,
              but we could use universal approximation properties as in
              \textcite{rozen2021moser} in order to obtain a similar result).

\subsection{Discretization bounds for GRW}
\label{sec:discr-bounds-grw}

In this section, we establish discretization bounds for GRW. Our results are a
straightforward extension of \textcite{cheng2022theorymanifold} to the case where
the drift term in the GRW is time-inhomogeneous. 

Since $\M$ is compact, we have that for any $x_1, x_2 \in M$, there exists a
minimizing geodesic such that $\gamma \in \rmc^\infty(\ccint{0,1}, \M)$ and
$\gamma(0) = x_1$ and $\gamma(1) = x_2$. When this choice is not unique we fix a
minimizing geodesic. We denote
$\Gamma_{x_1}^{x_2}: \ \mathrm{T}_{x_1} \M \to \mathrm{T}_{x_2} \M$ the
associated parallel transport. Let $b \in \rmc^\infty(\ccint{0,T}, \XM)$.

We start by introducing a family of GRWs defined on progressively finer
grids. Let $\gamma >0$, $X_0 \in \M$, $E_0 \in \mathrm{F}_{X_0} \M$ (the vector
space of frames at $X_0$) and consider the families
$\ensembleLigne{E_k^\ell}{k \in \{0, \dots, 2^\ell\}, \ \ell \in \nset}$,
$\ensembleLigne{X_k^\ell}{k \in \{0, \dots, 2^\ell\}, \ \ell \in \nset}$ such
that $X_0^0 = X_0$,
$X_1^0 = \exp_{X_0^0}[\gamma b(0, X_0^0) + \sqrt{\gamma} (\bfB_1 - \bfB_0)
E_0^0]$ and $E_1^0 = \Gamma_{X_0^0}^{X_1^0} E_0^0$ (note that $E_{2^\ell}^\ell$
is not used in the proof but defined for completeness). In addition, we have
that for any $\ell \in \nset$ with $\ell \geq 1$, $X_0^\ell = X_0$,
$E_0^\ell = E_0$ and for any $k \in \{0, \dots, 2^{\ell-1}-1\}$
\begin{align}
  &X_{2k+1}^{\ell} = \exp_{X_{2k}^\ell}[\gamma_\ell b(2k \gamma_\ell, X_{2k}^\ell) +  E_{2k}^\ell(\bfB_{(2k+1)\gamma_\ell} - \bfB_{2k \gamma_\ell}) ]  , \\
  &E_{2k+1}^\ell = \Gamma_{X_{2k}^\ell}^{X_{2k+1}^\ell} E_{2k}^\ell  , \\
  &X_{2k+2}^{\ell} = \exp_{X_{2k+1}^\ell}[\gamma_\ell b((2k+1) \gamma_\ell, X_{2k+1}^\ell) + E_{2k+1}^\ell (\bfB_{(2k+2)\gamma_\ell} - \bfB_{(2k+1) \gamma_\ell}) ]  , \\
  &E_{2k+2}^\ell = \Gamma_{X_{k+1}^{\ell-1}}^{X_{2k+2}^\ell} E_{k+1}^{\ell-1}  ,   \label{eq:coupling}
\end{align}
where $\gamma_\ell = \gamma / 2^\ell$.  For any $\ell \in \nset$, we also define
$(\bfX_t^\ell)_{t \in \ccint{0, \gamma}}$ such that for any $\ell \in \nset$,
$k \in \{0, \dots, 2^{\ell}-1\}$, we have for any
$t \in \coint{k \gamma_\ell, (k+1) \gamma_\ell}$ ,
$\bfX_{t}^\ell = \exp_{X_{k}^\ell}[(t - k\gamma_\ell) b(k \gamma_\ell, X_{k}^\ell) +
E_{k}^\ell  (\bfB_{t} - \bfB_{k \gamma_\ell}) ]$. Note that for any
$\ell \in \nset$ and $k \in \{0, \dots, 2^\ell -1\}$,
$\bfX_{k \gamma_\ell}^\ell = X_k$.

We are going to use the following useful lemma, see \cite[Lemma 62]{cheng2022theorymanifold}.
\begin{lemma}{}{bound_exp}
  Assume \rref{assum:manifold}. Then, there exists $C \geq 0$ such that for any
  $x, y \in \M$, $\gamma: \ccint{0,1} \to \M$  minimizing geodesic with 
  $\gamma(0) = x$, $\gamma(1) = y$ and $u \in \mathrm{T}_x \M$, 
  $v \in \mathrm{T}_y \M$ we have 
  \begin{equation}
    d(\exp_y[v], \exp_x[u])^2 \leq (1 +C \kappa^2 \exp[4 \kappa]) d(x,y)^2 + C \exp[4 \kappa] \normLigne{\Gamma_y^x v - u}^2 + 2 \langle \gamma'(0), \Gamma_y^x v - u \rangle  ,
  \end{equation}
  with $\kappa = \normLigne{u} + \normLigne{v}$.
\end{lemma}

We are now ready to state the main result of this section.

\begin{proposition}{}{bound_square}
  Assume \rref{assum:manifold}. Then, there exists $C \geq 0$ such that for any $\ell \in \nset$
  \begin{equation}
    \textstyle{ \expeLigne{\sup_{t \in \ccint{0, \gamma}} d(\bfX_t^\ell, \bfX_t^{\ell+1})^2} \leq C \gamma^3 2^{-2\ell}  . }
  \end{equation}
\end{proposition}

\begin{proof}{}{}
  Let $\ell \in \nset$, $k \in \{0, \dots, 2^{\ell} - 1\}$ and
  $t \in \ccint{k \gamma_\ell, (k+1) \gamma_\ell}$. We define
  $U_k^t = d(\bfX_t^\ell, \bfX_{t}^{\ell+1})^2$,
  $U_k = \sup \ensembleLigne{U_k^t}{t \in \ccint{k \gamma_\ell, (k+1)
      \gamma_\ell}}$ and $U_{-1}=0$. We also introduce for any
  $j \in \{0, \dots, 2^{\ell} - 1\}$ and for
  $t \in \coint{k \gamma_\ell, (2k+1) \gamma_{\ell+1}}$,
  $\bar{\bfX}_t^{\ell+1} = \bfX_t^{\ell+1}$ and for
  $t \in \coint{(2k+1) \gamma_{\ell+1}, (k+1) \gamma_{\ell}}$
  \begin{align}
    \bar{\bfX}_{t}^{\ell+1} &= \exp_{X_{2j}^{\ell+1}}[\gamma_{\ell+1} b(2 j \gamma_{\ell + 1}, X_{2j}^{\ell+1}) \\ & \quad + (t - (2k+1)\gamma_{\ell+1}) b((2j+1) \gamma_{\ell + 1}, X_{2j}^{\ell+1})  + (\bfB_{t} - \bfB_{j \gamma_\ell}) E_{2j}^{\ell+1} ]  .
  \end{align}
  Using this result and that for any $a,b \geq0$,
  $(a+b)^2 \leq (1+2^{-\ell}) a^2 + (1 +2^\ell)b^2$, we have that for any
  $t \in \ccint{k \gamma_\ell, (k+1) \gamma_\ell}$
  \begin{equation}
    \label{eq:inter_boundi_bound}
    \textstyle{
      U_{k+1}^t \leq (1 + 2^{-\ell}) d(\bfX_{t}^\ell, \bar{\bfX}_{t}^{\ell+1})^2 + (1 + 2^{\ell}) d(\bar{\bfX}_{t}^{\ell+1}, \bfX_{t}^{\ell+1})^2  .
      }
    \end{equation}
    Note that for $t \in \ccint{k\gamma_\ell, (2k+1) \gamma_{\ell+1}}$, the
    second term in \eqref{eq:inter_boundi_bound} is zero.  We now bound each one
    of these terms:
 \begin{enumerate}[wide, labelindent=0pt, label=(\alph*)]    
 \item First, we assume that
   $t \in \ccint{(k+1) \gamma_{\ell}, (2k+1) \gamma_{\ell+1}}$. Recall that 
\begin{align}
  \bar{\bfX}_{t}^{\ell+1} &= \exp_{X_{2k}^{\ell+1}}[\gamma_{\ell+1}b(k \gamma_{\ell }, X_{2k}^{\ell+1}) \\ & \quad (t - (2k+1)\gamma_{\ell+1}) b((2 k+1) \gamma_{\ell + 1}, X_{2k}^{\ell+1})  + (\bfB_{t} - \bfB_{k \gamma_\ell}) E_{2k}^{\ell+1} ]  , \\
  \bfX_{t}^\ell &= \exp_{X_k^\ell}[( t - k\gamma_\ell) b(k \gamma_\ell, X_k^\ell) + (\bfB_{t} - \bfB_{k \gamma_\ell}) E_{k}^{\ell} ]   . 
\end{align}
Hence, using \Cref{lemma:bound_exp}, we have that
\begin{align}
  \label{eq:upper_bound_distance_cheng}
  d(\bar{\bfX}_{t}^{\ell+1}, \bfX_{t}^\ell)^2 &\leq (1 +C \kappa_k^2 \exp[4 \kappa_k]) d(X_k^\ell,X_{2k}^{\ell+1})^2 \\
  & \qquad + C \exp[4 \kappa_k] \normLigne{\Gamma_{X_{2k}^{\ell+1}}^{X_k^\ell} v_k - u_k}^2 + 2 \langle w'(0), \Gamma_{X_{2k}^{\ell+1}}^{X_k^\ell} v_k - u_k \rangle  , 
\end{align}
with $w: \ \ccint{0,1} \to \M$ a minimizing geodesic between $X_k^\ell$ and
$X_{2k}^{\ell+1}$
\begin{align}
  &\kappa_k = \normLigne{u_k} + \normLigne{v_k}  , \\
  &u_k^1 = (t- k\gamma_\ell) b(k \gamma_\ell, X_k^\ell)   , \\
  &v_k^1 = \gamma_{\ell+1} b(2k \gamma_{\ell + 1}, X_{2k}^{\ell+1}) + (t - (2k+1) \gamma_{\ell+1}) b((2k+1) \gamma_{\ell + 1}, X_{2k}^{\ell+1})  , \\
  &u_k^2 = (\bfB_{t} - \bfB_{k \gamma_\ell}) E_{k}^{\ell}  , \qquad \quad \quad \  v_k^2 = (\bfB_{t} - \bfB_{k \gamma_\ell}) E_{2k}^{\ell+1}  , \\
    &u_k = u_k^1 + u_k^2 , \qquad \qquad \qquad \quad \ v_k = v_k^1 + v_k^2  .
\end{align}
In particular, since
$E_k^\ell = \Gamma_{X_{2k}^{\ell+1}}^{X_{k}^\ell} E_{2k}^{\ell+1}$ using
\eqref{eq:coupling}, we have that
$u_k^2 = \Gamma_{X_{2k}^{\ell+1}}^{X_{k}^\ell} v_k^2$. Therefore, combining this
result and that
$t- (2k+1)\gamma_{\ell+1} + \gamma_{\ell+1} = t - k \gamma_\ell$, we get that
\begin{align}
  \normLigne{\Gamma_{X_{2k}^{\ell+1}}^{X_{k}^\ell} v_k^1 - u_k^1} &\leq \gamma_{\ell+1} \normLigne{b(k \gamma_\ell, X_k^\ell) - \Gamma_{X_{2k}^{\ell+1}}^{X_{k}^\ell} b(k \gamma_\ell, X_{2k}^{\ell+1})} \\
  & \qquad + \gamma_{\ell+1} \normLigne{b(k \gamma_\ell, X_k^\ell) - \Gamma_{X_{2k}^{\ell+1}}^{X_{k}^\ell} b((2 k +1) \gamma_{\ell+1}, X_{2k}^{\ell+1})} \\
  &\leq \gamma_{\ell} \normLigne{b(k \gamma_\ell, X_k^\ell) - \Gamma_{X_{2k}^{\ell+1}}^{X_{k}^\ell} b(k \gamma_\ell, X_{2k}^{\ell+1})} + \Ltt_2 \gamma_{\ell}^2 \\
  &\leq \Ltt_1 \gamma_\ell  d(X_k^\ell, X_{2k}^{\ell+1}) + \Ltt_2 \gamma_\ell^2  .
\end{align}
Therefore, we get that
$\normLigne{u_k - v_k} \leq \Ltt_1 \gamma_\ell d(X_k^\ell, X_{2k}^{\ell+1}) +
\Ltt_2 \gamma_\ell^2$. In addition, we have that
$\normLigne{w'(0)} \leq d(X_k^\ell, X_{2k}^{\ell+1})$ since $w$ is a minimizing
geodesic. Combining these results and \eqref{eq:upper_bound_distance_cheng} we
get that
\begin{align}
  &d(\bar{\bfX}_{t}^{\ell+1}, \bfX_{t}^\ell)^2 \leq (1 +C \kappa_k^2 \exp[4 \kappa_k]) d(X_k^\ell,X_{2k}^{\ell+1})^2 \\
                                             & \qquad \qquad  \qquad + C \exp[4 \kappa_k] (\Ltt_1 \gamma_\ell  d(X_k^\ell, X_{2k}^{\ell+1}) + \Ltt_2 \gamma_\ell^2)^2 \\
                                             & \qquad \qquad  \qquad + 2 (\Ltt_1 \gamma_\ell  d(X_k^\ell, X_{2k}^{\ell+1}) + \Ltt_2 \gamma_\ell^2)d(X_k^\ell, X_{2k}^{\ell+1}) \\
                                             & \qquad \qquad \leq (1 +C \kappa_k^2 \exp[4 \kappa_k] + 2C \exp[4 \kappa_k] \Ltt_1^2 \gamma_\ell^2) d(X_k^\ell,X_{2k}^{\ell+1})^2 \\
                                             & \qquad \qquad  \qquad + 2 (\Ltt_1 \gamma_\ell  d(X_k^\ell, X_{2k}^{\ell+1}) + \Ltt_2 \gamma_\ell^2)d(X_k^\ell, X_{2k}^{\ell+1}) + 2 \Ltt_2^2 \gamma_\ell^4 \\
                                             & \qquad \qquad \leq (1 +C \kappa_k^2 \exp[4 \kappa_k] + 2C \exp[4 \kappa_k] \Ltt_1^2 \gamma_\ell^2 + 2 \Ltt_1 \gamma_\ell + 4 \Ltt_2 \gamma_\ell) d(X_k^\ell,X_{2k}^{\ell+1})^2  + 8 \Ltt_2 \gamma_\ell^3  ,
\end{align}
Hence, there exists $C_1 \geq 0$ (not dependent on $k$ or $\ell$) such that
\begin{equation}
  \label{eq:case_numeroooo_unoooo}
    (1 + 2^{-\ell}) d(\bar{\bfX}_{t}^{\ell+1}, \bfX_{t}^\ell)^2 \leq (1 +C_1\{ \kappa_k^2 \exp[4 \kappa_k] + \gamma_\ell^2 \exp[4\kappa_k] + 2^{-\ell}\}) d(X_k^\ell,X_{2k}^{\ell+1})^2 +  C_1 \gamma_\ell^3  . 
  \end{equation}
  Next, we assume that $t \in \ccint{k \gamma_\ell, (2k+1)
    \gamma_{\ell+1}}$. Recall that
\begin{align}
  \bar{\bfX}_{t}^{\ell+1} &= \exp_{X_{2k}^{\ell+1}}[(t - k\gamma_{\ell}) b( k \gamma_{\ell}, X_{2k}^{\ell+1})  + (\bfB_{t} - \bfB_{k \gamma_\ell}) E_{2k}^{\ell+1} ]  , \\
  \bfX_{t}^\ell &= \exp_{X_k^\ell}[( t - k\gamma_\ell) b(k \gamma_\ell, X_k^\ell) + (\bfB_{t} - \bfB_{k \gamma_\ell}) E_{k}^{\ell} ]   . 
\end{align}
Hence, using \Cref{lemma:bound_exp}, we have that
\begin{align}
  \label{eq:upper_bound_distance_cheng_duo}
  d(\bar{\bfX}_{t}^{\ell+1}, \bfX_{t}^\ell)^2 &\leq (1 +C \kappa_k^2 \exp[4 \kappa_k]) d(X_k^\ell,X_{2k}^{\ell+1})^2 \\
  & \qquad + C \exp[4 \kappa_k] \normLigne{\Gamma_{X_{2k}^{\ell+1}}^{X_k^\ell} v_k - u_k}^2 + 2 \langle w'(0), \Gamma_{X_{2k}^{\ell+1}}^{X_k^\ell} v_k - u_k \rangle  , 
\end{align}
with $w: \ \ccint{0,1} \to \M$ a minimizing geodesic between $X_k^\ell$ and
$X_{2k}^{\ell+1}$
\begin{align}
  &\kappa_k = \normLigne{u_k} + \normLigne{v_k}  , \\
  &u_k^1 = (t- k\gamma_\ell) b(k \gamma_\ell, X_k^\ell)   , \\
  &v_k^1 = (t - k\gamma_{\ell}) b(k \gamma_{\ell }, X_{2k}^{\ell+1})  , \\
  &u_k^2 = (\bfB_{t} - \bfB_{k \gamma_\ell}) E_{k}^{\ell}  , \qquad \quad \quad \  v_k^2 = (\bfB_{t} - \bfB_{k \gamma_\ell}) E_{2k}^{\ell+1}  , \\
    &u_k = u_k^1 + u_k^2 , \qquad \qquad \qquad \quad \ v_k = v_k^1 + v_k^2  .
\end{align}
In particular, since
$E_k^\ell = \Gamma_{X_{2k}^{\ell+1}}^{X_{k}^\ell} E_{2k}^{\ell+1}$ using
\eqref{eq:coupling} and
$t- (2k+1)\gamma_{\ell+1} + \gamma_{\ell+1} = t - k \gamma_\ell$, we have that
$u_k^2 = \Gamma_{X_{2k}^{\ell+1}}^{X_{k}^\ell} v_k^2$. Therefore, we get that 
\begin{align}
  \normLigne{\Gamma_{X_{2k}^{\ell+1}}^{X_{k}^\ell} v_k^1 - u_k^1} &\leq \gamma_{\ell+1} \normLigne{b(k \gamma_\ell, X_k^\ell) - \Gamma_{X_{2k}^{\ell+1}}^{X_{k}^\ell} b(k \gamma_\ell, X_{2k}^{\ell+1})}  \\
  &\leq \gamma_{\ell} \normLigne{b(k \gamma_\ell, X_k^\ell) - \Gamma_{X_{2k}^{\ell+1}}^{X_{k}^\ell} b(k \gamma_\ell, X_{2k}^{\ell+1})} + \Ltt_2 \gamma_{\ell}^2 \\
  &\leq \Ltt_1 \gamma_\ell  d(X_k^\ell, X_{2k}^{\ell+1})  .
\end{align}
Therefore, we get that
$\normLigne{u_k - v_k} \leq \Ltt_1 \gamma_\ell d(X_k^\ell, X_{2k}^{\ell+1})$. In addition, we have that
$\normLigne{w'(0)} \leq d(X_k^\ell, X_{2k}^{\ell+1})$ since $w$ is a minimizing
geodesic. Combining these results and \eqref{eq:upper_bound_distance_cheng_duo} we
get that
\begin{align}
  &d(\bar{\bfX}_{t}^{\ell+1}, \bfX_{t}^\ell)^2 \leq (1 +C \kappa_k^2 \exp[4 \kappa_k]) d(X_k^\ell,X_{2k}^{\ell+1})^2 \\
                                             & \qquad \qquad  \qquad + C \exp[4 \kappa_k] \Ltt_1^2 \gamma_\ell^2  d(X_k^\ell, X_{2k}^{\ell+1})^2 \\
                                             & \qquad \qquad  \qquad + 2 \Ltt_1 \gamma_\ell  d(X_k^\ell, X_{2k}^{\ell+1})d(X_k^\ell, X_{2k}^{\ell+1}) \\
                                             & \qquad \qquad \leq (1 +C \kappa_k^2 \exp[4 \kappa_k] + 2C \exp[4 \kappa_k] \Ltt_1^2 \gamma_\ell^2) d(X_k^\ell,X_{2k}^{\ell+1})^2 \\
                                             & \qquad \qquad  \qquad + 2 \Ltt_1 \gamma_\ell  d(X_k^\ell, X_{2k}^{\ell+1})^2 + 2 \Ltt_2^2 \gamma_\ell^4 \\
                                             & \qquad \qquad \leq (1 +C \kappa_k^2 \exp[4 \kappa_k] + 2C \exp[4 \kappa_k] \Ltt_1^2 \gamma_\ell^2 + 2 \Ltt_1 \gamma_\ell) d(X_k^\ell,X_{2k}^{\ell+1})^2  . 
\end{align}
Hence, there exists $C_1 \geq 0$ (not dependent on $k$ or $\ell$) such that for
any $t \in \ccint{k \gamma_\ell, (k+1) \gamma_\ell}$
\begin{equation}\label{eq:case_complete}
    (1 + 2^{-\ell}) d(\bar{\bfX}_{t}^{\ell+1}, \bfX_{t}^\ell)^2 \leq (1 +C_1\{ \kappa_k^2 \exp[4 \kappa_k] + \gamma_\ell^2 \exp[4\kappa_k] + 2^{-\ell}\}) d(X_k^\ell,X_{2k}^{\ell+1})^2 +  C_1 \gamma_\ell^3  . 
  \end{equation} 
\item We recall that if $t \in \ccint{k \gamma_{\ell}, (2k+1) \gamma_{\ell+1}}$
  the second term in \eqref{eq:inter_boundi_bound} is zero. Therefore in what
  follows, we assume $t \in \ccint{(2k+1) \gamma_{\ell+1}, (k+1) \gamma_\ell}$.
  We introduce
  \begin{align}
  \hat{\bfX}_{t}^{\ell+1} &= \exp_{X_{2k+1}^{\ell+1}}[(t - (2k+1)\gamma_{\ell+1}) \Gamma_{X_{2k}^{\ell+1}}^{X_{2k+1}^{\ell+1}} b((2k+1)\gamma_{\ell+1}, X_{2k}^{\ell+1}) \\ & \quad (\bfB_{t} - \bfB_{(2k+1) \gamma_{\ell+1}}) E_{2k+1}^{\ell+1}]  .     \label{eq:hat_X_redef}
\end{align}
In what follows, we provide an upper-bound for
$d(\bar{\bfX}_{t}^{\ell+1}, \bfX_{t}^{\ell+1})$. First, we have that
\begin{equation}
  d(\bar{\bfX}_{t}^{\ell+1}, \bfX_{t}^{\ell+1}) \leq d(\bar{\bfX}_{t}^{\ell+1}, \hat{\bfX}_{t}^{\ell+1}) + d(\hat{\bfX}_{t}^{\ell+1}, \bfX_{t}^{\ell+1})  . 
\end{equation}
We recall that
\begin{align}
    \bar{\bfX}_{t}^{\ell+1} &= \exp_{X_{2k}^{\ell+1}}[\gamma_{\ell+1} b(2 k \gamma_{\ell + 1}, X_{2k}^{\ell+1}) \\ & \quad + (t - (2k+1) \gamma_{\ell+1})  b((2 k+1) \gamma_{\ell + 1}, X_{2k}^{\ell+1})  + (\bfB_{t} - \bfB_{k \gamma_\ell}) E_{2k}^{\ell+1} ]  .  \label{eq:bar_X_redef}
\end{align}
  Denote $a_k, b_k$ such that
  \begin{align}
    &a_k =  b(2 k \gamma_{\ell + 1}, X_{2k}^{\ell+1}) + (\bfB_{(2k+1)\gamma_{\ell+1}} - \bfB_{k \gamma_\ell}) E_{2k}^{\ell+1}  , \\
    &b_k = (t - (2k+1)\gamma_{\ell+1}) b((2k+1) \gamma_{\ell + 1}, X_{2k}^{\ell+1}) + (\bfB_{t} - \bfB_{(2k+1) \gamma_{\ell+1}}) E_{2k}^{\ell+1}  .
  \end{align}
Using \eqref{eq:coupling}, \eqref{eq:hat_X_redef} and \eqref{eq:bar_X_redef} we have that
\begin{align}
  &X_{2k+1}^{\ell+1} = \exp_{X_{2k}^{\ell+1}}[a_k]  , \qquad \textstyle{\hat{\bfX}_{t}^{\ell+1} = \exp_{X_{2k+1}^{\ell+1}}[\Gamma_{X_{2k}^{\ell+1}}^{X_{2k+1}^{\ell+1}}b_k]  , \qquad \bar{\bfX}_{t}^{\ell+1} = \exp_{X_{2k}^{\ell+1}}[a_k + b_k]  . }
\end{align}
Using this result and \cite[Lemma 3]{sun2019escaping}, there exists $C_2 \geq 0$
(not dependent on $k$ or $\ell$) such that
\begin{equation}  
  d(\hat{\bfX}_{t}^{\ell+1}, \bar{\bfX}_{t}^{\ell+1}) \leq C_2 (\normLigne{a_k} + \normLigne{b_k})^3  . 
\end{equation}
Using this result and that for any $t \in \ccint{0,\gamma}$ and $x \in \M$,
$\normLigne{b(t,x)} \leq \Ktt$ we get that there exists $C_3 \geq 0$ (not
dependent on $k$ or $\ell$) such that
\begin{equation}
  \label{eq:bound_distance_uno}
  d(\hat{\bfX}_{t}^{\ell+1}, \bar{\bfX}_{t}^{\ell+1})^2 \leq C_3 (\gamma_{\ell+1}^6 + \normLigne{\bfB_{t} - \bfB_{(2k+1) \gamma_{\ell+1}}}^6 + \normLigne{\bfB_{(2k+1) \gamma_{\ell}} - \bfB_{(k+1) \gamma_{\ell}}}^6)  . 
\end{equation}
Finally, we recall that
  \begin{align}
    \hat{\bfX}_{t}^{\ell+1} &= \exp_{X_{2k+1}^{\ell+1}}[(t - (2k+1)\gamma_{\ell+1}) \Gamma_{X_{2k}^{\ell+1}}^{X_{2k+1}^{\ell+1}} b((2k+1)\gamma_{\ell+1}, X_{2k}^{\ell+1}) \\ \quad &+ (\bfB_{t} - \bfB_{(2k+1) \gamma_{\ell+1}}) E_{2k+1}^{\ell+1}]  , \\
    \bfX_{t}^{\ell+1} &= \exp_{X_{2k+1}^{\ell+1}}[(t - (2k+1)\gamma_{\ell+1}) b((2k+1)\gamma_{\ell+1}, X_{2k+1}^{\ell+1}) + (\bfB_{t} - \bfB_{(2k+1) \gamma_{\ell+1}}) E_{2k+1}^{\ell+1}]  . 
\end{align}
Let us define
\begin{align}
  &\tau_k = \normLigne{c_k} + \normLigne{d_k}  , \\
  &c_k = c_k^1 + c_k^2 , \qquad \qquad \qquad \qquad   \ d_k = d_k^1 + d_k^2  ,\\
  &c_k^1 = (t - (2k+1)\gamma_{\ell+1}) b((2k+1) \gamma_{\ell+1}, X_{2k+1}^{\ell+1})   , \\
  &d_k^1 = (t - (2k+1)\gamma_{\ell+1}) \Gamma_{X_{2k}^{\ell+1}}^{X_{2k+1}^{\ell+1}}b((2k+1) \gamma_{\ell+1}, X_{2k}^{\ell+1})  , \\
  &c_k^2 = d_k^2 = (\bfB_{t} - \bfB_{(2k+1) \gamma_{\ell+1}}) E_{2k+1}^{\ell+1}  . \label{eq:def_tau}
\end{align}
Using \Cref{lemma:bound_exp}, we get that
\begin{align}
  \label{eq:upper_bound_distance_cheng_inter}
  d(\bfX_{t}^{\ell+1}, \hat{\bfX}_{t}^{\ell+1})^2 \leq C \exp[4 \tau_k] \normLigne{c_k - d_k}^2 \leq C \Ltt_2^2 \gamma_{\ell+1}^2  \exp[4 \tau_k] d(X_{2k+1}^{\ell+1}, X_{2k}^{\ell+1})^2  .
\end{align}
In addition, using \Cref{lemma:bound_exp}, we get that
\begin{equation}
  d(X_{2k+1}^{\ell+1}, X_{2k}^{\ell+1})^2 \leq \exp[4 \normLigne{e_k}] \normLigne{e_k}  ,
\end{equation}
with
$e_k = \gamma_{\ell+1}b(k\gamma_\ell, X_{2k}^{\ell+1}) +
(\bfB_{(2k+1)\gamma_{\ell+1}} - \bfB_{k\gamma_{\ell}})
E_{2k}^{\ell+1}$. Combining this result and
\eqref{eq:upper_bound_distance_cheng_inter}, we get that
\begin{equation}
  \label{eq:bound_distance_duo}
  d(\bfX_{t}^{\ell+1}, \hat{\bfX}_{t}^{\ell+1})^2  \leq C_3 \gamma_{\ell+1}^2  (\gamma_{\ell+1}^2 + \normLigne{\bfB_{(2k+1)\gamma_{\ell+1}} - \bfB_{k\gamma_{\ell}}}^2)\exp[4 \tau_k + \normLigne{e_k}]   . 
\end{equation}
Combining \eqref{eq:bound_distance_uno} and \eqref{eq:bound_distance_duo}, there
exists $C_5$ such that
\begin{align}
  d(\bar{\bfX}_{t}^{\ell+1}, \bfX_{t}^{\ell+1})^2 &\leq C_5 \gamma_{\ell+1}^2  (\gamma_{\ell+1}^2 + \normLigne{\bfB_{(2k+1)\gamma_{\ell+1}} - \bfB_{k\gamma_{\ell}}}^2)\exp[4 \tau_k + \normLigne{e_k}] \\
  & \qquad +  C_5 (\gamma_{\ell+1}^6 + \normLigne{\bfB_{t} - \bfB_{(2k+1) \gamma_{\ell+1}}}^6 + \normLigne{\bfB_{(2k+1) \gamma_{\ell}} - \bfB_{(k+1) \gamma_{\ell}}}^6)  .  \label{eq:case_complete_2}
\end{align}
\end{enumerate}
In what follows, we denote
\begin{align}
  \alpha_k &= C_1\{ (\kappa_k^+)^2 \exp[4 \kappa_k] + \gamma_\ell^2 \exp[4\kappa_k^+] + 2^{-\ell} \}  . \\
  \beta_k &=  C_1  \gamma_\ell^3 + C_5 (1 + 2^\ell) \gamma_{\ell+1}^2  (\gamma_{\ell+1}^2 + \normLigne{\bfB_{(2k+1)\gamma_{\ell+1}} - \bfB_{k\gamma_{\ell}}}^2)\exp[4 \tau_k^+ + \normLigne{e_k}] \\
  & \qquad \textstyle{ +  C_5 (1 + 2^\ell) \ (\gamma_{\ell+1}^6 + \sup_{t \in \ccint{k\gamma_\ell, (k+1) \gamma_\ell}} \{\normLigne{\bfB_{t} - \bfB_{(2k+1) \gamma_{\ell+1}}}^6\}} \\ & \qquad \textstyle{ + \normLigne{\bfB_{(2k+1) \gamma_{\ell}} - \bfB_{(k+1) \gamma_{\ell}}}^6)  , }
\end{align}
with
$\tau_k^+ = \sup \ensembleLigne{\normLigne{c_k} + \normLigne{d_k}}{t \in
  \ccint{k \gamma_\ell, (k+1) \gamma_\ell}}$, see \eqref{eq:def_tau}. Therefore,
using \eqref{eq:inter_boundi_bound}, \eqref{eq:case_complete} and
\eqref{eq:case_complete_2}, we get that for any $k \in \{0, \dots, 2^\ell -1\}$
\begin{equation}
  U_{k+1} \leq (1 + \alpha_k) U_k + \beta_k  . 
\end{equation}
Let $\{R_k\}_{k=-1}^{2^\ell}$ such that $R_{-1}=0$ and for any $k \in \{0, \dots, 2^{\ell}-1\}$
\begin{equation}
  R_{k+1} = (1 + \alpha_k) R_k + \beta_k  . 
\end{equation}
Then, for any $k \in \{0, \dots, 2^{\ell}-1\}$, we have that $R_{2^\ell -1} \geq R_k \ge U_k$.
Therefore
\begin{equation}
  \label{eq:bound_sup}
  \expeLigne{R_{2^\ell}} \geq \expeLigne{\sup \ensembleLigne{U_k}{k \in \{0, \dots, 2^\ell\}}} \geq \expeLigne{\sup \ensembleLigne{d(\bfX_t^\ell, \bfX_t^{\ell+1})^2}{t \in \ccint{0,\gamma}}}  .
\end{equation}
In addition, using that for any $k \in \{0, \dots, 2^\ell - 1\}$,
$\CPELigne{\alpha_k}{\mcf_k} = \bar{\alpha}_k$ and
$\CPELigne{\beta_k}{\mcf_k} = \bar{\beta}_k$ are constant, where
$\mcf_k = \sigma(\ensembleLigne{\bfB_t}{t \in \ccint{0, k
    \gamma_\ell}})$. Therefore, we get that for any  $k \in \{0, \dots, 2^\ell - 1\}$
\begin{equation}
  \expeLigne{R_{k+1}} = (1 + \bar{\alpha}_k) \expeLigne{R_k} + \bar{\beta}_k  . 
\end{equation}
Therefore, using the discrete Gr\"onwall lemma we get that for any $k \in \{0, \dots, 2^{\ell}-1\}$
\begin{equation}
  \textstyle{
  \expeLigne{R_{2^\ell}} \leq \bar{\beta}_{2^{\ell}-1} + \exp[\sum_{n=0}^{2^\ell - 1} \bar{\alpha}_n] \sum_{j=0}^{2^\ell - 1} \bar{\beta}_j \bar{\alpha}_j   . }
\end{equation}
In addition, there exists $C_8 \geq 0$ such that for any
$k \in \{0, \dots, 2^\ell\}$, $\bar{\alpha}_k \leq C_8 2^{-\ell}$ and
$\bar{\beta}_k \leq C_8 \gamma^3 2^{-2\ell}$. Hence, there exists $C_9 \geq 0$
such that
\begin{equation}
  \expeLigne{R_{2^\ell}} \leq C_9 \gamma^3 2^{-2\ell}  ,
\end{equation}
which concludes the proof upon using \eqref{eq:bound_sup}.
  


\end{proof}
\begin{proposition}{}{limit_seq}
  Assume \rref{assum:manifold}. Then, there exists
  $(\bfX_t)_{t \in \ccint{0,\gamma}}$ such that
  $\lim_{\ell \to +\infty} \sup \ensembleLigne{d(\bfX_t^\ell, \bfX_t)}{t \in
    \ccint{0, \gamma}} =0$ and $(\bfX_t)_{t \in \ccint{0,\gamma}}$ is a weak
  solution to $\rmd \bfX_t = b(t, \bfX_t) \rmd t + \rmd \bfB_t^\M$.
\end{proposition}

\begin{proof}{}{}
  The proof is a straightforward application of \Cref{prop:bound_square} and
  \cite[A.1 (Step 2 and Step 3), A.2]{cheng2022theorymanifold}.
\end{proof}

\begin{proposition}{}{borne_one_step}
  Assume \rref{assum:manifold}. Then, there exists $C \geq 0$ such that $\expe{d(X_1^0, \bfX_\gamma)^2 \leq C \gamma^{3/2}}$.
\end{proposition}

\begin{proof}{}{}
  Using \Cref{prop:bound_square}, there exists $C \geq 0$ such that for any $\ell \in \nset$
    \begin{equation}
    \textstyle{ \expeLigne{\sup_{t \in \ccint{0, \gamma}} d(\bfX_t^\ell, \bfX_t^{\ell+1})} \leq C \gamma^{3/2} 2^{-\ell}  . }
  \end{equation}
Therefore, combining this result and \Cref{prop:limit_seq} we get that for any $\ell \in \nset$
\begin{equation}
  \textstyle{ \expeLigne{\sup_{t \in \ccint{0, \gamma}} d(\bfX_t^\ell, \bfX_t)} \leq 2 C \gamma^{3/2}  , }
\end{equation}
which concludes the proof.
\end{proof}

Finally, we consider the two following processes $(X_k^1, X_k^2)_{k \in \nset}$
such that for any $k \in \nset$ and $i \in \{1, 2\}$
\begin{equation}
  X_{k+1}^i = \exp_{X_k^i}[\gamma b(k \gamma, X_k^i) + \sqrt{\gamma} E_k^i Z_k ]  ,
\end{equation}
where $\{Z_k\}_{k \in \nset}$ is a family of independent Gaussian random
variables with zero mean and identity covariance matrix, and for any
$k \in \nset$, $E_k^1$ is a frame for $\mathrm{T}_{X_k^1} \M$ and
$E_k^2 = \Gamma_{X_k^1}^{X_k^2} E_k^1$.

\begin{proposition}{}{explosion}
  Assume \rref{assum:manifold}. Then, there exists $C \geq 0$ such that for any $k \in \nset$
  \begin{equation}
    \expe{d(X_k^1, X_k^2)} \leq \exp[C k \gamma] \expe{d(X_0^1, X_0^2)}  . 
  \end{equation}
\end{proposition}

\begin{proof}{}{}
  Let $k \in \nset$. Using \Cref{lemma:bound_exp}, there exists $D \geq 0$ such that 
\begin{align}
  d(X_{k+1}^1, X_{k+1}^2)^2 &\leq (1 +D \kappa_k^2 \exp[4 \kappa_k]) d(X_k^1,X_k^2)^2 \\
  & \qquad + D \exp[4 \kappa_k] \normLigne{\Gamma_{X_{k}^{2}}^{X_k^1} v_k - u_k}^2 + 2 \langle w'(0), \Gamma_{X_{k}^{2}}^{X_k^1} v_k - u_k \rangle  ,
\end{align}
with $w: \ \ccint{0,1} \to \M$ a minimizing geodesic between $X_k^1$ and
$X_{k}^2$
\begin{align}
  &\kappa_k = \normLigne{u_k} + \normLigne{v_k}  , \\
  &u_k^1 = \gamma b(k \gamma, X_k^1)   , \\
  &v_k^1 = \gamma b(k \gamma, X_{k}^{2})  , \\
  &u_k^2 = \sqrt{\gamma} Z_k E_{k}^{1}  , \qquad \quad \quad \  v_k^2 = \sqrt{\gamma} Z_k E_k^2  , \\
    &u_k = u_k^1 + u_k^2 , \qquad \qquad \quad \ v_k = v_k^1 + v_k^2  .
\end{align}
We have that $\Gamma_{X_k^2}^{X_k^1} v_k^2 = v_k$ and  
\begin{equation}
  \normLigne{\Gamma_{X_k^2}^{X_k^1} v_k^1 - u_k^1} \leq \Ltt_1 \gamma d(X_k^1, X_k^2)  . 
\end{equation}
In addition, $\normLigne{w'(0)} \leq d(X_k^1, X_k^2)$. Therefore, we get that
\begin{equation}
  d(X_{k+1}^1, X_{k+1}^2)^2 \leq (1 +D \kappa_k^2 \exp[4 \kappa_k] + D \gamma^2 \exp[4 \kappa_k] +2 \gamma) d(X_k^1,X_k^2)^2   . 
\end{equation}
Hence, using that for any $t \geq 0$, $\sqrt{1+t} \leq 1 + t/2$, we have
\begin{equation}
  d(X_{k+1}^1, X_{k+1}^2) \leq (1 +D \kappa_k^2 \exp[4 \kappa_k] + D \gamma^2 \exp[4 \kappa_k] +2 \gamma) d(X_k^1,X_k^2)   . 
\end{equation}
Therefore, we get that there exists $C \geq 0$ such that 
\begin{equation}
  \expeLigne{d(X_{k+1}^1, X_{k+1}^2)} \leq (1 + C \gamma) \expeLigne{d(X_k^1,X_k^2)}  ,
\end{equation}
which concludes the proof.
\end{proof}



\section{Proof of proposition \ref{prop:implicit_der}}
\label{sec:implicit-losses}

The regularity conditions on $p_{t|s}(x_t|x_s)s_t(x_t)$ are
\begin{itemize}
    \item $p_{t|s}(x_t|x_s)s_t(x_t)$ is a vector field in $C_1$ $\forall x_s$.
    \item $|p_{t|s}(x_t|x_s)s_t(x_t)| \in L_1$ $\forall x_s$.
    \item $\dive(p_{t|s}(x_t|x_s)s_t(x_t)) \in L_1$ $\forall x_s$.
\end{itemize}
These conditions are not difficult to show. We can manually control $s_t$ by our choice of score network, and $p_{t|s}(x_t|x_s)$ is controlled by choice of noising process. Under these conditions we can prove the statement.

\begin{proof}{}{}
  Let $t \in \ocint{0,T}$ and $s_t \in \rmc^\infty(\M)$. Using a
  divergence theorem for non-compact manifolds \cite[see][p.2]{gaffney1954}, we have
  \begin{align}
    \ell_{t|s}(s_t) &{= \int_{\M \times \M} \normLigne{\nabla \log p_{t|s}(x_t|x_s)}^2 \rmd \Pbb_{s,t}(x_s,x_t) + \int_\M \normLigne{s_t(x_t)}^2 \rmd \Pbb_{t}(x_t)} \\
    & \qquad \qquad {- 2 \int_{\M \times \M} \langle \nabla \log p_{t|s}(x_t|x_s), s_t(x_t) \rangle_\M \rmd \Pbb_{s,t}(x_s,x_t)} 
\end{align}
Looking at the last term
\begin{align}
    &\int_{\M \times \M} \langle \nabla \log p_{t|s}(x_t|x_s), s_t(x_t) \rangle_\M \rmd \Pbb_{s,t}(x_s,x_t)\\ &= \int_{\M \times \M}  \langle \nabla \log p_{t|s}(x_t|x_s), s_t(x_t) \rangle_\M p_{t|s}(x_t|x_s)p_s(x_s) \rmd (\piinv \otimes \piinv) (x_s, x_t) \\
    &=\int_{\M } \left\{\int_\M  \langle \nabla p_{t|s}(x_t|x_s), s_t(x_t) \rangle_\M \rmd \piinv(x_t)\right\} p_s(x_s) \rmd \piinv(x_s) \\
    &= -\int_{\M } \left\{\int_\M   \dive(s_t)(x_t)p_{t|s}(x_t|x_s)  \rmd \piinv(x_t)\right\} p_s(x_s) \rmd \piinv(x_s) \quad \text{by the divergence theorem} \\
    &=- \int_{\M \times \M} \dive(s_t)(x_t)\rmd \Pbb_{s,t}(x_s,x_t) =- \int_{\M \times \M} \dive(s_t)(x_t)\rmd \Pbb_{t}(x_t)
\end{align}
Therefore
\begin{align}
    \ell_{t|s}(s_t) &={\int_{\M \times \M} \normLigne{\nabla \log p_{t|s}(x_t|x_s)}^2 \rmd \Pbb_{s,t}(x_s,x_t) + \int_\M \normLigne{s_t(x_t)}^2 \rmd \Pbb_{t}(x_t)} \\
                & \qquad \qquad { +2 \int_{\M} \dive(s_t)(x_t)\rmd \Pbb_{t}(x_t)}  ,
  \end{align}
  which concludes the proof.
\end{proof}


\section{Comparison with Moser flows}
\label{sec:comp-with-moser}
In this section, we compare ourselves with \textcite{rozen2021moser} in greater
details. \textcite{rozen2021moser} also aims at interpolating between a reference
distribution $\piinv$ and a target distribution $p_0$. We assume that we have
access to the density $\piinv$ and that we know how to sample form $\piinv$
(which is often the case if $\piinv$ is the uniform distribution on
$\M$).

We then consider the following interpolation
$\hat{p}_t = (1-t) \hat{p}_0 + t \hat{p}_1$, with $\hat{p}_0 = \piinv$ and
$\hat{p}_1 = p_0$. Let $(\bfX_t)_{t \in \ccint{0,1}}$ be given by
$\bfX_0 \sim \hat{p}_0$ and $\rmd \bfX_t = \mathbf{v}_t(\bfX_t) \rmd t$ where
for any $t \in \ccint{0,1}$,
$\mathbf{v}_t = \mathbf{u} / ((1-t) \hat{p}_0 + \hat{p}_1)$, with
$\dive(\mathbf{u}) = \hat{p}_0 - \hat{p}_1$. Using the Fokker-Planck equation,
we have that for any $t \in \ccint{0,1}$, $\bfX_t \sim \hat{p}_t$. In
\textcite{rozen2021moser}, $\mathbf{u}$ is replaced by a parametric version
$\mathbf{u}_\theta$ and the authors optimize the loss
\begin{equation}
  \textstyle{
    \ell(\theta) = \expeLigne{(\hat{p}_0 - \mathrm{div}(\mathbf{u}_\theta))^{+, \vareps}(\bfX_1)} + \lambda \int_\M (\hat{p}_0 - \mathrm{div}(\mathbf{u}_\theta))^{-, \vareps}(x) \rmd x ,
    }
  \end{equation}
  with $\lambda, \vareps >0$ and for any $f: \ \M \to \rset$,
  $f^{+, \vareps} = \max(f, \vareps)$ and
  $f^{-, \vareps} = \vareps - \min(f, \vareps)$.  Given $\mathbf{u}_\theta$, we
  then consider $(\bfX_t^\theta)_{t \in \ccint{0,1}}$ such that
  $\rmd \bfX_t^\theta = \mathbf{v}_{t}^\theta(\bfX_t^\theta) \rmd t$, where for
  any $t \in \ccint{0,1}$,
  $\mathbf{v}_{t}^\theta = \mathbf{u}_\theta / (\hat{p}_0 + t
  \dive(\mathbf{u}_\theta))$. Note that $\mathbf{u}^\theta$ also enables density
  estimation using that $\hat{p}_1 = \hat{p}_0 -
  \dive(\mathbf{u}^\theta)$. Density estimation is not directly accessible using
  RSGM, however in \Cref{sec:dens-estim-with} we propose a way to perform such
  an estimation using Fisher score in a manner akin to \textcite{choi2021density}.

  Let $\hat{p}_0 = \piinv$ to be the uniform distribution on $\M$. As RSGM,
  Moser flow defines a continuous time interpolation between $p_0$ and
  $\piinv$. One major difference between the two approaches is that Moser flows
  perform the interpolation in \emph{density space}, i.e.
  $\hat{p}_t = (1-t) \hat{p}_0 + t \hat{p}_1$ for any $t \in \ccint{0,1}$,
  whereas RSGM performs the interpolation in \emph{sample space}, i.e.
  $p_t = \int_{\M} p_0(y) p_{t|0}(y,x) \rmd \piinv(y)$. Interpolation in the
  \emph{density space} results in spontaneous creation of density, whereas
  interpolation in \emph{sample space} corresponds to a displacement of the
  density, see \Cref{fig:interpolation_pdf,fig:interpolation_hist}. In that respect, Moser flows can
  be seen as \emph{vertical displacement} whereas RSGM corresponds to
  \emph{horizontal displacement}, see \textcite{santambrogio2017euclidean}.
  The drawback with the `spontaneous creation of density' of Moser flows, is that when solving trajectories in \emph{sample space}---for sampling or likelihood evaluation purposes---the Stein score's amplitude can get extremely high in settings where the reference and target distributions have little overlap as shown on \cref{fig:interpolation_score_norm}.

\begin{figure}[t]
  \centering
    \begin{subfigure}{\textwidth}
        \includegraphics[width=\textwidth]{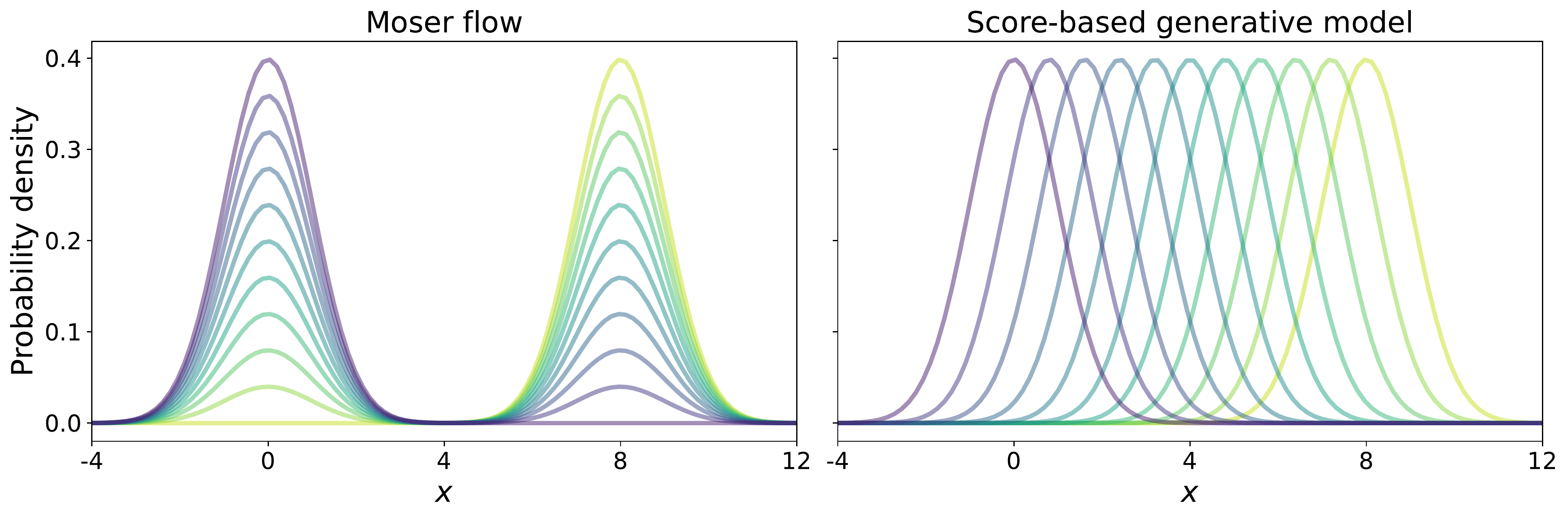}
        \caption{Interpolated density between the reference $\piinv = \mathrm{N}(0, 1)$ and target $p_0 = \mathrm{N}(8, 1)$ distributions.} \label{fig:interpolation_pdf}
    \end{subfigure}
    \vfill
    \begin{subfigure}{\textwidth}
        \includegraphics[width=\textwidth]{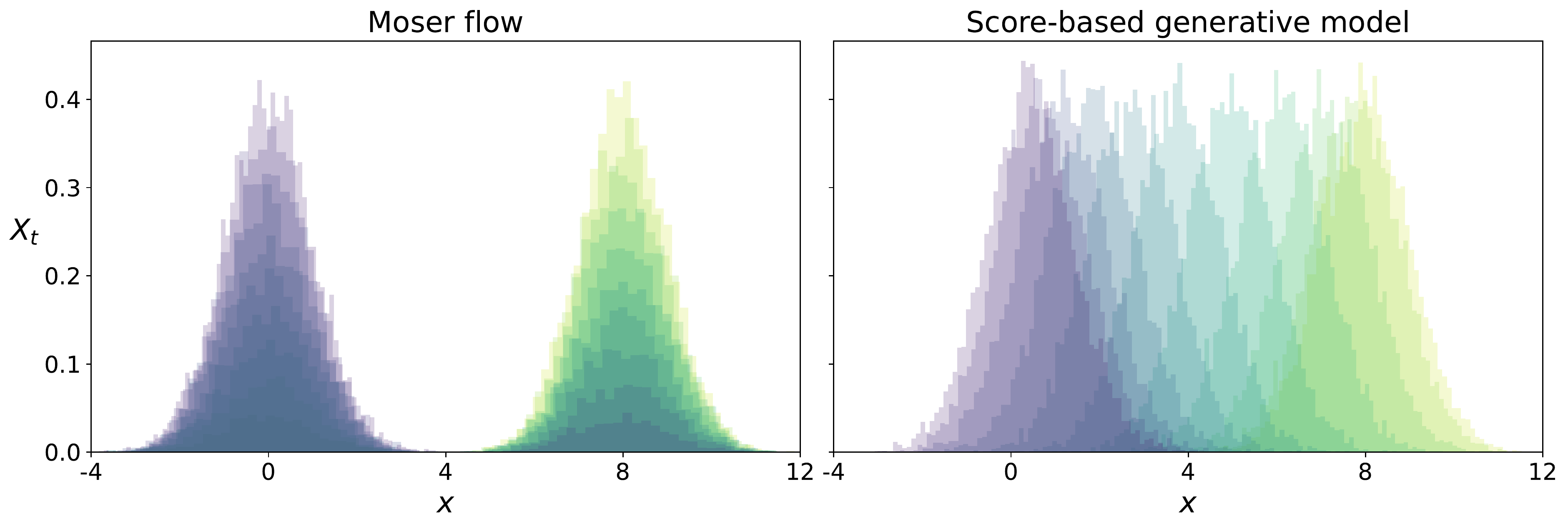}
        \caption{Interpolated histograms between the reference $\piinv = \mathrm{N}(0, 1)$ and target $p_0 = \mathrm{N}(8, 1)$ distributions.}
        \label{fig:interpolation_hist}
    \end{subfigure}
    \vfill
    \begin{subfigure}{\textwidth}
        \includegraphics[width=\textwidth]{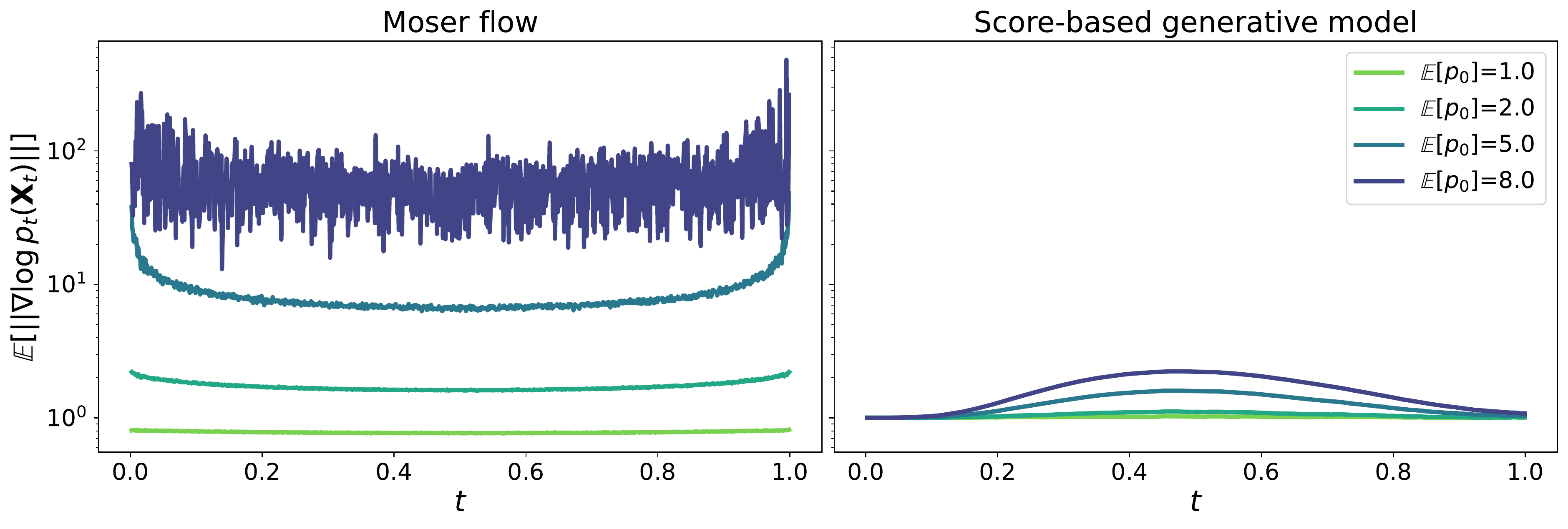}
        \caption{Expected norm of the Stein score along trajectories interpolating between reference and target $p_0 = \mathrm{N}(a, 1)$ distributions for different target mean.}
        \label{fig:interpolation_score_norm}
    \end{subfigure}
    \caption{
    The reference distribution is $\piinv = \mathrm{N}(0, 1)$. 
    }
    \label{fig:comp_moser_diffusion}
\end{figure}


\section{Density estimation with Fisher score}
\label{sec:dens-estim-with}

In this section, we show how we can adapt ideas from \textcite{choi2021density} for
density estimation on $\M$ using the Fisher score. The main idea of using Fisher
score is to leverage the following decomposition for any $x \in \M$
\begin{equation}
  \textstyle{
    \log p_0(x) = \log p_T(x) - \int_0^T \partial_t \log p_t(x) \rmd t .
    }
\end{equation}
Assume that an approximation $\hat{\mathbf{s}}_\theta$ of $\partial_t \log p_t$
(the Fisher score) is available then we have that for any $x \in \M$
\begin{equation}
  \textstyle{
    \log p_0(x) \approx \log \piinv (x) - \int_0^T \hat{\mathbf{s}}_\theta(x) \rmd t .
    }
  \end{equation}
  Before turning to our main result, we state the following lemma.

   \begin{lemma}{}{bound_unif_lemma}
   Assume \rref{assum:manifold}. Then, there exists $C, T_0 \geq 0$ such that for any
   $x \in \M$ and $T \geq T_0$,
   $\absLigne{p_T(x) - 1} \leq C \exp[-\lambda_1 T / 2]$, where $\lambda_1$ is the
   first non-negative eigenvalue of $-\Delta_\M$ in
   $\mathrm{L}^2(p_{\textup{ref}})$.
 \end{lemma}

 \begin{proof}{}{}
   First, using \Cref{prop:brownian_conv_repeat}, there exists $C_0 \geq 0$ such
   that for any $T \geq 1/2$ we have
   \begin{equation}
     \textstyle{\int_\M \absLigne{p_T(x) - 1} \rmd \piinv(x)}\leq C_0 \rme^{-\lambda_1 T} .
   \end{equation}
   Using \cite[Corollary 5.5]{grigor1999estimates}, \cite[Theorem
   1.2]{hsu1999estimates}and the fact that $\M$ is compact, there exists
   $C_1, \beta \geq 0$ such that for any $T \geq 1/2$ and $x_0, x_T \in \M$
   \begin{equation}
\label{eq:bound_grad}
     \normLigne{\nabla p_{T|0}(x_T|x_0) } \leq C_1(1 + T^\beta) .
   \end{equation}
   In addition, using \cite[Proposition 14]{croke1980some} we have that there
   exists $C_2, r_0 > 0$ such that for any $x_0 \in \M$ and
   $r \in \ooint{0, r_0}$
\begin{equation}
\label{eq:lower_bound_volume}
  \textstyle{\int_{\cball{x_0}{r}} \rmd \piinv(x) \geq C_2 r^d . }
\end{equation}
Assume that that
$\textstyle{\int_\M \absLigne{p_T(x) - 1} \rmd \piinv(x)} \leq \vareps$ and that
there exists $x_0 \in \M$ such that $\abs{p_T(x) - 1} > \kappa \vareps$ with
$\kappa > 0$ and let $T \geq T_0$ with
$T_0 = (\kappa \vareps/(2C_1))^{1/\beta}$. Then, using \eqref{eq:bound_grad} and
\eqref{eq:lower_bound_volume}, we have for any $r \in \ooint{0,r_0}$
\begin{equation}
  \textstyle{\vareps \geq \int_{\cball{0}{r}} \abs{p_T(x) - 1} \geq C_2r^{d}(\kappa \vareps - C_1(1 + T^\beta) r).}
\end{equation}
Since $\kappa \vareps / (2C_1(1+T^\beta)) \in \ooint{0,r_0}$ we have
\begin{equation}
  \textstyle{\vareps  \geq C_2(\kappa \vareps)^{d+1}/(4C_1(1+T^\beta)).}
\end{equation}
Therefore, we get that
\begin{equation}
  \label{eq:8}
  \textstyle{\vareps  \geq C_2(\kappa \vareps)^{d+1}/(4C_1(1+T^\beta)).}
\end{equation}
Therefore, we get that
$\kappa \leq (4 C_1 (1 + T^\beta)/ C_2)^{1/(d+1)}
\vareps^{-1/(d+1)}$. Therefore, we have that for any $x \in \M$
\begin{equation}
  \label{eq:bound_unfi1}
  \absLigne{p_T(x) - 1} \leq (8 C_1 (1 + T^\beta)/ C_2)^{1/(d+1)} \vareps^{1 - 1/(d+1)} .
\end{equation}
Let $T_0 \geq 0$ such that for any $T \geq T_0$ we have
\begin{equation}
(8 C_1 (1 + T^\beta)/ C_2)^{1/(d+1)} C_0^{1 - 1/(d+1)} \rme^{-(1 - 1/(d+1)) \lambda_1 T} \leq 2^{1-\beta} C_1 .
\end{equation}
Combining this result and \eqref{eq:bound_unfi1}, we get that for any $x \in \M$ and $T \geq 0$
\begin{equation}
  \label{eq:bound_unfi2}
  \absLigne{p_T(x) - 1} \leq (8 C_1 (1 + T^\beta)/ C_2)^{1/(d+1)} C_0^{1 - 1/(d+1)} \rme^{-(1 - 1/(d+1)) \lambda_1 T} ,
\end{equation}
which concludes the proof.
\end{proof}

  The following proposition quantifies this approximation.

  \begin{proposition}{}{}
    Assume \rref{assum:manifold} and that
    $p_0 \in \rmc^\infty(\M, \ooint{0,+\infty})$. Let $x_0 \in \M$ and assume
    that for any $t \in \ccint{0,T}$,
    $\absLigne{\hat{\mathbf{s}}_\theta(t, x_0) - \partial_t \log p_t(x_0)} \leq \Mtt$
    with $\Mtt \geq 0$. Then, there exists $C, T_0 \geq 0$ such that for any $T \geq 0$
    \begin{equation}
      \textstyle{
        \absLigne{\log p_0(x_0)  - \int_0^T \hat{\mathbf{s}}_\theta(t, x_0) \rmd t } \leq C \exp[-\lambda_1 T/2 ] + \Mtt T .,
        }
      \end{equation}
      where $\lambda_1$ is the first non-negative eigenvalue of $-\Delta_\M$ in
      $\mathrm{L}^2(p_{\textup{ref}})$.
  \end{proposition}

  \begin{proof}{}{}
    First using, \Cref{lemma:bound_unif_lemma}, there exists $C_0, T_0^{(a)} \geq 0$ such that for any $T \geq T_0^{(a)}$
    \begin{equation}
      \absLigne{p_T(x_0) - 1} \leq C_0 \exp[-\lambda_1 T / 2].
    \end{equation}
    Let $T_0^{(b)} = \abs{\log(C_0)} / \lambda_1$. Using that for any
    $s \in \coint{1/2, +\infty}$ we have that
    $\absLigne{\log(1 +s)} \leq 2 \log(2) \absLigne{s}$ we get that for any
    $T \geq \max(T_0^{(a)}, T_0^{(b)})$
    \begin{equation}
      \absLigne{\log p_T(x_0)} \leq 2 \log(2) C_0 \exp[-\lambda_1 T / 2],
    \end{equation}
    which concludes the proof.
  \end{proof}

  In practice, we do not have access to $\partial_t \log p_t$. However, following
  \cite[Proposition 2]{choi2021density}, we have the following property.
  
  \begin{proposition}{}{loss-function_fisher}
    Let $\hat{\mathbf{s}}$ such that for any $t \in \ccint{0,T}$ and $x \in \M$,
    $\hat{\mathbf{s}}(t,x) = \partial_t \log p_t(x)$. Then, we have that
    $\hat{\mathbf{s}} = \argmin \ensembleLigne{L(\mathbf{s})}{\mathbf{s} \in
      \rmc^\infty(\ccint{0,T} \times \M, \rset)}$, where for any $\mathbf{s} \in
      \rmc^\infty(\ccint{0,T} \times \M, \rset)$ we have 
    \begin{align}
      L(\mathbf{s}) &= \textstyle{(1/2) \expeLigne{\int_0^T \lambda(t) \mathbf{s}(t, \bfX_t) \rmd t } + \expeLigne{\int_0^T \lambda(t)  \partial_t \mathbf{s}(t, \bfX_t) \rmd t}} \\
      & \qquad \textstyle{ + \expeLigne{\int_0^T \partial_t \lambda(t)  \partial_t \mathbf{s}(t, \bfX_t) \rmd t} + \expeLigne{\lambda(0)\mathbf{s}(0, \bfX_0)} - \expeLigne{\lambda(T)\mathbf{s}(T, \bfX_T)} ,}
    \end{align}
    where $\lambda \in \rmc^\infty(\ccint{0,T}, \rset)$ is a weighting function.
  \end{proposition}

  \begin{proof}{}{}
    For any $t \in \ccint{0,T}$ and $x_t \in \M$ we have
    \begin{equation}
      \textstyle{
        \hat{\mathbf{s}}(x_t) = \int_\M \partial_t \log p_{t|0}(x_t|x_0) p_{0|t}(x_0|x_t) \rmd x_0 .
        }
    \end{equation}
    Hence, since $\M$ is compact and
    $\hat{\mathbf{s}} \in \rmc^\infty(\ccint{0,T} \times \M, \rset)$, we have that
    $\hat{\mathbf{s}} = \argmin \ensembleLigne{L_0(\mathbf{s})}{\mathbf{s} \in
      \rmc^\infty(\ccint{0,T} \times \M, \rset)}$ where for any
    $\mathbf{s} \in \rmc^\infty(\ccint{0,T} \times \M, \rset)$ we have
    \begin{align}
      L_0(\mathbf{s}) &= \textstyle{\int_0^T \lambda(t) \int_{\M \times \M} (\mathbf{s}(t, x_t) - \partial_t \log p_{t|0}(x_t|x_0))^2 \rmd p_{0,t}(x_0, x_t) \rmd t} \\
                      &= \textstyle{\int_0^T \lambda(t) \int_{\M} \mathbf{s}(t, x_t)^2 \rmd p_t(x_t) \rmd t } \\ & \qquad \qquad\textstyle{ - 2 \int_0^T \lambda(t) \int_{\M \times \M} \mathbf{s}(t, x_t) \partial_t \log p_{t|0}(x_0, x_t) \rmd p_{0,t}(x_0,x_t) \rmd t } \\
      & \qquad \qquad \textstyle{+ \int_0^T \lambda(t) \int_{\M}  \rmd p_t(x_t) \rmd t} \label{eq:loss_L0}
    \end{align}
In addition, we have that
\begin{align}
  &\textstyle{\int_0^T \lambda(t) \int_{\M\times\M} \mathbf{s}(t, x_t) \partial_t \log p_{t|0}(x_t|x_0) \rmd p_{0,t}(x_0, x_t) \rmd t} \\
  &\qquad \qquad = \textstyle{\int_0^T \int_{\M\times\M}  \lambda(t)  \mathbf{s}(t, x_t) \partial_t p_{t|0}(x_t)\rmd p_0(x_0) \rmd \piinv(x_t) \rmd t} .
\end{align}
By integration by parts we get 
\begin{align}
  &\textstyle{\int_0^T \int_{\M\times\M}  \lambda(t)  \mathbf{s}(t, x_t) \partial_t p_{t|0}(x_t)\rmd p_0(x_0) \rmd \piinv(x_t) \rmd t} \\
  & \qquad = - \textstyle{\int_0^T \int_{\M\times\M}  \partial_t (\lambda(t)  \mathbf{s}(\cdot, x_t))(t)  \rmd p_{0,t}(x_0, x_t)  \rmd t}\\
  & \qquad \qquad \qquad \textstyle{+ \lambda(T) \int_\M \mathbf{s}(T, x_T) \rmd p_T(x_T)  - \int_\M \mathbf{s}(0, x_0) \rmd p_0(x_0)} \\
  & \qquad = - \textstyle{\int_0^T \int_{\M\times\M}  \partial_t \lambda(t)  \mathbf{s}(t, x_t) \rmd p_{t}(x_t)  \rmd t} - \textstyle{\int_0^T \int_{\M\times\M}   \lambda(t)  \partial_t \mathbf{s}(t, x_t) \rmd p_{t}(x_t)  \rmd t}\\
  & \qquad \qquad \qquad \textstyle{+ \lambda(T) \int_\M \mathbf{s}(T, x_T) \rmd p_T(x_T)  - \lambda(0) \int_\M \mathbf{s}(0, x_0) \rmd p_0(x_0)}     
\end{align}
Combining this result and \eqref{eq:loss_L0} we get that
\begin{align}
        L_0(\mathbf{s}) &= \textstyle{\int_0^T \lambda(t) \int_{\M \times \M} (\mathbf{s}(t, x_t) - \partial_t \log p_{t|0}(x_t|x_0))^2 \rmd p_{0,t}(x_0, x_t) \rmd t} \\
                        &= \textstyle{\int_0^T \lambda(t) \int_{\M} \mathbf{s}(t, x_t)^2 \rmd p_t(x_t) \rmd t} + 2 \textstyle{\int_0^T \int_{\M\times\M}  \partial_t \lambda(t)  \mathbf{s}(t, x_t) \rmd p_{t}(x_t)  \rmd t} \\
  & \qquad +2 \textstyle{\int_0^T \int_{\M\times\M}   \lambda(t)  \partial_t \mathbf{s}(t, x_t) \rmd p_{t}(x_t)  \rmd t}  \textstyle{- \lambda(T) \int_\M \mathbf{s}(T, x_T) \rmd p_T(x_T)}    \\
      & \qquad + \textstyle{\lambda(0) \int_\M \mathbf{s}(0, x_0) \rmd p_0(x_0)} \textstyle{+ \int_0^T \lambda(t) \int_{\M} ^2 \rmd p_t(x_t) \rmd t},
\end{align}
which concludes the proof.
  \end{proof}

  Hence, using \Cref{prop:loss-function_fisher}, we could estimate jointly the
  spatial (or Stein) score used in RSGM and the Fisher score considered in this
  section, see \textcite{choi2021density}.


\section{Extensions} \label{sec:extensions}

\subsection{Schr\"odinger bridge.} For Euclidean SGM, the generative model is
given by an approximation of the time-reversal of the noising dynamics
$(\bfX_t)_{t \in \ccint{0,T}}$ while the backward dynamics
$(\bfY_t)_{t \in \ccint{0,T}}$ is initialized with the invariant distribution of
the noising dynamics (the uniform distribution $\piinv$ in case of
RSGM). However, in order for the method to yield good results we need
$\mathcal{L}(\bfY_0) \approx \mathcal{L}(\bfX_T)$ \cite[see][Theorem
1]{debortoli2021neurips}. Usually, this requires the number of steps in the
backward process to be large in order to keep $T$ large and $\gamma$ small
(where $\gamma > 0$ is the stepsize in the GRW). Another
limitation of SGM is that existing methods target an easy-to-sample reference
distribution. Hence, classical SGM cannot interpolate between two distributions
defined by datasets. To circumvent this problem, one can consider a process
whose initial and terminal distribution are pinned down using Schr\"odinger
bridges
\citep{schrodinger1932theorie,leonard2012schrodinger,chen2016entropic,debortoli2021neurips,vargas2021solving}. 
\subsection{Conditional RSGM.} Another extension of interest is conditional
sampling. By amortizing SGM with respect to an observation $y$ it is possible to
approximately sample from a given posterior distribution. In the Euclidean
setting this idea has been successfully applied for several image processing
problems such as deblurring, denoising or inpainting \citep[see for
instance][]{kawar2021snips,kawar2021stochastic,lee2021priorgrad,sinha2021d2c,batzolis2021conditional,chung2021come}. Similarly,
RSGM can be amortized to handle such situations in the case where the
underlying posterior distribution is supported on a manifold.
Practically, this requires for the score
network takes an additional input, i.e\ $\mathbf{s}_\theta\left(t,x;y\right)$.

\subsection{Invariant distributions}
%
In what follows, we propose an extension for modelling probability distributions which known invariance.
That is, we assume that $p_0\left(\rho(g) x\right) = p_0(x)$ for all $g \in G$, with $G$ a group and $\rho: G \rightarrow \text{GL}_n(\rset)$ a representation.
Following
\textcite{kohler2020Equivariant}, we have that 
if $\piinv$ is invariant w.r.t.\ $G$ and $\phi: \M \rightarrow \M$ is equivariant
w.r.t.\ to $G$, then the pushforward probability density $p=\piinv\circ\phi^{-1}$ is invariant w.r.t.\ $G$.

Let's consider the probability flow $\phi$ associated with the reverse diffusion \eqref{eq:time_reversal_manifold}---given by $\rmd \bfY_t = \{-b(\bfY_t) + \nabla \log p_{T-t}(\bfY_t)\} \rmd t + \rmd \bfB_t^\M$--- i.e.\ the solution of the following ODE (see \cref{sec:likel-comp})
\begin{equation}
\rmd \bfY_t = \{-b(\bfY_t) +\tfrac{1}{2}~ \nabla \log p_{T-t}(\bfY_t)\} \rmd t.
\end{equation}


%
In practice, the Stein score $\nabla \log p_{t}$ is approximated with the score network $\bm{s}_\theta(t, \cdot)$.
It is sufficient to parametrize the score network so that it is equivariant w.r.t. its second argument
---assuming that $\rho(g)$ and the drift $b$ commute (e.g.\ which is true for a linear drift)---since we then have
\begin{align}
    \left[- b + \tfrac{1}{2}~  \bm{s}_\theta\left(T - t, \cdot \right) \right] \left(\rho(g) \bfY_t \right) = 
    \rho(g) \left[- b + \tfrac{1}{2}~ \bm{s}_\theta\left(T - t, \cdot\right)\right] (\bfY_t).
\end{align}

\red{
\section{Stereographic baseline details}
\label{sec:stereo_exp}

In the experiments on the sphere we compare to a Stereographic Score-Based baseline model. This model is an alternative to the RSGM the we propose in order to construct score-based models on manifolds without having to construct the intrinsic approach presented in the paper as Riemannian Score-Based models. They can be applied to more cases than just the sphere.

In general these models work as follows:
\begin{enumerate}
    \item Project the datapoints from the manifold to Euclidean space through a invertible\footnote{Note that this may not be a bijection. For example for the sphere we use the stereographic projection of the earth onto the plane, which misses out a single point, opposite the projection point.} function $f: \mathcal{M} \to \R^d$.
    \item Train a Euclidean score-based generative model on the datapoints projected to Euclidean space, giving a density $p_{\theta}$ on $\R^d$ (where $\theta$ are the parameters of the density).
    \item Define the density on the manifold as the pushforward of the density in Euclidean space under the inverse of the bijection, $P_{\theta, \mathcal{M}} = f^{-1}_* p_{\theta}$.
\end{enumerate}

One could also apply these models to the torus. By using the bijection $f: \theta \mapsto \tan(\theta)$ we can project each coordinate onto the real line.

In general we found that these models perform less well than their intrinsic counterparts. In order to map density near the seams of the bijection, it requires the model to send data points off to infinity in the Euclidean space. This is numerically challenging and leaves artefacts in the pushforward density on the manifold. In addition, these methods depend on the bijection used to project the data into a Euclidean space and therefore are not intrinsic.}
\section{Experimental details}
\label{sec:exp_detail}

In what follows we describe the experimental settings used to generate results introduced in \cref{sec:experiments}.
The models and experiments have been implemented in Jax~\citep{jax2018github}, using a modified version of the Riemannian geometry library Geomstats~\citep{geomstats2020}.

Anonymized code can be found at \href{https://anonymous.4open.science/r/rimannian-score-sde}{here}\footnote{https://anonymous.4open.science/r/rimannian-score-sde}. Due to difficulties referencing anonymized repositories, the modified version of geomstats is included as a zip file in the supplementary material. Additionally modified versions of the \verb'submitit' and \verb'hydra-submitit-launcher' packages are not supplied for the same reasons, but the default versions of these will suffice for most users. Full code and all repos will be publicly available after publication.

\paragraph{Models}
Following \textcite{song2020score}, the score-based generative models (SGMs)
diffusion coefficient is parametrized as $g(t) = \sqrt{\beta(t)}$ with
$\beta: \ t \mapsto \beta_{\min} + (\beta_{\max} - \beta_{\min}) \cdot t$.  

\paragraph{Architecture}
The architecture of the score network $\bm{s}_\theta$ is given by a multilayer
perceptron with $5$ hidden layers for the Earth and $SO(3)$ experiments, and $3$ for the high-dimension experiments with $512$ units each.  We use sinusoidal
activation functions.  We decompose the output of the score network on the set
of divergence free vector fields as per \cref{sec:riem-score-appr}.

\paragraph{Loss}
Where not specified, SGMs are trained with the sliced score matching (SSM) loss $\ellim_t$,
relying on the Hutchinson estimator for computing the divergence with Rademacher
noise described in \cref{sec:riem-score-appr}.
We found that training with the denoising score matching (DSM) loss $\ell_{t|0}$ gave similar results.
Regarding the weighting function, for DSM loss $\ell_{t|0}$ we use $\lambda_t = \text{Var}[X_t|X_0]$ (where we rely on the closed-form standard deviation available in the Euclidean setting as a proxy for the compact manifold setting), while for the ISM/SSM losses $\ellim_t$ we use $\lambda_t = g(t)^2 = \beta(t)$.

\paragraph{Optimization}
All models are trained by the stochastic optimizer Adam \citep{kingma2015Adam}
with parameters $\beta_1=0.9$, $\beta_2=0.999$, batch-size of $512$ data-points. 
The learning rate is annealed with a linear
ramp from $0$ to $1000$ and from then with a cosine schedule. 

\paragraph{Likelihood evaluation and sample drawing}
We rely on the Dormand-Prince solver \citep{dormand1980family}, an adaptive
Runge-Kutta 4(5) solver, with absolute and relative tolerance of $1e-5$ to
compute approximate numerical solutions of any ODEs. For the rollouts of the SGM SDEs we use a Euler Maruyama predictor and no corrector. Unless stated we use 100 step rollouts. 

\paragraph{Hardware} Models are trained on a cluster with a mixture of GeForce RTX 1080, 1080 Ti and 2080 Ti GPU cards.

\subsection{Sphere}
\paragraph{Data}

We randomly split the datasets intro training, validation and test datasets with $(0.8, 0.1, 0.1)$ proportions. In each case the earth is
approximated as a perfect sphere.

\paragraph{Models}
The mixture of Kent distributions \citep{peel2001fitting} were optimised using
the EM algorithm and the number of components were selected from a grid search
over the range $5, 10, 15, 20, 25, 30, 40, 50, 75, 100$, based on validation set
likelihood and $250$ EM iterations. The number of components selected were:
Volcano $25$, Earthquake $50$, Flood $100$ and Fire $100$.

For the stereographic SGM--which is a
standard SGM with an Ornstein–Uhlenbeck process followed with the inverse
stereographic projection--we found $\beta_{\min}=0.001$ and $\beta_{\max}=2$ to
work best.

\paragraph{Optimization}
The score-based models are trained for $600k$ iterations for all datasets but `Flood' where
$300k$ performed best.

\paragraph{Additional experimental results}

\paragraph{Approximate forward sampling}
Standard Euclidean SGMs rely on a Ornstein--Ulhenbeck (OU) forward process \eqref{eq:forward_SDE} which can easily be simulated since  $\bfX_t | \bfX_0$ is Gaussian.
In contrast, for most manifolds one has to rely on an approximate sampling scheme---see \cref{sec:geodesic_random_walk}.
First, we directly assess the quality of the approximate samples $\hat{\bfX}_t | \bfX_0$ obtained via geodesic random walk (GRW), against `exact' samples ${\bfX}_t | \bfX_0$ which are obtained by using a high number of discretization steps ($N=1000$).
We report on \cref{fig:approx_forward_mmd} the discrepancy between these distributions for different values of discretization steps $N$, as measured by maximum mean discrepancy (MMD)~\citep{gretton2012kernel}.
We see that from $N=5$ the approximate samples are very closely distributed to the true samples. 
Then, in order to assess the impact of this approximation on the RSGMs' performance, we report on \cref{fig:approx_forward_logp} the log-likelihood when varying the number of discretization steps $N$. 
We similarly observe that apart from very small values of $N$, the models' performance is very robust to the approximation quality of the forward sampling samples.
\begin{figure}[h]
    \centering
    \begin{subfigure}{0.49\textwidth}
        \includegraphics[width=\textwidth]{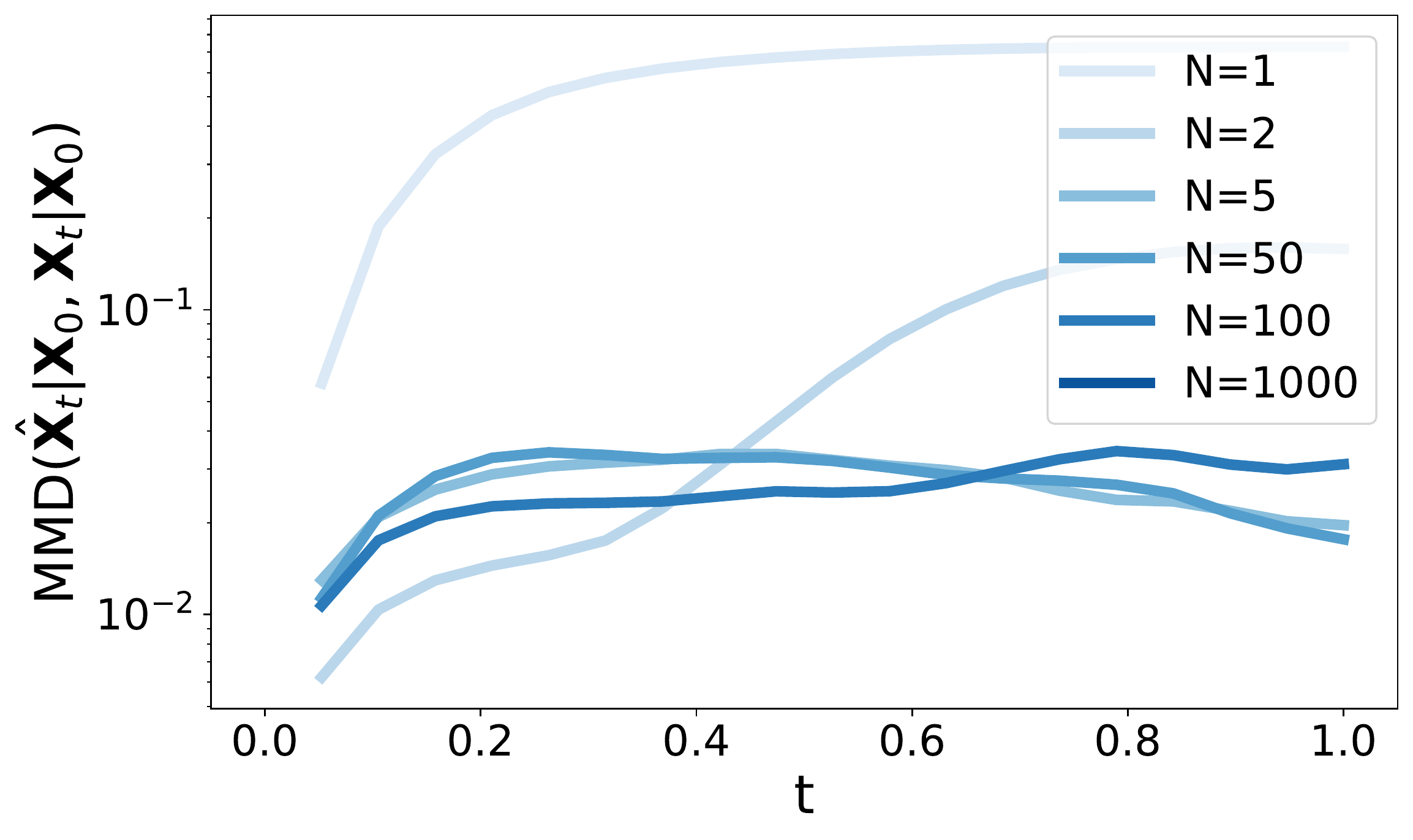}
        \caption{
        Maximum mean discrepancy (MMD) distance between \emph{`exact'} (i.e.\ approximated with $N=1000$ steps) $\bfX_t | \bfX_0$ and \emph{approximate} $\hat{\bfX}_t | \bfX_0$ at for every $t \ \in [0, 1]$.
        }
        \label{fig:approx_forward_mmd}
    \end{subfigure}
    \hfill
    \begin{subfigure}{0.49\textwidth}
        \includegraphics[width=\textwidth]{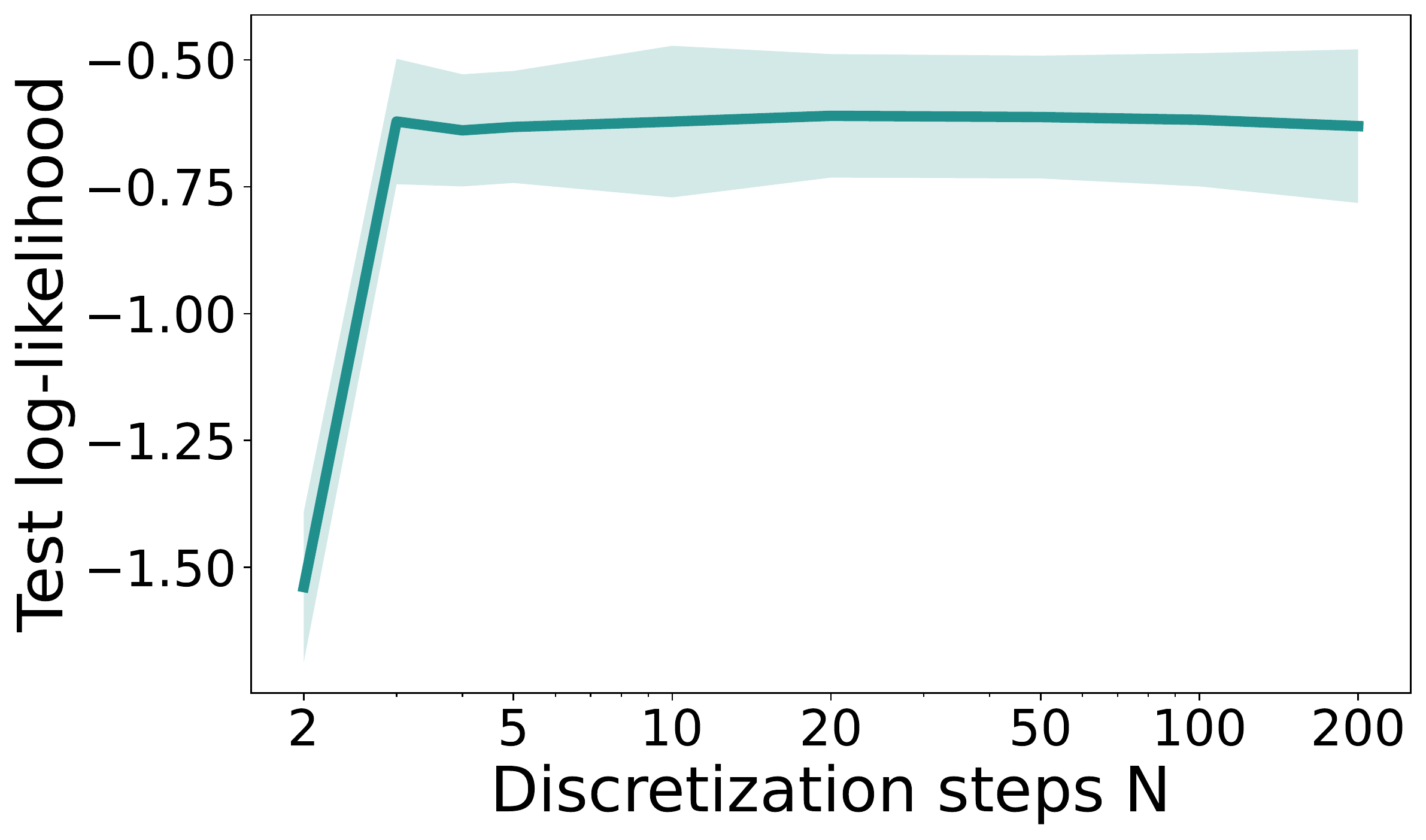}
        \caption{
        Test log-likelihood of trained RSGMs on the Flood dataset while varying the number of discretization steps $N$ when simulating  forward sampling $\bfX_t | \bfX_0$. 
        }
        \label{fig:approx_forward_logp}
    \end{subfigure}
    \vspace{-0.5em}
    \caption{
        Ablation study on the impact of the forward sampling approximation quality on $\mathbb{S}^2$. 
    }
    \label{fig:approximate_forward}
\end{figure}

\paragraph{DSM loss $\ell_{t|0}$ }
On \cref{fig:s2_heatmap}, we show how the test log-likelihood varies with respect to the two hyparameters of the DSM loss, by training RSGMs over a grid of values for $\tau$ and $J$ on the Flood dataset.
We can see that the Varadhan approximation by itself ($\tau = 1$) yields descent performance, although a wise combination of Varadhan approximation with a truncation of the heat kernel can give even better results.
The performance is relatively robust to the choice of such hyperparameters as long as $\tau$ and $J$ are high enough.
\begin{figure}[h]
    \centering
        \includegraphics[width=0.7\textwidth]{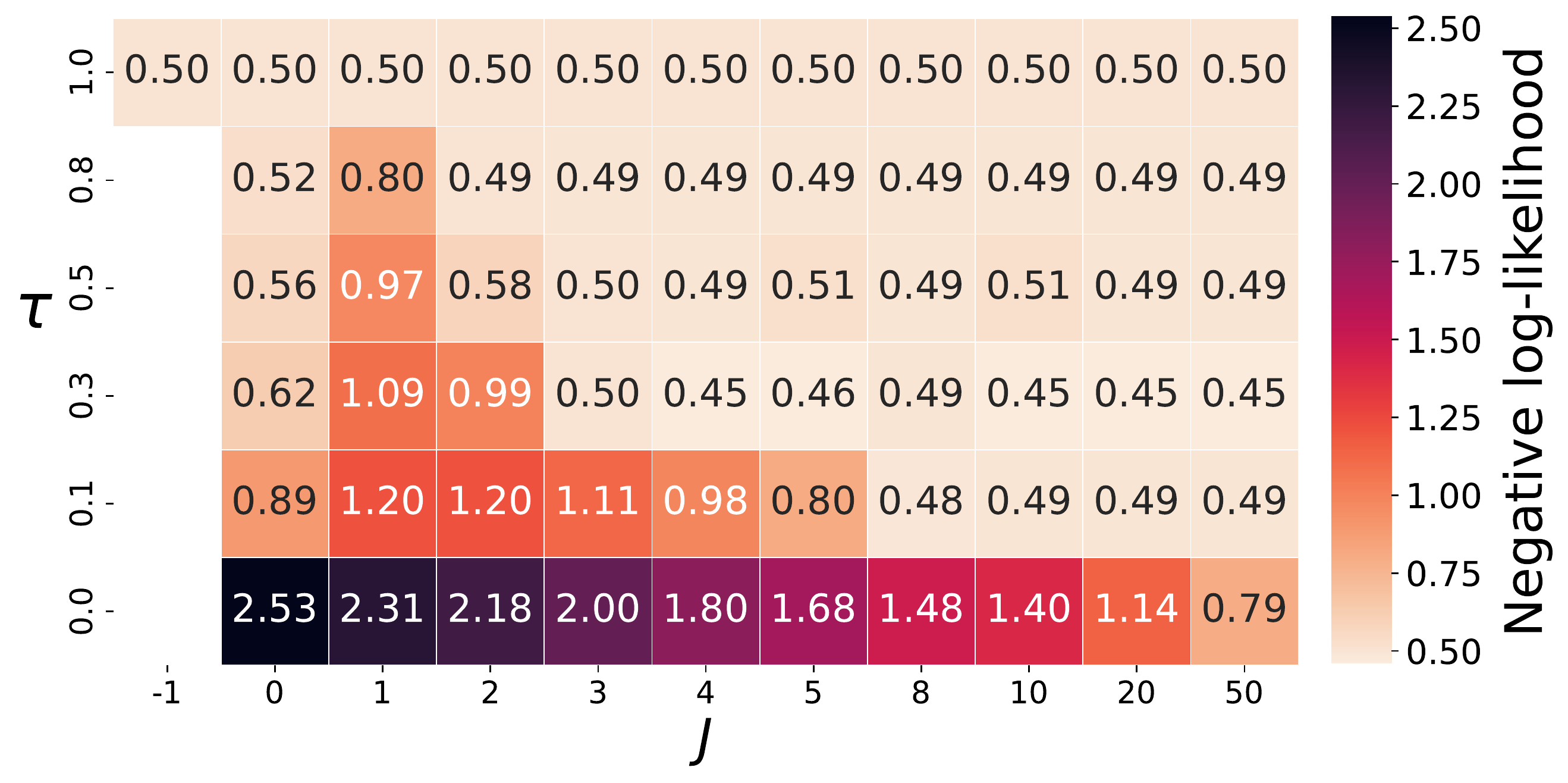}
       
    \vspace{-0.5em}
    \caption{
        Ablation study on the denoising score matching (DSM) loss $\ell_{t|0}$ when combining the heat kernel truncation and the Varadhan approximation: $\nabla_{x_t} \log p_{t|0}(x_t|x_0) \approx \mathbbm{1}(t \le \tau) \exp^{-1}_{x_t}(x_0) + \mathbbm{1}(t > \tau) S_{J,t}(x_0,x_t)$.
    }
    \label{fig:s2_heatmap}
\end{figure}

\subsection{Torus}

\paragraph{Data}

The synthetic data trained on consists of a wrapped Gaussian distribution on $\mathbb{T}^n$ with uniformly chosen random mean and standard deviation of $0.2$. Such a distribution is defined by taking the density of a Normal distribution in the tangent space of the manifold at the mean and passing it through the exponential map at the mean.

\paragraph{Architecture}

To parametrize the vector field on $\mathbb{T}^n$ we use a single filed per dimension pointing in a consistent direction around the i$^{th}$ component in the product, with unit norm.

\paragraph{Models}

All models were trained with the same 3 layer, 512 units per layer MLP across different dimension sizes.

\paragraph{Optimization}

The models are optimized for 50$k$ iterations. The RSGM models are trained with both the implicit score-matching loss and the sliced score-matching loss.

\subsection{Special Orthogonal group}
%

Applications of orthogonal constraints span
various fields, such as protein docking with ligands binding pose prediction
\citep{ganea2022independent}, robotics and Computer vision with rigid body
transformation estimation \citep{barfoot2011Pose,deepdirectstat2018}, and
medical imaging for data alignment \citep{hou2018Computinga}.

\paragraph{Data}
We consider the synthetic dataset consisting of samples in
$\mathrm{SO}_3(\rset^d)$\footnote{This manifold is $3$-dimensional.} from the
  mixture distribution with density
  $p(\mathrm{Q})=\frac{1}{K} \sum_{k=1}^K
  \mathrm{N}^W(\mathrm{Q}|\mathrm{Q}_k,\sigma_k^2)$ with $K \in \nset$,
  where for any $k \in \{1, \dots, K\}$, we have that
  $\mathrm{Q} = \mathrm{Q}_k\exp_{\Id}[\sigma_k \hat{z}]$ with
  $z \sim \mathrm{N}(0, \Id_{\R^3})$ satisfies
  $Q \sim \mathrm{N}^W(\mathrm{Q}_k, \sigma_k)$ and $(\cdot)^{\wedge} : \R^3 \rightarrow \mathfrak{so}(3)$.
  For any $k \in \{1, \dots, K\}$, we set
  $\mathrm{Q}_k \sim \mu$ where $\mu$ is the uniform distribution on
  $\mathrm{SO}_3(\rset)$ and $\sigma_k^2 \sim \mathrm{IG}(\alpha=100, \beta=1)$, where $\mathrm{IG}$ is the inverse Gaussian distribution.
We choose $K=32$ mixture components.
We showcase a conditional sampling extension of our model---see \cref{sec:extensions} for more
details---
by targeting individual mixture components $p(Q|k)$.
Our model is trained using the $\ell_{t|0}$ (DSM)
loss along with the Varadhan asymptotic approximation, see \eqref{eq:varadhan}.

\paragraph{Architecture}
To parametrize the vector field, we rely on the basis of the Lie group,
$\mathfrak{so}(n)= \ensembleLigne{\mathrm{A} \in \mathrm{M}_d(\R)}{\mathrm{A}^\top =
  -\mathrm{A} }$ given by $\mathrm{E}_{ij} = \mathrm{U}_{ij} - \mathrm{U}_{ji}$
for $i, j \in \{1, \dots, d\}$ with $i < j$ and
$\mathrm{U}_{ij} = (\updelta_{ij}(k,\ell))_{1\leq k, \ell \leq d}$, which
induces a basis on the tangent spaces $\mathrm{T}_{\mathrm{Q}} \mathrm{SO}_d$
for any $\mathrm{Q} \in \mathrm{SO}_d(\rset)$ given by
$\{\mathrm{Q} \mathrm{E}_{ij}\}_{1 \leq i<j \leq d}$.
This is the divergence-free vector field approach described in \cref{sec:riem-score-appr}.

\paragraph{Models}
We compare our proposed approach against Moser flows~\citep{rozen2021moser} and a wrapped-exponential baseline
\citep{falorsi2019reparameterizing} defined as the pushforward along the
transformation
$\R^3 \xrightarrow[]{F^{-1}_\theta} \R^3 \xrightarrow[]{g} \R^3
\xrightarrow[]{\wedge} \mathfrak{so}(3) \xrightarrow[]{\exp} \mathrm{SO}_3(\rset)$
with $F^{-1}_\theta$ denoting the approximate time-reversed diffusion, $g$
denoting the radial operator defined by
$g: \ x \mapsto 2\uppi \tanh(\|x\|) x / \|x\|$, $(\cdot)^{\wedge} : \R^3 \rightarrow \mathfrak{so}(n)$ the
isomorphism given by the basis on $\mathfrak{so}(3)$ and $\exp$ the matrix
exponential.
The radial $g$
  operator's constant $2\uppi$ is chosen as the injectivity radius of the group
  so that the transformation $\tanh \circ \wedge \circ \exp$ is injective (the
  set of elements with no preimage is then only the cut locus which is known to
  have measure zero). 
Henceforth, this wrapped-exponential transformation cannot be bijective, it is either injective \emph{or} surjective depending on the choice of radius in the radial operator $g$.

\paragraph{Optimization}
Models are trained for $100k$ iterations.
The Riemannian SGM is trained with the Varhadan approximation of the denoising score-matching loss (DSM)~\cref{sec:riem-score-appr}, and the wrapped-exponential model relies on the exact DSM loss.
After a first hyperparameter exploration, a grid search is performed over $\texttt{learning\_rate} \in [2e-5, 4e-5]$, for SGMs over $\beta_f \in [0.5, 1, 2, 4, 6, 8, 10]$ and for Moser flows over $K \in [1000, 10000]$ and $\lambda_{\min} \in [1, 10, 100]$.




\end{document}